\documentclass{article} % For LaTeX2e
\usepackage{iclr2022_conference,times}

\usepackage{graphics} % for pdf, bitmapped graphics files
\usepackage{epsfig} % for postscript graphics files
\usepackage{mathptmx} % assumes new font selection scheme installed
\usepackage{enumerate}

\usepackage{times} % assumes new font selection scheme installed
\usepackage{amsmath} % assumes amsmath package installed
\usepackage{amssymb}  % assumes amsmath package installed
\usepackage{mathrsfs}
\usepackage{color}

\usepackage[normalem]{ulem}

\usepackage{algorithmicx}
\usepackage[ruled]{algorithm}
\usepackage{algpseudocode}

\newtheorem{thm}{Theorem}
\newtheorem{lem}{Lemma}

\newtheorem{proof}{Proof}
\newtheorem{assum}{Assumption}
\newtheorem{definition}{Definition}

\usepackage{amsmath,amsfonts,bm}

\def\eqref#1{equation~\ref{#1}}
\def\Eqref#1{Equation~\ref{#1}}
\def\1{\bm{1}}

\DeclareMathAlphabet{\mathsfit}{\encodingdefault}{\sfdefault}{m}{sl}
\SetMathAlphabet{\mathsfit}{bold}{\encodingdefault}{\sfdefault}{bx}{n}

\usepackage{hyperref}
\usepackage{url}

\DeclareMathOperator{\Expect}{\mathbb{E}}

\title{On the Convergence  of mSGD and AdaGrad \\for Stochastic Optimization}

\author{Ruinan Jin  \\
LSC, NCMIS, Academy of Mathematics and Systems Science\\
Chinese Academy of Sciences, Beijing 100190, China\\
  School of Mathematical Sciences\\
   University of   	Chinese Academy of Sciences, Beijing 100049, China.\\
\And
Yu Xing  \\
Division of Decision and Control Systems\\
KTH Royal Institute of
Technology\\
 SE-100 44 Stockholm, Sweden\\
\texttt{yuxing2@kth.se}
\AND
Xingkang He \thanks{Corresponding author.} \\
Division of Decision and Control Systems\\
KTH Royal Institute of
Technology\\
SE-100 44 Stockholm, Sweden\\
Department of Electrical Engineering \\
University of Notre Dame, IN, USA\\
\texttt{xingkang@kth.se; xhe9@nd.edu}
}

\iclrfinalcopy % Uncomment for camera-ready version, but NOT for submission.
\begin{document}

\maketitle

\begin{abstract}
	% 			As the rapid increase of data amount in machine learning, optimization-based methods inevitably confront   computational issues, which can be well handled by leveraging stochastic optimization algorithms.  
		As one of the most fundamental stochastic optimization algorithms, stochastic gradient descent (SGD) has been intensively developed and extensively applied in machine learning in the past decade. There have been some modified SGD-type algorithms, which outperform the SGD in many competitions and applications in terms of convergence rate and accuracy, such as momentum-based SGD (mSGD) and adaptive gradient algorithm (AdaGrad). Despite these empirical successes, the theoretical properties of these algorithms have not been well established due to technical difficulties. With this motivation, we focus on convergence analysis of mSGD and AdaGrad for any smooth (possibly non-convex) loss functions in stochastic optimization. First, we prove that the {iterates of mSGD are} asymptotically convergent to a connected set of stationary points with probability one, which is more general than existing works on subsequence convergence or convergence of time averages. Moreover, we prove that the loss function of mSGD decays at a certain rate faster than that of SGD. In addition, we prove the {iterates of AdaGrad are} asymptotically convergent to a connected set of stationary points with probability one. Also, this result extends the results from the literature on subsequence convergence and the convergence of time averages. Despite the generality of the above convergence results, we have relaxed some assumptions of gradient noises, convexity of loss functions, as well as boundedness of {iterates}.
\end{abstract}

	\section{Introduction}
	In recent years, the rapid development of machine learning has stimulated a lot of applications of optimization algorithms to employing tremendous data in practical scenarios. 
	%As the one of the main approaches, 
	Many optimization algorithms of machine learning are based on gradient descent (GD). 
	A typical  GD algorithm, minimizing a  loss function $g(\theta)\in \mathbb{R}$ via seeking an $N$-dimensional real-valued  parameter $\theta^*=\arg\min\limits_{\theta\in \mathbb{R}^N}g(\theta)$, writes as follows
	\begin{equation}\label{eq_GD}
	\begin{aligned}
	\theta_{n+1}=\theta_{n}-\epsilon_{n} \nabla_{\theta_n} g(\theta_{n}),
	\end{aligned} 
	\end{equation}
	where $\theta_{n}$ is the estimate of $\theta^*$ at step $n$, $\epsilon_{n}$ is a positive step size (learning rate) to be designed,   and $\nabla_{\theta_n} g(\theta_{n})$ stands for the gradient of $g(\theta_{n})$ at step $n$.  When certain conditions are satisfied,  $\theta_{n}$ in \eqref{eq_GD} will  converge to the optimal solution.
	However, since the computation of $\nabla_{\theta_n} g(\theta_{n})$ relies on all data at each step, it is inefficient to apply  GD-based algorithms (e.g.,  \eqref{eq_GD}) when the data amount is very large.  
	Therefore, more and more attention has been given on how to  accelerate the GD-based algorithms. 
	One of the attempts is stochastic gradient descent (SGD) which   originated from \citet{1951A}. Instead of calculating $\nabla_{\theta_n} g(\theta_{n})$ over all data, SGD algorithms use  a 
	relatively small proportion of data to estimate the gradient,
	%(unbiased estimate of $\nabla_{\theta_n} g(\theta_{n})$) to compute the gradient, 
	i.e.,   $\nabla_{\theta_n} g(\theta_{n},\xi_{n})$ where $\xi_{n}$ is  a random vector introduced to choose a subset of data for each update. 
	Then the  SGD modified from GD \eqref{eq_GD}  is as follows
	\begin{equation}\label{sgd}
	\begin{aligned} 
	%v_n=,
	\theta_{n+1}=\theta_{n}-\epsilon_{n} \nabla_{\theta_n} g(\theta_{n} ,\xi_{n}).
	\end{aligned} \end{equation}
	Besides  reducing gradient computation,   $g(\theta_{n}, \xi_{n})$ can  be used as an estimate of the gradient for the scenarios where  the accurate gradient is unavailable due to external noises.  
	
	In recent years,   SGD   has shown its prominent advantages in dealing with high dimensional optimization problems such as regularized empirical risk minimization and training deep neural networks (see, e.g. \citet{Graves2013Speech,Nguyen2018SGD,HintonReducing,Krizhevsky2012ImageNet} and references therein). 
	In \citet{Nguyen2018SGD, Bottou2012Stochastic}, the convergence and the convergence rate of   SGD   have been analyzed.
	
	Despite   outstanding successes of SGD, it usually has a relatively slow convergence rate induced by   random gradients and it provides fluctuating iterates in the learning process.
	In order to improve the theoretical and empirical performance of SGD, 
	there have been a number of investigations, including  
	%accelerate the SGD algorithm,
	%improve the theoretical and empirical performance of SGD, there are a number of investigations 
	%there are vast literature on modifications of it. 
	three typical   algorithms for accelerating the convergence rate: (1) The momentum-based stochastic gradient descend (mSGD), which reduces the update variance by averaging the past gradients  \citep{1964Some}; (2) The adaptive gradient algorithm (AdaGrad), which replaces the step size of SGD with adaptive step size  \citep{JMLR:v12:duchi11a}; (3) Adaptive momentum gradient algorithm (Adam), which integrates   mSGD and   AdaGrad \citep{KingmaAdam}. However, the convergence results of these algorithms have not been well established, especially for non-convex loss functions.
	In this paper, we focus on the investigation of   mSGD and    AdaGrad. We believe that the convergence analysis results in this paper are helpful for the future research on the convergence or generalization of  Adam and other stochastic optimization algorithms.
	%\begin{enumerate}
	%\item[$\bullet$]The momentum-based stochastic gradient descend (mSGD), which reduces the update variance by averaging the past gradients to accelerate the convergence to the optimal point \citet{1964Some}.
	%\item[$\bullet$]The adaptive gradient algorithm (AdaGrad), which replaces the step size of SGD with adaptive step size to accelerate the convergence to the optimal point \citet{JMLR:v12:duchi11a}. 
	%\item[$\bullet$]Adaptive moment estimation (Adam), which combine mSGD and AdaGrad \citet{KingmaAdam}.
	%\end{enumerate}
	%that aim to 
	%Another attempt for the acceleration of  GD-based algorithms is the momentum-based GD (mGD), whose  idea was originally from \citet{Polyak1964Some}. 

	%The momentum-based GD reduces the update variance by averaging the past gradients  to accelerate the convergence to the optimal point.
	%The inspiration of \citet{Polyak1964Some} roots in the fact that an unrestrained {{heavy ball is quicker than a running person when they are both going down to the valley bottom}}.(physical explanation???)
	%Authors in \citet{QianOn} proved the convergence and analyzed the convergence rate. In \citet{nesterov1983method}, an enhanced version of mGD was presented, i.e., Nestrov momentum-based gradient descent.		
	%Based on the above two aspects, it is natural to investigate the momentum-based stochastic gradient descent (mSGD) for the convergence acceleration of SGD-type algorithms. 
	%To accelerate the convergence of SGD-type algorithms, 
	The technique of   mSGD  is originally developed by Polyak \citep{1964Some} for the acceleration of   convergence rate of gradient-based methods. A typical expression of mSGD is as follows \citep{2015Accelerated,sutskever2013importance}
	\begin{equation}\label{mSGD}
	\begin{aligned}
	v_n=\alpha v_{n-1}+\epsilon_{n} \nabla_{\theta_n} g(\theta_{n} ,\xi_{n}),\quad \theta_{n+1}=\theta_{n}-v_n, 
	\end{aligned} 
	\end{equation}
	where $\alpha\in[0,1)$ and $\epsilon_{n}>0$ are  relatively momentum coefficient and step size (learning rate), respectively. Another type of  mSGD with the name of stochastic heavy ball (SHB) has also been studied \citep{2019Understanding,1972Stochastic}, 
	\begin{equation} \label{SHB}
	\begin{aligned}
	v_n=\beta_{n} v_{n-1}+(1-\beta_{n})\nabla_{\theta_n} g(\theta_{n} ,\xi_{n}),\quad \theta_{n+1}=\theta_{n}-\gamma_{n}v_n,
	\end{aligned} 
	\end{equation}
	where $\beta_{n}\in(0,1)$ and $\gamma_{n}$ are relatively momentum coefficient and step size (learning rate). 
%	In Section~3, we will show that mSGD and SHB are   equivalent under general conditions. 
	In recent years, mSGD has been widely employed  in the applications of deep learning such as   image classification   \citep{Krizhevsky2012ImageNet}, fault diagnosis \citep{tang2018adaptive},  statistical image reconstruction \citep{kim2014combining}, etc. Moreover, a number of variants on momentum are emerging, see, e.g., synthesized Nesterov variants (SNV) \citep{2016Numerical},  robust momentum \citep{2017A}, and PID-control based methods \citep{2018A}.   
	The importance of momentum in deep learning has been illustrated in  \citet{sutskever2013importance} through experiments.
%	Some other widely used algorithms like Adam \citet{KingmaAdam} %and %PMSProp \citet{hinton2012coursera},
%	also borrow the idea of mSGD. 
	However, the theoretical analysis for convergence  and convergence rate of mSGD needs further investigation, especially for non-convex loss functions which are common in deep learning. Most   existing  results  are  established with guaranteed  subsequence convergence or convergence of time averages\footnote{The convergence definitions are given in Section~3.} (see~\citet{yang2016unified,2019Understanding,1977Comparison,YM1983Limit}, {\cite{2020A}} and references therein). Nevertheless, there is a still distance to asymptotic convergence, which is usually more general and useful in practical applications requiring stable iterates for each realization. { \cite{2020Almostuiuiyuyiuy} studied the asymptotic convergence of mSGD with  time-varying  parameter $\alpha_n$, which however is less common than the static $\alpha$ in practical applications.	}
% 	of $g(\overline{\theta}_{n+1})$ by using the time-average estimates $\overline{\theta}_{n+1}$,  where $\overline{\theta}_{n+1}=\omega_{n}\theta_{n}+(1-\omega_{n})\overline{\theta}_{n}$ and $\omega_n\in (0,1)$. However, the convergence of $g(\theta_{n+1})$ was not guaranteed.

	%The result of \citet{2019Understanding} is nearest to asymptotic convergence, but still exists huge distance (we will show concepts of subsequence convergence and strong convergence in Section~3). 
	%And the setting of the momentum coefficient in \citet{2019Understanding} is not same as the practical application (we will show more 
	%details in Section 3).
	
	%Second we introduce the AdaGrad. 
	Convergence analysis for   SGD unveils that the convergence rate heavily depends on the  choice of  step size $\epsilon_{n}$ \citep{1951A,Nguyen2018SGD}, which may consume a huge amount of efforts on fine tune. {In order to deal with this problem, the algorithm AdaGrad with adaptive step size  is proposed in \citet{2011Adaptive} and \citet{mcmahan2010adaptive} concurrently.}
	 { In the literature, there are two main forms of AdaGrad}. One is based on the norm of gradients as follows { \citep{2010Less}} 
	\begin{equation}\label{AdaGrad}\begin{aligned}
	S_{n}=S_{n-1}+\big\|\nabla_{\theta_{n}} g(\theta_{n},\xi_{n})\big\|^{2},
	\quad \theta_{n+1}=\theta_{n}-\frac{\alpha_0}{\sqrt{S_{n}}}\nabla_{\theta_{n}}g(\theta_{n},\xi_{n}),\end{aligned} \end{equation}where $\alpha_{0}>0$ is a constant. {The other form is based on the coordinate-wise gradients \citep{2011Adaptive,2018On,2020A}
	\begin{equation}\nonumber\begin{aligned}
	Q_{n}=Q_{n-1}+\nabla_{\theta_{n}}g(\theta_{n},\xi_{n})^{T}\nabla_{\theta_{n}}g(\theta_{n},\xi_{n}),
	\quad \theta_{n+1}=\theta_{n}-{\overline{\alpha}_0}{Q_{n}^{-\frac{1}{2}}}\nabla_{\theta_{n}}g(\theta_{n},\xi_{n}),\end{aligned} \end{equation} where $\overline{\alpha}_{0}$ is a constant. In our paper we focus on the norm form (\eqref{AdaGrad})}. In recent years,  AdaGrad has  shown its effectiveness in the field of sparse optimization \citep{2013Estimation},  tensor factorization \citep{2018Canonical}, and deep learning \citep{Heaton2017Ian}. Some algorithm variants   like RMSProp \citep{tielemanlecture} and SAdaGrad \citep{2018SADAGRAD} are also studied. However, there are few results on the convergence  of AdaGrad. Most of these  results  only prove   subsequence convergence or convergence of time averages  (see, e.g. \citet{2019A,2018Onp,A2020A,2018AdaGrad}). 
	Although \citet{2018On} and {\cite{2020Asymptotic}} studied  the asymptotic convergence of a modified AdaGrad algorithm,   the result is not applicable to   AdaGrad in \eqref{AdaGrad}. Thus, the asymptotic convergence of AdaGrad in \eqref{AdaGrad} is still open.

	%And finally we make a short introduction about Adam. This algorithm can be seen as a combination of mSGD and AdaGrad, which add a momentum term and an adaptive step size. 
	%Its concrete form is as follow
	%\begin{equation}\begin{aligned}
	%&v_{n}=\frac{1}{1-\alpha^{n}}\big(\alpha v_{n-1}+(1-\alpha)\nabla_{\theta_{n}}g(\theta_{n},\xi_{n})\big)
	%\\&S_{n}=\frac{1}{1-\beta^{n}}\big(\beta S_{n-1}+(1-\beta)\big\|\nabla_{\theta_{n}}g(\theta_{n},\xi_{n})\big\|^{2}\big)
	%\\&\theta_{n+1}=\theta_{n}-\frac{s}{\sqrt{S_n}}v_{n},
	%\end{aligned}\end{equation}
	%where $0<\alpha<1$, $0<\beta<1$ and $s>0$ are constants. Relevant content can reference \citet{KingmaAdam} and \citet{2019Onkoi}, which is not the emphasis of our paper.

	In this theoretical paper, we aim to establish the convergence of mSGD and AdaGrad under mild conditions. 	The main contributions of this paper are three-fold: 
	\begin{enumerate}
		\item[$\bullet$] 
		We prove that the {iterates of mSGD are} asymptotically convergent to a connected set of stationary points for possibly non-convex loss function almost surely (i.e., with probability one), which is  more general than  existing works on subsequence convergence.
		%By using stochastic approximation techniques \citet{chen2006stochastic}, we  %By using stochastic approximation techniques \citet{chen2006stochastic}, % 
		\item[$\bullet$]
		%we prove that the estimate sequence of mSGD is convergent with probability one, which is  more general than  existing works on subsequence convergence. Second, 
		We quantify the convergence rate of   mSGD for the loss functions. Through this convergence rate we can get a theoretical explanation of why mSGD can be seen as an acceleration of SGD. Moreover, we provide the convergence rate of mean-square gradients and connect it to the convergence  of time averages.

		\item[$\bullet$]  We prove  the {iterates of AdaGrad are}   asymptotically convergent to a connected set of stationary points almost surely for  possible non-convex loss functions. The convergence result for the AdaGrad   extends the subsequence convergence in the literature.
		
		%Focusing on a generalized version of the adaptive stepsizes popularized by AdaGrad \citet{JMLR:v12:duchi11a}. We prove that AdaGrad is strong convergence for each trajectory in the P-L conditions for loss function (potentially non-convex). 
		
		%8、各个方法的比较
		%
		%Karpathy做了一个这几个方法在MNIST上性能的比较，其结论是：
		%
		%adagrad相比于sgd和momentum更加稳定，即不需要怎么调参。而精调的sgd和momentum系列方法无论是收敛速度还是precision都比adagrad要好一些。在精调参数下，一般Nesterov优于momentum优于sgd。而adagrad一方面不用怎么调参，另一方面其性能稳定优于其他方法。
		% 

		%\item[$\bullet$]Adaptive moment estimation (Adam), which combine mSGD and AdaGrad \citet{KingmaAdam}.
	\end{enumerate}

		The remainder of the paper is organized as follows. In Section~2, we introduce the related work considering the convergence of mSGD and AdaGrad. The main results of the paper are given in Section~3, where we study the convergence and convergence rate of mSGD  as well as the convergence of AdaGrad. Section~4 concludes the whole paper. Sections~5~and~6 are Code of Ethics and Reproducibility, respectively. The proofs are given in Appendix.
	
	%\textbf{Problem: }Prove the convergence of the   mSGD  algorithm in \eqref{mSGD}, and analyze its convergence rate.

	\section{Related work}
	\textbf{Convergence of mSGD:} For the normalized mSGD (SHB), Polyak \citep{1977Comparison,1964Some} and Kaniovski \citep{YM1983Limit} studied its convergence (subsequence convergence and convergence of time averages) properties for convex loss functions.   Igor Gitman \citep{2019Understanding} provided some convergence results of mSGD (SHB) for non-convex loss functions,   but there is   a considerable distance to the asymptotic convergence. Moreover,   there is a requirement for uniform boundedness of a noise term in \citet{2019Understanding}, i.e.,  $\Expect(\|\nabla_{\theta_{n}}g(\theta_{n},\xi_{n})-\nabla_{\theta_{n}}g(\theta_{n}\|)\le \delta$, which confines the application scope of mSGD (SHB). In addition, the designs of   momentum coefficients in \citet{1977Comparison,1964Some,YM1983Limit,2019Understanding}   are not consistent with some practical applications \citep{Smith2018Don,sutskever2013importance}. Therefore, the asymptotic convergence of mSGD for   convex and non-convex loss functions needs further investigation.
	
	\textbf{Convergence rate of mSGD:} 
	Despite outstanding empirical successes of mSGD, there are few results on convergence rate of mSGD.
	In these results,   \citet{2020Convergence} studied a class of convex loss functions, and obtained a convergence rate of time averages without reflecting the role of momentum parameter.   \citet{2019Understanding} and Nicolas  \citet{2017Momentum} respectively investigated the asymptotic convergence rate of mSGD (SHB) by restricting to  quadratic  loss functions.  \citet{2020An} studied the properties of SHB, where   the relation between the loss function of SHB and the step size in every step was studied under the setting that $\alpha_{n}$ and $\beta_{n}$ are   constants. The convergence rate of time averages was also studied. However, since  the momentum parameters are not consistent with some   applications \citep{Smith2018Don,sutskever2013importance} and the standard mSGD in \eqref{mSGD} is not covered, further studies  are needed.

	%The convergence rate of the gradient descend with momentum (mGD) were established by Polyak for the case of convex functions \citet{1963Gradient,1977Comparison}. And for results of convergence rate of mSGD (SHB) are relative less. 
	% And in recent years. 

	\textbf{Convergence of AdaGrad:} In the original work for AdaGrad \citep{JMLR:v12:duchi11a}, it was  proved that AdaGrad can converge faster in the time averages sense  if   gradients are sparse and the loss function is convex. Similar results were   established by    \citet{2018Onp}, and  \citet{2018AdaGrad}.   \citet{2019A} and   \citet{A2020A} established convergence results in the subsequence sense.  Asymptotic convergence was obtained in \citet{2018On} for non-convex functions, but the form of the algorithm is no longer standard, as discussed in \citet{JMLR:v12:duchi11a}. Although such a change alleviates the difficulty in the  proof of asymptotic convergence, it cannot be applied to the study of the AdaGrad in \eqref{AdaGrad}.
	 Moreover, they required that  a noise term is of point-wise boundedness (i.e.,  $\big\|\nabla_{\theta_{n}}g(\theta_{n},\xi_{n})-\nabla_{\theta_{n}}g(\theta_{n})\big\|\le \delta$, where $\delta$ is a positive constant), which however is relatively restrictive.

	\section{Main results}
	In this section, we provide the main results of this paper, including the analysis of convergence and convergence rate of mSGD in \eqref{sgd} and the analysis of convergence of AdaGrad in \eqref{AdaGrad}. In the following, $\mathbb{R}^{N}$ denotes the $N$-dimensional Euclidean space and $\|\cdot\|$ stands for the 2-norm, i.e., the Euclidean norm. To proceed, we need some definitions, consisting of  asymptotic convergence, subsequence convergence, mean-square convergence, and convergence of time averages.
	\begin{definition}
		\label{strong} (\textbf{Asymptotic convergence})
		A sequence $\{x_{n}\}$ is   asymptotically convergent  to a set $\mathbb{K}$, if  
		$\lim_{n\rightarrow+\infty}\Big(\inf_{x\in \mathbb{K}}\big\|x_{n}-x\|\Big)=0.$
	\end{definition}
	\begin{definition}
		\label{weak} (\textbf{Subsequence convergence})
		A sequence $\{x_{n}\}$  converges in subsequence to a set $\mathbb{K}$, if there exists  at least one subsequence $\{x_{k_{{n}}}\}$ of $\{x_{n}\}$  such that
		$\lim_{n\rightarrow+\infty}\Big(\inf_{x\in \mathbb{K}}\big\|x_{k_{{n}}}-x\|\Big)=0.$
	\end{definition}
		\begin{definition}
		\label{weak} (\textbf{Mean-square convergence})
		A stochastic sequence $\{x_{n}\}$ converges in mean square to a fixed vector $x$, if  
		$\lim_{n\rightarrow+\infty}\mathbb{E}(\|x_{n}-x\|^2)=0.$
	\end{definition}
	\begin{definition}
		\label{weak} (\textbf{Convergence of time averages})
		A stochastic sequence $\{x_{n}\}$ converges in time averages to a fixed vector $x$, if  
		$\lim_{T\rightarrow+\infty}\frac{1}{T}\sum_{n=1}^{T}\mathbb{E}(\|x_{n}-x\|^2)=0.$
	\end{definition}
	It is obvious that asymptotic convergence implies   subsequence convergence, and that mean-square convergence ensures   convergence of time averages, but not vice versa. 
	
	\subsection{Convergence of mSGD}
	In this subsection, with the help of some stochastic approximation techniques \citep{chen2006stochastic}, we aim to prove that $\theta_{n}$ in \eqref{mSGD} is asymptotically convergent to a connected component $J^{*}$ of the set $J:=\{\theta|\|\nabla_{\theta}g(\theta)\|=0\}$ almost surely (a.s.) under proper conditions. When this connected component degenerates to a stationary point $\theta^{*}$, it holds that $\theta_{n}\rightarrow \theta^{*}, \text{ a.s.}$.
	
	In contrast to the existing works of  subsequence convergence (cf.  \citet{2019A,2018Onp,A2020A,2018AdaGrad}), we aim to prove that $\theta_{n}$ of   mSGD in  \eqref{mSGD} is able to achieve asymptotic convergence. 
	Since  mSGD in  \eqref{mSGD} is a stochastic algorithm, we aim to establish its a.s. asymptotic convergence.
	%In order to analyze the convergence of $\theta_{n}$, we need to establish some   properties of $g(\theta_{n})$ and $\|\nabla_{\theta_{n}}g(\theta_{n})\|$.
	%%focus on the property of $g(\theta_{n})$ or $\|\nabla_{\theta_{n}}g(\theta_{n})\|$, which is inspired by the Lyapunov function in the domain of control theory. 
	%We try to prove that $g(\theta_{n})\rightarrow g(\theta^{*})\ \ a.s.$ or $\|\nabla_{\theta_{n}}g(\theta_{n})\|\rightarrow 0\ \ a.s.$. %And for the case of weak %convergence, there are many different way. 
	%%In the research of the convergence, we need some basic assumptions of $g(\theta)$.
	%%加一句强收敛可以推出弱收敛！！！
To proceed, we need some reasonable assumptions with respect to   noise sequence $\{\xi_{n}\}$ and   loss function $g(\theta)$. 
	\begin{assum}\label{ass_noise}
		Noise sequence $\{\xi_{n}\}$ are mutually independent  and independent of $\theta_{1}$ and $v_{0}$, such that 
		%	 $g(\theta_{n})=\Expect_{\xi_{n}}\big(g(\theta_{n},\xi_{n})\big).$ 
		$g(x)=\Expect_{\xi_{n}}\big(g(x,\xi_{n})\big)$ 
		for any $x\in \mathbb{R}^{N}.$
	\end{assum}
	\begin{assum}\label{ass_g}(\textbf{Loss function assumption}) Loss function $g(\theta)$ satisfies the following conditions:
		\begin{enumerate}[1)]
			\item  $g(\theta)$ is a non-negative and  continuously differentiable function.
			\item The set of stationary points of $\|\nabla_{\theta}g(\theta)\|$ is not an empty set, that is $$J:=\{\theta|\|\nabla_{\theta}g(\theta)\|=0\}\not=\emptyset.$$
			
			\item  $\nabla_{\theta}g(\theta)$ satisfies the Lipschitz condition, i.e., there is a   scalar $c>0$, such that for any $ x,y\in \mathbb{R}^{N}$
			\begin{equation}\nonumber\begin{aligned}
			\big\|\nabla_{x}g(x)-\nabla_{y}g(y)\big\|\le c\|x-y\|.
			\end{aligned}\end{equation}
			\item There is a   scalar $M>0$  such that for any $\theta\in \mathbb{R}^{N}$ and positive integer  $ n$, 
			\begin{equation}\label{0908}\begin{aligned}
			&\Expect_{\xi_{n}}\Big(\big\|\nabla_{\theta}g(\theta)-\nabla_{\theta}g(\theta,\xi_{n})\big\|^{2}\Big)\le M\big(1+g(\theta)\big).
			\end{aligned}\end{equation}
			%The loss function $g(\theta)\in C^{2}$   is a non-negative function, satisfying the following  conditions.
			%\item There is a unique real-valued vector $\theta^*$, satisfying $\nabla_{\theta^{*}}g(\theta^{*})=0,$ 
			%and $\exists \delta>0$ such that $\inf_{\theta\notin\mathring{U}(\theta^{*},\delta)}\big\|\nabla_{\theta}g(\theta)\big\|>0.$
			%\item The second derivative of $g(\theta)$ exists. Moreover, there is a postive scalar $c$, such that $\forall \theta\in \mathbb{R}^{N}$, $\lambda_{\max}(H_{\theta\theta})<c$, where $H_{\theta\theta}$ is the Hessian Matrix of $g(\theta)$.
			%\item There exist two constants $s,t>0$, such that
			%\begin{equation}\label{naue}
			%\begin{aligned}
			%s<\liminf_{\theta\rightarrow {\theta^{*}}} \frac{\|\nabla_{\theta}g(\theta)\|^{2}}{g(\theta)-g(\theta^{*})}\le\limsup_{\theta\rightarrow {\theta^{*}}} \frac{\|\nabla_{\theta}g(\theta)\|^{2}}{g(\theta)-g(\theta^{*})}<r.
			%\end{aligned}
			%\end{equation}
		\end{enumerate}
	\end{assum}
 Assumption~\ref{ass_noise} and conditions 1)--3) of Assumption~\ref{ass_g} are common in the literature \citep{2019Understanding}.
	Assumption~\ref{ass_g} does not pose  any requirement on the convexity of $g(\theta)$. In other words, we allow any convex or non-convex loss functions $g(\theta)$ as long as they satisfy this assumption.
	 Condition  4) corresponds to the condition in  \citet{Shalev2007Pegasos,NemirovskiRobust,hazan2014beyond,2019Understanding,yang2016unified,1977Comparison,YM1983Limit} where the following inequality is assumed to hold  for any $\theta\in \mathbb{R}^{N}$ and positive integer  $ n$, 
	\begin{equation}\label{0908_exist}
	\begin{aligned} 
	\Expect_{\xi_{n}}\Big(\big\|\nabla_{\theta}g(\theta)-\nabla_{\theta}g(\theta,\xi_{n})\big\|^{2}\Big)\le K,\ \forall \theta\in \mathbb{R}^{N},
	\end{aligned}\end{equation}
	where $K$ is a positive scalar. Note that \eqref{0908} reduces to \eqref{0908_exist} if $g(\theta)$ is uniformly upper bounded over the space $ \mathbb{R}^{N}$. Since \eqref{0908} does not need this uniform boundedness, it  substantially extends \eqref{0908_exist} such that mSGD is also applicable to the scenarios with unbounded loss functions.  
	%\sout{As pointed as   Lam. Nguyen \citet{Nguyen2018SGD}, this condition which need the stochastic gradient is uniformly bounded can not be fitted in many model.}
	
	Different from deterministic GD-type algorithms with a constant step size, in order to ensure the convergence of mSGD, we need a decreasing step size for   counteracting the randomness induced by  noise $\{\xi_{n}\}$ \citep{2019Understanding,1951A}. Specifically, we make the following assumption on step size $ \epsilon_{n}$ together with a fixed momentum coefficient $\alpha$.
	
	%\sout{For the momentum coefficient $\alpha$ in \eqref{mSGD}, there are much less restricted. People like to use a fixed number which is close to $1$ (like $0.9$ or $0.95$)  on actual applications} \citet{Smith2018Don,Sutskever2013On}.

	%\begin{assum}\label{ass2}
	%	For any   vector $\theta\in \mathbb{R}^{N}$, there is a   scalar $M>0$  such that
	%	\begin{equation}\begin{aligned}
	%	&\Expect\bigg(\big\|\nabla_{\theta}g(\theta)-\nabla_{\theta}g(\theta,\xi_{n})\big\|^{2}\bigg)\le M\bigg(1+\Expect\big(g(\theta)\big)\bigg).
	%	\end{aligned}\end{equation}
	%\end{assum}
	
	%%					where 
	%%					This equation makes $g(\theta_{n},\xi_{n})$ retain some characteristics of $g(\theta_{n})$.
	%				  where $\{\xi_{n}\}$ is a vector sequence consisting of independent random vectors and $\{\xi_{n}\}$ is independent of $\{\theta_{1}\}$. $g(\theta_{n}, \xi_{n})$ is a stochastic observation of $g(\theta_{n})$  assumed to satisfy the following condition
	
	\begin{assum}\label{assum8}
		Momentum coefficient $\alpha\in [0,1)$ and 
		the sequence of   step size $\epsilon_{n}$ is positive,  monotonically decreasing to zero, such that
		$	\sum_{n=1}^{+\infty}\epsilon_{n}=+\infty$ and 
		$\sum_{n=1}^{+\infty}\epsilon_{n}^{2}<+\infty.$
	\end{assum}
	The setting of $\epsilon_{n}$ in Assumption~\ref{assum8} is consistent with   stochastic approximation for  root seeking of functions \citep{chen2006stochastic} as well as some stochastic optimization algorithms \citep{1972Stochastic,1977Comparison,YM1983Limit,2019Understanding}. 
	An explicit example of step size $\epsilon_{n}$ satisfying Assumption~\ref{assum8} is   $\epsilon_{n}=1/n$. 
	Although   SHB in \eqref{SHB} shares a similar expression as mSGD, the provided conditions of step sizes in 
	\citet{1972Stochastic,1977Comparison,YM1983Limit,2019Understanding} are not applicable to the general cases of mSGD with static $\alpha$ \citep{Smith2018Don,sutskever2013importance}.

	{
		In fact, SHB in \eqref{SHB} has a similar expression as mSGD in  \eqref{mSGD}. 
		 Multiplying $\gamma_{n}$ on both sides of the first equation of \eqref{SHB} yields 
		\begin{equation} \label{kliop}
		\begin{aligned}
		&\gamma_{n}v_{n}=\frac{\gamma_{n}}{\gamma_{n-1}}\beta_{n}(\gamma_{n-1} v_{n-1})+\gamma_{n}(1-\beta_{n})\nabla_{\theta_n} g(\theta_{n} ,\xi_{n})\\&\theta_{n+1}=\theta_{n}-\gamma_{n}v_n.
		\end{aligned} 
		\end{equation}We treat $\gamma_{n}v_{n}$ as one term, as the role of $v_{n}$ in \eqref{mSGD}. By comparing coefficients of \eqref{kliop} and \eqref{mSGD}, we get $\alpha^{(n)}=(\gamma_{n}/\gamma_{n-1})\beta_{n}$ and $\epsilon_{n}=\gamma_{n}(1-\beta_{n})$. 
			In the literature  \citep{1972Stochastic,1977Comparison,YM1983Limit,2019Understanding},
	the parameter setting of SHB in \eqref{SHB}  has the following requirements
% 	The setting of $\alpha$ and $\epsilon_{n}$ in this paper is  milder and  closer to practical applications than the setting  from \citet{1972Stochastic,1977Comparison,YM1983Limit,2019Understanding} for  the SHB in \eqref{SHB}, which is 
		\begin{equation}\nonumber\begin{aligned}
		&\sum_{n=1}^{+\infty}\gamma_{n}=+\infty, \quad 
		\sum_{n=1}^{+\infty}\gamma_{n}^{2}<+\infty,
		\quad \beta_{n}<1, \quad
		\lim_{n\rightarrow+\infty}\beta_{n}=0
		\\&or
		\\&\sum_{n=1}^{+\infty}\gamma_{n}=+\infty,\quad \sum_{n=1}^{+\infty}(1-\beta_{n})^{2}<+\infty,\quad \sum_{n=1}^{+\infty}\frac{\gamma_{n}^{2}}{1-\beta_{n}}<+\infty,\quad \lim_{n\rightarrow+\infty}\beta_{n}=1.\end{aligned}\end{equation}
		According to the above requirements on $\beta_n$ and $\gamma_n$, the requirements  on $\epsilon_n$ and $\alpha^{(n)}$ are
		\begin{equation}\nonumber\begin{aligned}	
		&\sum_{n=1}^{+\infty}\epsilon_{n}=+\infty, \quad 
		\sum_{n=1}^{+\infty}\epsilon_{n}^{2}<+\infty, \quad \lim_{n\rightarrow+\infty}\alpha^{(n)}=0.
		\\& or
		\\&\sum_{n=1}^{+\infty}\epsilon_{n}=\sum_{n=1}^{+\infty}\gamma_{n}(1-\beta_{n})\le\sqrt{\sum_{n=1}^{+\infty}(1-\beta_{n})^{3}\sum_{n=1}^{+\infty}\frac{\gamma_{n}^{2}}{1-\beta_{n}}}<+\infty, \quad \limsup_{n\rightarrow+\infty}\alpha^{(n)}=1.
		\end{aligned}\end{equation}
	{ We can see that the  static momentum parameter $\alpha$, which is widely used in practical applications, does not satisfy these conditions.} { In \cite{2020Almostuiuiyuyiuy}, the authors studied an algorithm (SHB-IMA) that has a similar form to mSGD given in \eqref{mSGD}, but their conditions for parameters cannot cover the case with a static  $\alpha$.}
% 		No matter which type $\epsilon_{n}$ and $\alpha$ are, can not fit actual (mSGD) applications.
		}
	
	Before providing the main theorem for convergence, we need a useful lemma
	In the analysis of asymptotic convergence of mSGD, the following lemma plays an important role.
	\begin{lem}\label{lem_bound}
		%			Suppose $\{\theta_{n}\}$ is a sequence of estimates generated by \eqref{mSGD}. 
		Consider the mSGD in \eqref{mSGD}.
		If Assumptions~1--3 hold, then for $\forall \theta_{1}\in \mathbb{R}^{N}$ and $v_{0}\in \mathbb{R}^{N}$, there is a scalar $T(\theta_{1},v_{0})$, such that 
		$	\Expect\big(g(\theta_{n})\big)<T(\theta_{1},v_{0})$ for any $n\geq 1$.
	\end{lem}
	Lemma~\ref{lem_bound} actually guarantees   stability of mSGD.
	Let $\mathcal{F}_{n}=\sigma(\theta_1,v_0,\{\xi_{i}\}_{i=1}^n)$ be the minimal $\sigma$-algebra generated by $\theta_1,v_0,\{\xi_{i}\}_{i=1}^n.$ As a result, $\theta_n$ is adapted to $\mathcal{F}_{n-1}$.
	Then for any  $n\in\mathbb{ N}_+$, it holds that
	\begin{align}\label{eq_derive_bound}
	&\Expect\Big(\big\|\nabla_{\theta_{n}}g(\theta_{n})-\nabla_{\theta_{n}}g(\theta_{n},\xi_{n})\big\|^{2}\Big)
	=\Expect\bigg(\Expect\Big(\big\|\nabla_{\theta_{n}}g(\theta_{n})-\nabla_{\theta_{n}}g(\theta_{n},\xi_{n})\big\|^{2}\Big|\mathcal{F}_{n-1}\Big)\bigg)\nonumber\\
	=&\Expect\bigg(\Expect_{\xi_{n}}\Big(\big\|\nabla_{\theta_{n}}g(\theta_{n})-\nabla_{\theta_{n}}g(\theta_{n},\xi_{n})\big\|^{2}\Big)\bigg)
	\le M\Big(1+\Expect\big(g(\theta_{n})\big)\Big)\le M(1+T(\theta_{1},v_{0})),
	\end{align}
	where the second equality follows from the   independence between  $\xi_{n}$ and $\{\theta_1,v_0,\{\xi_{i}\}_{i=1}^{n-1}\}$ in Assumption~\ref{ass_noise}, and the last two inequalities hold due to 4) in Assumption~\ref{ass_g} and Lemma~\ref{lem_bound}, respectively.
	Intuitively, the  result in \eqref{eq_derive_bound}  means that the fluctuation induced by random noise $\{\xi_{i}\}_{i=1}^{n-1}$ is well restrained. Note that the derived result in \eqref{eq_derive_bound} is totally different from   \eqref{0908_exist} required in the literature \citep{Shalev2007Pegasos,NemirovskiRobust,hazan2014beyond,2019Understanding,yang2016unified,1977Comparison,YM1983Limit}, since \eqref{0908_exist} needs a uniformly upper bound $K$ over the whole space (i.e., $\theta\in \mathbb{R}^{N}$) which is difficult to satisfy when loss function $g(\theta)$ is quadratic or cubic with respect to $\theta$ over unbounded parameter space.
	In contrast, regardless of the order of $g(\theta)$ with respect to $\theta$, \eqref{eq_derive_bound} ensures   boundedness of $\Expect\Big(\big\|\nabla_{\theta_{n}}g(\theta_{n})-\nabla_{\theta_{n}}g(\theta_{n},\xi_{n})\big\|^{2}\Big)$ for learning any fixed true parameter $\theta^*.$ This reflects a favorable learning process against random noise $\xi_n$ for dealing with general loss functions $g(\theta)$.  
	This result paves the way to  the following theorem on asymptotic convergence of mSGD.
	\begin{thm}\label{thm_converg1}
		Consider the mSGD in \eqref{mSGD}. If Assumptions~1--3 hold, then   for $\forall \theta_{1}\in \mathbb{R}^{N}$ and $\forall v_{0}\in \mathbb{R}^{N}$, there exists a connected    set $J^{*}\subseteq J$ such that the {iterate} $\theta_n$ is convergent to the set $J^{*}$ almost surely, i.e., 
		\begin{equation}\nonumber\begin{aligned}
		\lim_{n\rightarrow\infty} d(\theta_{n},J^{*})=0,
		%			\inf_{\theta\in J^{*}}\big\|\theta_{n}-\theta\|\stackrel{n\rightarrow\infty}{\longrightarrow}0, 
		\qquad \text{a.s.}\end{aligned}\end{equation}
		where $d(x,J^{*})=\inf_{y}\{\|x-y\|,y\in J^{*}\} $  denotes the distance between point $x$ and set $J^{*}.$
	\end{thm}
	In Theorem 1, we prove that the iterates of mSGD  asymptotically converge to a connected set of stationary points almost surely.   When this connected set degenerates to a stationary point $\theta^{*}$, it holds that $\theta_{n}\rightarrow \theta^{*}, \text{ a.s.}$.
	%Our results is more profound and more closest to practical application than classical works and need more mild conditions. 
	The result enables 	engineers for the design of proper  momentum coefficients    and step sizes (like $\alpha=0.9$, $\epsilon_{n}=0.1/n$) in the related applications of mSGD with mathematic guarantee. 
% 	Some practical implications and guidelines for setting step sizes and momentum parameters in practical machine learning applications are presented in Section~4.
	\subsection{Convergence rate of mSGD}
	In this subsection, we analyze the convergence rate of mSGD. Before that, given  positive real sequences $\{a_{n}\}$ and $\{b_{n}\}$, we let $a_n=O(b_n)$ if there is a constant $c>0$, such that $a_n/b_n<c$ for any $n\geq1.$ 
	For quantitative analysis, we need some extra assumptions.
	\begin{assum}\label{ass_goiop}(\textbf{Loss function assumption}) Loss function $g(\theta)$ satisfies the following conditions:
		\begin{enumerate}[1)]
			\item  $g(\theta)$ is a non-negative and  continuously differentiable function. The set of its stationary points $J=\{\theta|\|\nabla_{\theta}g(\theta)\|=0\}$ is a bounded set which has only finite connected components $J_{1},...,J_{n}$. In addition,  there is {$\epsilon'>0$}, such that for any $i\in\{1,2,\dots,n\}$ and ${0<d(\theta,J_{i} )<\epsilon'}$, it holds that $\big|g(\theta)-g_{i}\big|\neq 0$, where $g_{i}=\{g(\theta)|\theta\in J_{i}\}$ is a constant.
			\item  For any $i\in\{1,2,\dots,n\}$, it holds that
			\begin{equation}\nonumber
		\begin{aligned}
		&\liminf_{d(\theta,J_{i})\rightarrow 0}\frac{\|\nabla_{\theta}g(\theta)\|^{2}}{g(\theta)-g_{i}}\geq s>0.
		\end{aligned}
		\end{equation}
		\end{enumerate}
	\end{assum}

{	  In many problems of machine learning, especially  deep learning, because of strong non-linearity mapping from input data to output data and the structure complexity of employed models,    loss functions could be non-convex and may have multiple local critical points.} 
	Assumption \ref{ass_goiop} does not require convexity of the loss function, but guarantees the properties of the loss function around critical points. { Assumption 4 2) can be treated as a local version of Polyak-Lojasiewicz (P-L) condition.} 
	\begin{thm}\label{thm_rate}
	Consider the mSGD in \eqref{mSGD} with the noise following a uniform sampling distribution. If Assumptions \ref{ass_noise}-\ref{ass_goiop} hold, for $\forall v_{0}\in \mathbb{R}^{N}$ and $\forall \theta_{1}\in \mathbb{R}^{N}$, it holds that
		\begin{equation}\label{o,.;l}
		\begin{aligned}
		&\Expect\big(\|\nabla_{\theta_{n}}g(\theta_{n})\|^{2}\big)=O\Big(e^{-\sum_{i=1}^{n}\frac{s\epsilon_{i}}{p(1-\alpha)^{2}}}\Big),
		\end{aligned}
		\end{equation}where $p=\exp\bigg\{\sum_{k=1}^{\infty}M\epsilon_{k}^{2}\bigg\}$ and $M$ is defined in  condition 4) of Assumption \ref{ass_g}.
	\end{thm}
	In Theorem \ref{thm_rate}, let $\alpha=0$, then we can obtain the convergence rate of SGD,
	\begin{equation}\nonumber
	\begin{aligned}
	&\Expect\big(\|\nabla_{\theta_{n}}g(\theta_{n})\|^{2}\big)=O\Big(e^{-\sum_{i=1}^{n}\frac{s}{p}\epsilon_{i}}\Big).
	\end{aligned}
	\end{equation}
{ According to the obtained bounds, the convergence rate of mSGD with $\alpha\in (0,1)$ is larger than that of SGD, which is the case with $\alpha=0$.} Moreover, Theorem \ref{thm_rate} provides a stronger characterization of convergence rate than some existing works considering the convergence rate of time averages $\frac{1}{T}\sum_{i=1}^{T}\Expect(\|\nabla_{\theta_{n}}g(\theta_{n})\|^{2})=O(T^{l})$ ($l=-1$ in \citet{2020An} and $l=-1/2$ in \citet{2020Convergence}). In the following, we will elaborate on this point. From \eqref{o,.;l},   there exists a scalar $t_0>0$ such that  $\forall n>0$
	\begin{equation}\nonumber
		\begin{aligned}
		&\frac{\Expect\big(\|\nabla_{\theta_{n}}g(\theta_{n})\|^{2}\big)}{\exp{\Big\{-\sum_{i=1}^{n}\frac{s\epsilon_{i}}{p(1-\alpha)^{2}}\Big\}}}<t_0,
		\end{aligned}
		\end{equation}
		implying
		\begin{equation}\nonumber
		\begin{aligned}
		&\frac{\frac{1}{T}\sum_{n=1}^{T}\Expect\big(\|\nabla_{\theta_{n}}g(\theta_{n})\|^{2}\big)}{\frac{1}{T}\sum_{n=1}^{T}\exp{\Big\{-\sum_{i=1}^{n}\frac{s\epsilon_{i}}{p(1-\alpha)^{2}}\Big\}}}<t_0.
		\end{aligned}
		\end{equation}So it holds that
		\begin{equation}\nonumber
		\begin{aligned}
		&\frac{1}{T}\sum_{n=1}^{T}\Expect\big(\|\nabla_{\theta_{n}}g(\theta_{n})\|^{2}\big)=O\bigg(\frac{1}{T}\sum_{n=1}^{T}e^{-\sum_{i=1}^{n}\frac{s\epsilon_{i}}{p(1-\alpha)^{2}}}\bigg).
		\end{aligned}
		\end{equation}
		
% 		Next, we will show that it is feasible to ensure that $\forall l<0$, there is
% 		\begin{equation}\nonumber
% 		\begin{aligned}
% 		&\frac{1}{T}\sum_{n=1}^{T}e^{-\sum_{i=1}^{n}\frac{s\epsilon_{i}}{p(1-\alpha)^{2}}}<T^{l},
% 		\end{aligned}
% 		\end{equation}when $T$ is sufficient large by designing proper parameter $\alpha$ and step size $\epsilon_{n}$. 
        Let $\epsilon_{n}=\frac{1}{n}$, then we know that $\sum_{n=1}^{T} \epsilon_{n} = O(\ln T) $. Hence,
		\begin{equation}\nonumber
		\begin{aligned}
		&\frac{1}{T}\sum_{n=1}^{T}e^{-\sum_{i=1}^{n}\frac{s\epsilon_{i}}{p(1-\alpha)^{2}}}=O\bigg(\frac{1}{T}\sum_{n=1}^{T}e^{-\frac{s\ln (n)}{p(1-\alpha)^{2}}}\bigg)=O\bigg(\frac{1}{T}\sum_{n=1}^{T}n^{-\frac{s}{p(1-\alpha)^{2}}}\bigg).
		\end{aligned}
		\end{equation}
		If $\frac{s}{p(1-\alpha)^{2}} < 1$, $\sum_{n=1}^{T}n^{-\frac{s}{p(1-\alpha)^{2}}}=O\Big(T^{-\frac{s}{p(1-\alpha)^{2}}+1}\Big)$, so we have
		\begin{equation}\nonumber
		\begin{aligned}
		&\frac{1}{T}\sum_{n=1}^{T}e^{-\sum_{i=1}^{n}\frac{s\epsilon_{i}}{p(1-\alpha)^{2}}}=\frac{1}{T}O\Big(T^{-\frac{s}{p(1-\alpha)^{2}}+1}\Big)=O\Big(T^{-\frac{s}{p(1-\alpha)^{2}}}\Big).
		\end{aligned}
		\end{equation}
		If $\frac{s}{p(1-\alpha)^{2}}=1$, it holds that $\sum_{n=1}^{T}n^{-\frac{s}{p(1-\alpha)^{2}}} =O(\ln T)$, and hence
		\begin{equation}\nonumber
		\begin{aligned}
		&\frac{1}{T}\sum_{n=1}^{T}e^{-\sum_{i=1}^{n}\frac{s\epsilon_{i}}{p(1-\alpha)^{2}}}=\frac{1}{T}O (\ln T )=O\Big(\frac{\ln T}{T}\Big).
		\end{aligned}
		\end{equation}
		If $\frac{s}{p(1-\alpha)^{2}}>1$, then $\sum_{n=1}^{T}n^{-\frac{s}{p(1-\alpha)^{2}}} := C_{\alpha}(T) \to C_\alpha$ as $T \to +\infty$, where $C_{\alpha}$ is a positive constant, so
		\begin{equation}\nonumber
		\begin{aligned}
		&\frac{1}{T}\sum_{n=1}^{T}e^{-\sum_{i=1}^{n}\frac{s\epsilon_{i}}{p(1-\alpha)^{2}}} = \frac{C_{\alpha}(T)}{T} = O\Big(\frac{C_\alpha}{T}\Big).
		\end{aligned}
		\end{equation}
		Note that the coefficient $C_\alpha$ depends on $s$, $p$, and $\alpha$, where $s$ is the constant given in Assumption~\ref{ass_goiop}   only depending on $g(\theta)$, and $p$ is given in Theorem~\ref{thm_rate} and relies on $M$ in 4) of Assumption~\ref{ass_g} and step size $\epsilon_{n}$. 
		The above observation indicates that setting $\alpha$ close to one makes mSGD achieve convergence rate of time averages of order $O(1/T)$. As remarked previously, larger $\alpha$ implies smaller coefficient $C_\alpha$ and thus quicker convergence rate. Interestingly, since $C_\alpha(T)$ is monotonically increasing, the convergence rate has a relatively small coefficient when $T$ is small. This could illustrate why mSGD can achieve   better performance  in the early phase of iteration.

% 	We will show some practical machine learning applications in Section 4.
	
	%\section{Proof of Main results}		
	%In this section, we provide the proofs of the main results of this paper. Before that,  some   lemmas and their proofs are presented.

	%\subsection{Useful lemmas}

	\subsection{Convergence of AdaGrad}
	In this subsection, we aim to establish the convergence of AdaGrad. Compared to the study of mSGD, the design of adaptive step size increases the technical difficulties. To proceed,  the required conditions   on   loss function $g(\theta)$ are summarized as follows. 
	\begin{assum}\label{ass_g_poi}Loss function $g(\theta)$   in \eqref{AdaGrad} satisfies the following conditions:
		%	The following conditions hold for $g(\theta)$ defined in \eqref{AdaGrad}:
		\begin{enumerate}[1)]
			\item $g(\theta)$ is a non-negative and  continuously differentiable function. The set of its stationary points $J=\{\theta|\|\nabla_{\theta}g(\theta)\|=0\}$ is a bounded set which has only finite connected components $J_{1},...,J_{n}$. In addition,  there is $\epsilon_{1}>0$, such that for any $i$ and $0<d(\theta,J_{i} )<\epsilon_{1}$, it holds that $\big|g(\theta)-g_{i}\big|\neq 0$, where $g_{i}=\{g(\theta)|\theta\in J_{i}\}$ is a constant.
			
% 			Since $g$ is con, we can get that $\{g(\theta)|\theta\in J_{i}\}$ is a constant, we defined this constant is $g_{i}$.
% 			And then we need a condition that $\exists \ \epsilon_{1}$, $\forall i$, when $d(\theta,J_{i} )<\epsilon_{1}$, there is $\big|g(\theta)-g_{i}\big|\neq 0$.
			
			\item The gradient $\nabla_{\theta}g(\theta)$ satisfies the Lipschitz condition, i.e., for any $x,y\in \mathbb{R}^{N}$,
			\begin{equation}\nonumber\begin{aligned}
			\big\|\nabla_{x}g(x)-\nabla_{y}g(y)\big\|\le c\|x-y\|.
			\end{aligned}\end{equation}

			\item 
			There are two constants $M'>0$ and $a>0$ such that for any $\theta\in \mathbb{R}^{N}$ and $n\in \mathbb{N}_{+}$, 
			\begin{equation}\label{0908_2}\begin{aligned}
			&E_{\xi_{n}}\Big(\big\|\nabla_{\theta}g(\theta,\xi_{n})\big\|^{2}\Big)\le M'\big\|\nabla_{\theta}g(\theta)\big\|^{2}+a.
			\end{aligned}\end{equation}
		\end{enumerate}
	\end{assum}
	%It is known that the P-L condition in 4) of Assumption \ref{ass_g_poi} is satisfied for  strong convex loss functions. 
	%Moreover, some non-convex functions also satisfy the P-L condition \citet{zhang2015restricted,2016Linear}).  
	Condition 2) is the same as in 2) of Assumption~\ref{ass_g}. {Condition 1) is relatively weak, because it does not require any convexity of the loss function or  global conditions as P-L condition.}
% 	Assumption~\ref{ass_g_poi} does not need any convexity of loss function, and can be satisfied by general loss functions in machine learning.
	{
	There are many functions satisfying Assumption 5 1) but not convex, such as $y=\sin^{2}{(x)}$, $y=(x-1)(x-2)(x-3)(x-4)$, and $y=\cos^{2}{(x)}$}. Similar to condition 4) of Assumption 2, condition 3) is a condition to restrain the noise influence.   \Eqref{0908_2} is milder than $\|\nabla_{\theta}g(\theta,\xi_{n})-\nabla_{\theta}g(\theta)\|\le S\ \ a.s.$ \citep{2018On,A2020A}, which is relatively restrictive in dealing with unbounded noises, e.g., $\nabla_{\theta}g(\theta,\xi_{n})=\nabla_{\theta}g(\theta)+\xi_{n}$, where $\xi_{n}$ is  independent identically distributed and Gaussian.

	%%But unlike (essentially or weakly) strong convexity, the P-L condition alone does not imply convexity of $g$ (\citet{zhang2015restricted,2016Linear}). 
	%Our condition on $\|\nabla_{\theta}g(\theta,\xi_{n})\|^{2}$ is weaker than some existing works like $\|\nabla_{\theta}g(\theta,\xi_{n})-\nabla_{\theta}g(\theta)\|\le S\ \ a.s.$ in \citet{2018On,A2020A} which can not fit even the situation like $\nabla_{\theta}g(\theta,\xi_{n})=\nabla_{\theta}g(\theta)+\lambda_{n}$ ($\{\lambda_{n}\}$ is an independent identically distributed Gaussian sequence).
	
	Compared to SGD and mSGD, there are   more challenges in analyzing the convergence of AdaGrad. The challenges mainly come from two aspects: (1) since AdaGrad does not have a  decreasing step size, the noise influence on AdaGrad cannot be restrained as mSGD  which is with the help of decreasing step size satisfying Assumption \ref{assum8};
	(2)   adaptive step size of AdaGrad in \eqref{AdaGrad} (i.e., $\alpha_{0}/\sqrt{\sum_{i=1}^{n}\|\nabla_{\theta_{i}}g(\theta_{i},\xi_{i})\|^{2}}$) is a random variable   conditionally dependent of $\xi_{n}$ given   $\sigma\{\theta_{1},\ \xi_{1},...,\ \xi_{n-1}\}$. Then when we deal with  terms in the proof like
	\begin{equation}\label{qazse}
	\begin{aligned}
	&\frac{\alpha_{0}}{\sqrt{S_{n}}}\nabla_{\theta_{n}}g(\theta_{n})^{T}\nabla_{\theta_{n}}g(\theta_{n},\xi_{n})\quad \text{ or }\quad \frac{\alpha_{0}^{2}}{S_{n}}\big\|\nabla_{\theta_{n}}g(\theta_{n},\xi_{n})\big\|^{2},
	\end{aligned}
	\end{equation}we cannot make the conditional expectation to transform \eqref{qazse} to 
	\begin{equation}\nonumber
	\begin{aligned}
	&\frac{\alpha_{0}}{\sqrt{S_{n}}}\big\|\nabla_{\theta_{n}}g(\theta_{n})\big\|^{2}\quad \text{ or } \quad \frac{\alpha_{0}^{2}}{S_{n}}E_{\xi_{n}}\Big(\big\|\nabla_{\theta_{n}}g(\theta_{n},\xi_{n})\big\|^{2}\Big).
	\end{aligned}
	\end{equation}  In the literature, \citet{2018On}  changed the step size to
	\begin{equation}\nonumber\begin{aligned}
	\frac{\alpha_{0}}{\sqrt{\sum_{i=1}^{n-1}\|\nabla_{\theta_{i}}g(\theta_{i},\xi_{i})\|^{2}}},\end{aligned}\end{equation} {and \cite{2020Asymptotic} changed the step size to ${\gamma_{n+1}}/{\sqrt{\omega_{n}+\epsilon}}$ where 
	\begin{equation}\nonumber\begin{aligned}
	\omega_{n}&=\omega_{n-1}+\gamma_{n}\big(p_{n}\|\nabla_{\theta_{n}}g(\theta_{n},\xi_{n})\|^{2}-q_{n}\omega_{n-1}\big),
% 	=&h(\|\nabla_{\theta_{n-1}}g(\theta_{n-1},\xi_{n-1})\|^{2},\|\nabla_{\theta_{n-2}}g(\theta_{n-2},\xi_{n-2})\|^{2},...,\|\nabla_{\theta_{1}}g(\theta_{1},\xi_{1})\|^{2})
	\end{aligned}\end{equation}
	and $\gamma_n, p_n, q_n$ are tuned parameters}, in order 
	to make it     conditionally independent of $\xi_{n}$ given  $\sigma\{\theta_{1},\ \xi_{1},...,\ \xi_{n-1}\}$, which however is no longer the standard AdaGrad as in \eqref{AdaGrad} \citep{2010Less,2018Onp}.
	
	Before we provide the main theorem of this subsection, a useful lemma is worth discussing.
	\begin{lem}\label{lem4}
		Let $\{\alpha_{k}\}$, $\{\beta_{k}\}$ being non-negative random variable sequences, such that the following conditions hold almost surely
		\begin{itemize}
			\item $\sum_{k=1}^{+\infty}\alpha_{k}=\infty$;
			\item $\sum_{k=1}^{+\infty}\alpha_{k}\beta_{k}<+\infty$;
			%\item
			%$\beta_{k}\le \gamma_{k}$ and $\lim_{\epsilon\rightarrow 0^{+}}\sup_{\{k:\beta_{k}<\epsilon\}}\gamma_{k}=0$;
			%\item There is  a finite non-negative random variable $h$  and a random sequence $u_{k}$   satisfying $\sum_{k=1}^{+\infty}u_{k}<\infty\ \ a.s.$, such that 	$\gamma_{k+1}-\gamma_{k}\le h\alpha_{k}+u_{k}$. 
		\end{itemize}	
		Then there exists a subsequence $\{\beta_{n_{k}}\}$ of $\{\beta_{k}\}$, such that  $\beta_{n_{k}}\stackrel{k\rightarrow\infty}{\longrightarrow}0$ almost surely. 
	\end{lem}
	
	\begin{proof} 
We aim to prove    $\liminf_{k\rightarrow+\infty }\beta_{k}=0$ by contradiction. 
	Suppose $\liminf_{k\rightarrow+\infty }\beta_{k}=u>0$, then   $\exists k_{0}$, such that $\forall k>k_{0}$,   $\beta_{k}>u/2$. It follows that
    \begin{equation}\label{007654321}\begin{aligned}
    \sum_{k=k_{0}+1}^{+\infty}\alpha_{k}<\frac{2}{u}\sum_{k=k_{0}+1}^{+\infty}\alpha_{k}\beta_{k}<+\infty.
    \end{aligned}\end{equation}Obviously, there is a contradiction between \eqref{007654321} and the condition $\sum_{k=1}^{+\infty}\alpha_{k}=+\infty$. Thus,  $\liminf_{k\rightarrow+\infty }\beta_{k}=0$. So we get that there exists a subsequence $\{\beta_{n_{k}}\}$ of $\{\beta_{k}\}$ holds $\beta_{n_{k}}\stackrel{k\rightarrow\infty}{\longrightarrow}0$ almost surely.
    \end{proof}
	This  lemma, inspired by   Proposition 2 in Ya. I. Alber \citep{1998On}, is quite useful  in
	the convergence analysis of AdaGrad. Then we are ready to provide the convergence result of AdaGrad in the following theorem.
	
	\begin{thm}\label{thm_converg_adagrad}
		Consider the AdaGrad in \eqref{AdaGrad}. If Assumptions \ref{ass_noise} and \ref{ass_g_poi} hold, then   for $\forall \theta_{1}\in \mathbb{R}^{N}$ and $S_{0}=0$, there exists a connected component of the set $J^{*}\subseteq J$, such that   the estimate $\theta_n$ is convergent to the set $J^{*}$ almost surely, i.e., 
		\begin{equation}\nonumber\begin{aligned}
		\lim_{n\rightarrow\infty} d(\theta_{n},J^{*})=0.
		\qquad \text{a.s.}\end{aligned}\end{equation}
	\end{thm}	
	In Theorem \ref{thm_converg_adagrad}, we prove the standard AdaGrad is convergent almost surely, in contrast to the convergence of time averages in \citet{2018Onp,2018AdaGrad} which focus on the metric $\frac{1}{T}\sum_{i=1}^{T}\Expect(\|\nabla_{\theta_{n}}g(\theta_{n})\|^{2})$, or the subsequence convergence in \citet{2019A,A2020A}, focusing on $\min_{n=0,1,2,...,T}\Expect(\|\nabla_{\theta_{n}}g(\theta_{n})\|)$.
	% or abbreviated AdaGrad in \citet{2018On}.
	% And  Our condition on noise do not need the pointwise bounded (\citet{2018On,A2020A}) or expected-sense bounded (\citet{2019A,2018Onp,2018AdaGrad}). In order to present our results vividly, we make some practical implications in Section 4.

	\section{Conclusion and Future Work}
	In this paper, we studied the convergence of two algorithms extensively used in machine learning applications, namely,  momentum-based stochastic gradient descent (mSGD) and  adaptive step stochastic gradient descent (AdaGrad).  
	By considering general loss functions (either convex or non-convex), we first establish the almost sure asymptotic convergence of mSGD. Moreover, we find the   convergence rate of the mSGD and reveal that the mSGD indeed has a faster convergence rate than the SGD.
	Furthermore,  we prove    AdaGrad is convergent almost surely under mild conditions.  
	Subsequence convergence and convergence of time averages in the literature are substantially extended in this work to asymptotic convergence. {To better understand the AdaGrad, we will study its convergence rate in the future.}

\newpage

\bibliographystyle{iclr2022_conference}
	\bibliography{references}

	\newpage

	\appendix
	
	\section{Appendix Outline}
	Section~\ref{append_mSGD} aims to verify the convergence results of mSGD. 
	Several auxiliary lemmas are first provided, followed by a proof outline for the main results of mSGD. Then the proofs of these lemmas are given.  Theorems~\ref{thm_converg1}-\ref{thm_rate} are proved in Sections~\ref{append_thm_converg} and~\ref{append_pf_thm_rate}, respectively.
	Section~\ref{append_pf_adgrad} aims to verify the convergence results of AdaGrad. 
	Some  auxiliary lemmas are given at the beginning, followed by a proof outline of the main results for AdaGrad in Section~\ref{pf_outline2}. Then we provide the proofs of the lemmas. Theorem~\ref{thm_converg_adagrad} is proved in Section~\ref{appen_thm_converg_adagrad}.

	\section{Convergence and convergence rate of mSGD}\label{append_mSGD}
The following lemmas are used for the proofs of Theorems~\ref{thm_converg1}-\ref{thm_rate}. 	Hereafter, $\mathbb{N}_+$ denotes the set of all positive integers.

A function $f$ is in a  (differentiability) class  $C^k$ if its $l$–th derivatives     exist and are continuous, where $l\leq k$.

% 	To make our proof more clear, we apart our proof into some useful lemmas. 
	\begin{lem}\label{lem_ggg}{(Lemma 1.2.3 in \cite{2004Introductory})}
		Suppose $f(x)\in C^{1}\ \ (x\in \mathbb{R}^{N})$ with gradient satisfying the following Lipschitz condition
		\begin{equation}\nonumber\begin{aligned}
		\|\nabla f(x)-\nabla f(y)\|\le c\|x-y\|,
		\end{aligned}
		\end{equation}then for any $x, y\in \mathbb{R}^{N}$,  it holds that
		\begin{equation}\nonumber\begin{aligned}
		f(y)+\nabla f(y)^{T}(x-y)-\frac{c}{2}\|x-y\|^{2}\le f(x)\le f(y)+\nabla f(y)^{T}(x-y)+\frac{c}{2}\|x-y\|^{2}.
		\end{aligned}
		\end{equation}
    \end{lem}

% Hessian and PL condition    
% Actually, the upper bound condition in Assumption \ref{ass_g_poi} 4) and Assumption \ref{ass_g_4} 2) is implied by Assumption \ref{ass_g} 3), as shown in the following lemma. But in order to make our conditions more clear, we explicitly assume it holds.  

%     \begin{lem}\label{lem_gggg}
% 		Suppose $f(x)\in C^{1}\ \ (x\in \mathbb{R}^{N})$, and $f$ has a global optimal point $f^{*}$ with $f(x)\geq f^{*}$. If the gradient of $f$ satisfies the Lipschitz condition
% 		\begin{equation}\nonumber\begin{aligned}
% 		\|\nabla f(x)-\nabla f(y)\|\le c\|x-y\|,
% 		\end{aligned}
% 		\end{equation}then  for any $x, y\in \mathbb{R}^{N}$, there is a scalar $\ t>0$ such that
% 		\begin{equation}\nonumber\begin{aligned}
% 		\|\nabla f(x)\|^{2}\le t(f(x)-f^{*}).
% 		\end{aligned}
% 		\end{equation}
%     \end{lem}

	%		Denote  $J$  the zero set of $\nabla_{\theta}g(\theta)=0$, i.e., $\nabla_{\theta}g(\theta)=0, \forall \theta\in J.$	
	\begin{lem}\label{pouikm,l}
		Suppose $f(x)\in C^{1}\ \ (x\in \mathbb{R}^{N})$ with gradient satisfying the following Lipschitz condition
		\begin{equation}\nonumber
			\begin{aligned}
		\big\|\nabla f(x)-\nabla f(y)\big\|\le c\|x-y\|, 
		\end{aligned}
		\end{equation}
		and the set $S=\{x|\nabla f(x)=0\}$ is bounded and only has finite connected components $\{S_{1},S_{2},\dots,S_{m}\}$. Furthermore, assume there exists $ \epsilon_{1}'>0$, such that for any $i=1,2,...,m$ and $x\in\{x|0<d(x,S_{i})<{\epsilon_{1}'\}}$, it holds that $\big|f(x)-f_{i}\big|\neq 0$, where  $f_{i}=f(x)$ for $x\in S_{i}$. 
		Then for any $i=1,2,...,m$, if 	there is $ {\epsilon_{0}'}>0$ satisfying $d(x,S_{i})<{\epsilon_{0}'}$, it follows that
		\begin{equation}\nonumber\begin{aligned}
		\big\|\nabla f(x)\big\|^{2}\le{2c}\big|f(x)-f_{i}\big|.
		\end{aligned}
		\end{equation}
% 		and due to $f(x)$ maintain a constant in every connected component, so we defined . And we have a condition that 
% 		Then we can conclude the conclusion that 
    \end{lem}
	\begin{lem} \citep{wang2019almost} \label{lem_summation_MDS}
		% 	Suppose that $\{X_{n}\}\in \mathbb{R}^{N}$ is a martingale difference sequence, then it holds that $\sum_{k=0}^{\infty}X_k<+\infty \ a.s.,$ if   
		% 	$\sum_{k=0}^{\infty}\Expect\big(\|X_k\|^2\big)<+\infty.$		
		Suppose that $\{X_{n}\}\in \mathbb{R}^{N}$ is an $\mathcal{L}_2$ martingale difference sequence, and $(X_{n},\mathscr F_{n})$ is an adaptive process. Then it holds that $\sum_{k=0}^{\infty}X_k<+\infty \ a.s.,$ if 
		\begin{equation}\nonumber\begin{aligned}
		\sum_{n=1}^{\infty}\Expect(\|X_{n}\|^{2})<+\infty,	 \quad   \text{ or }  \quad  \sum_{n=1}^{\infty}\Expect\big(\|X_{n}\|^{2}\big|\mathscr{F}_{n-1}\big)<+\infty. \quad  a.s.
		\end{aligned}
		\end{equation}
	\end{lem}		
	%		This lemma is a fundamental conclusion on the weighted sums of the martingale difference \citet{miao2013almost}.	
	\begin{lem}\label{lem_summation}
		Suppose that $\{X_n\}\in \mathbb{R}^{N}$ is a non-negative sequence of random variables, then it holds that $\sum_{n=0}^{\infty}X_n<+\infty \ a.s.,$ if   
		$\sum_{n=0}^{\infty}\Expect\big(X_n\big)<+\infty.$   			
	\end{lem}

	%
	%\begin{lem}\label{lem_bound}
	%	%			Suppose $\{\theta_{n}\}$ is a sequence of estimates generated by \eqref{mSGD}. 
	%	Consider the mSGD in \eqref{mSGD}.
	%	If  2) in Assumption 1, and Assumptions 3, 4 hold, then   for $\forall \theta_{1}\in \mathbb{R}^{N}$ and $\forall v_{0}\in \mathbb{R}^{N}$, there is a $  T>0$, such that\begin{equation}\nonumber
	%	\begin{aligned}
	%	\Expect\big(g(\theta_{n})\big)<T<+\infty,\qquad \forall n\geq 1.
	%	\end{aligned}
	%	\end{equation}
	%\end{lem}

	\begin{lem}\label{lem_mid}
		Suppose $\{v_{n}\}$ is a sequence  generated by mSGD in  \eqref{mSGD}. Under Assumptions~ \ref{ass_noise}--\ref{assum8}, it holds that  $\sum_{n=0}^{+\infty}\Expect\big(\|v_n\|^{2}\big) < c(v_{0},\theta_{1})$,   and $\sum_{n=0}^{+\infty}\|v_n\|^{2} < +\infty\ \ a.s.$, where $c(v_{0}, \theta_{1})$ is a constant only related to $v_{0}$ and $\theta_{1}$.
	\end{lem}
	\begin{lem}\label{lem_deriv}
		Suppose $\{\theta_{n}\}$ is a sequence  generated by mSGD in  \eqref{mSGD}. Under Assumptions \ref{ass_noise}--\ref{assum8}, it holds that: for $n\geq 1,$
		\begin{align*}
		&\sum_{t=1}^{n}\epsilon_{t}\Expect\big(\|\nabla_{\theta_{t}}g(\theta_{t})\|^{2}\big)<B(v_{0}, \theta_{1})<+\infty, \qquad
		 \sum_{t=1}^{n}\epsilon_{t}\|\nabla_{\theta_{t}}g(\theta_{t})\|^{2}<+\infty.
		\end{align*}
		where $B(v_{0}, \theta_{1})>0$ is a constant only related to $v_{0}$ and $\theta_{1}$.
	\end{lem}
	
	\begin{lem}\label{lem_conve}
		Suppose   $\{\theta_n\}$ is a sequence generated by mSGD in  \eqref{mSGD} such that $\{g(\theta_n)\}$ converges a.s.. If Assumptions \ref{ass_noise}--\ref{assum8} hold, and  
		\begin{equation}\label{yuandian}\begin{aligned}&g(\theta_{n+1})\le \zeta_{n}-b\sum_{t=1}^n\epsilon_t\big\|\nabla g(\theta_t)\big\|^2 \ a.s.,\end{aligned}\end{equation}where $\zeta_{n}$ is a random variable such that $\lim_{n\rightarrow +\infty}\zeta_{n}=\zeta<\infty \ a.s.$, and  $b$ is a positive constant, then there exists a connected component $J^{*}$ of $J:=\{\theta|\nabla_{\theta}g(\theta)=0\}$, such that
		\begin{equation}\nonumber\begin{aligned}
		d(\theta_{n},J^{*})\stackrel{a.s.}{\longrightarrow}0.\end{aligned}\end{equation}
	\end{lem}
    \subsection{Proof Outline of Theorems 1 and 2 }\label{pf_outline1}
	{The proof is in light of the Lyapunov method. We aim to prove that $g(\theta_{n})$ is  convergent a.s., and then to prove $\nabla_{\theta_{n}}g(\theta_{n})\rightarrow 0 \ a.s.$ With these two results, we are able to get  $\theta_{n}\rightarrow J^{*}\ a.s.$
% 	Regarding the convergence of $g(\theta_{n})$, we note that 
% 	To prove that $g(\theta_{n})$ is convergent a.s., in the SGD, we can easily to construct a martingale sequence as $\Expect(p_{n+1}(1+g(\theta_{n+1}))|\mathscr{F}_{n})\le p_{n}(1+g(\theta_{n}))$ ($p_{n}$ is a convergent positive sequence) to show $g(\theta_{n})$ is convergent and naturally acquired its convergent rate (through $Martingale\ convergence\ theorem$). But in mSGD, we can not construct a martingale sequence. 
	In the following, we provide  the proof outline to show how to obtain the provided results of mSGD.
	
    Step 1: We prove mSGD is a stable algorithm, i.e.,  $\Expect(g(\theta_{n}))<K<+\infty, \forall \ n$, in Lemma \ref{lem_bound}. The idea is to prove that a weighted sum of the loss function value, i.e.,   
    \begin{equation}\nonumber\begin{aligned}
		U_{n}:=\sum_{t=1}^{n}\big(1/(2-\alpha)\big)^{n-t}\Expect\big(1+g(\theta_{t+1})\big)\end{aligned}\end{equation} is bounded through a recursion formula (a rough form)
		\begin{equation}\nonumber\begin{aligned}
		U_{n}-U_{n-1}\le A\alpha^{n}+\underbrace{B\sum_{t=1}^{n-1}\big(1/(2-\alpha)\big)^{n-t}\epsilon_{t}^{2}\Expect\big\|\nabla_{\theta_{t}}g(\theta_{t},\xi_{t})\big\|^{2}}_{I}\ (A,\ B \ are \ two\ constants).\end{aligned}\end{equation}Then we apply Assumption \ref{ass_g} 4) to $I$ and then obtain
		\begin{equation}\nonumber\begin{aligned}
		U_{n}-U_{n-1}&\le A\alpha^{n}+\underbrace{B\sum_{t=1}^{n-1}\big(1/(2-\alpha)\big)^{n-t}\epsilon_{t}^{2}\Expect\big(1+g(\theta_{t})\big)}_{R}\ (A,\ B \ are \ two\ constants).
		\end{aligned}\end{equation} Combining  $U_{n-1}$ and $R$ leads to
		\begin{equation}\nonumber\begin{aligned}
		F_{n}-F_{n-1}&\le A\alpha^{n}+B' \bigg(\frac{1}{2-\alpha}\bigg)^{n}\ (A,\ B' \ are \ two\ constants),
		\end{aligned}\end{equation}where
		\begin{equation}\nonumber\begin{aligned}
	F_{n}:=\sum_{t=1}^{n}\bigg(\frac{1}{2-\alpha}\bigg)^{n-t}Z(t+1)\Expect\bigg(e_{t+1}^{(n)}\big(1+g(\theta_{t+1})\big)\bigg)\end{aligned}\end{equation} and
		\begin{equation}
		\nonumber
		\begin{aligned}
		Z(t)=\prod_{k=t}^{+\infty}(1+M_{0}\epsilon_{k}^{2})=(1+M_{0}\epsilon_{t}^{2})\prod_{k=t+1}^{+\infty}(1+M_{0}\epsilon_{k}^{2})=(1+M_{0}\epsilon_{t}^{2})Z(t+1).
		\end{aligned}
		\end{equation} 
		Thus, we are able to obtain $E(g(\theta_{n}))<F_{n-1}\le A\sum_{t=1}^{+\infty} \alpha^{t}+B\sum_{t=1}^{+\infty}\big(1/(2-\alpha)\big)^{t}<+\infty$.

	Step 2: From Lemma \ref{lem_bound} and the condition $\sum_{t=1}^{+\infty}\epsilon_{n}^{2}<+\infty$, we are able to prove that  $\sum_{t=1}^{n}\|v_{t}\|^{2}$ and $\sum\epsilon_{t}\|\nabla_{\theta_{t}}g(\theta_{t})\|^{2}$ are convergent a.s. respectively, as stated in Lemma \ref{lem_mid} and Lemma \ref{lem_deriv}.
	
	Step 3: We divide $g(\theta_{n})$ into three terms \begin{equation}\nonumber\begin{aligned}g(\theta_{n})=\sum_{t=1}^{n}A(n)\|v_{t}\|^{2}+\sum_{t=1}^{n}B_{n}\epsilon_{t}\|\nabla_{\theta_{n}}g(\theta_{n})\|^{2}+\sum_{t=1}^{n}C_{n}^{T}(\nabla_{\theta_{n}}g(\theta_{n},\xi_{n})-\nabla_{\theta_{n}}g(\theta_{n})).\end{aligned}\end{equation} From Lemma \ref{lem_mid} and Lemma \ref{lem_deriv}, we are able to prove that $\sum_{t=1}^{n}A(n)\|v_{t}\|^{2}+\sum_{t=1}^{n}B_{n}\epsilon_{t}\|\nabla_{\theta_{n}}g(\theta_{n})\|^{2}$ is convergent a.s.. From the  convergence theorem for martingale-difference sum (Lemma \ref{lem_summation_MDS}), we  prove that $\sum_{t=1}^{n}C_{n}^{T}(\nabla_{\theta_{n}}g(\theta_{n},\xi_{n})-\nabla_{\theta_{n}}g(\theta_{n}))$ is convergent a.s. Then we prove $g(\theta_{n})$ is convergent a.s. 
	
	Step 4: By  Lemma \ref{lem_conve} and  the convergence of $g(\theta_{n})$ in Step 3, we get $\theta_{n}\rightarrow J^{*}$ a.s..
	
	Step 5: After the proof of the convergence of mSGD, we analyze the iterates of $F_{n}$. Then under  a new assumption $\liminf_{d(\theta,J_{i})\rightarrow 0}\|\nabla_{\theta}g(\theta)\|^{2}/\big(g(\theta)-g_{i}\big)\geq s>0$, we are able to obtain the convergent rate of mSGD.
	}

	\subsection{Proof of Lemma~\ref{pouikm,l}}\label{qappend_pf_lem4e}
    First we construct a closed and bounded set $S$ which satisfies $S\supset \bigcup_{i=1}^{m} S_{i}$. Since $\|\nabla f(x)\|\ (x\in S)$ is a continuous function on a closed set,    $\|\nabla f(x)\|\ (x\in S)$ is a uniformly continuous function. Then $\exists \  \epsilon_{2}'>0$, $\forall \ S_{i} $, if $d(x,S_{i})<\epsilon_{2}'$, there is $\|\nabla f(x)\|<c\epsilon_{1}'/4$. We assign $\epsilon_{0}=\min\{\epsilon_{1}'/4,\epsilon_{2}'\}$.
    
    Let $S'_{i}=\{x|d(x,S_{i})<\epsilon_{1}'\}$, $S''_{i}=\{x|d(x,S_{i})<\epsilon_{0}'\}$. Since   $d(x,S_{i})$ is a continuous function, { $\forall x_{0}\in S''_{i}$}, {we can always find a straight line $l_{0}$   paralle to $\nabla f(x_{0})$ and passing through $x_{0}$, defined  as 
    \begin{equation}\nonumber\begin{aligned}
    l_{0}:\ x=x_{0}+\frac{\nabla f(x_{0})}{\|\nabla f(x_{0})\|}t\ (t\in \mathbb{R}).
    \end{aligned}\end{equation} From $S_{i}''\subset S_{i}'$, we know  $x_{0}\in S_{i}'$.   Since $S_{i}'$ is an open set,  there exists $ \alpha_{0}<0<\beta_{0}$, such that $x_{0}+t(\nabla f(x_{0})/\|\nabla f(x_{0})\|)\in S_{i}'\ (\alpha_{0}<t<\beta_{0})$, and
    \begin{equation}\nonumber\begin{aligned}
    &\alpha_{0}=\sup_{t<0}\Bigg(x_{0}+t\bigg(\frac{\nabla f(x_{0})}{\|\nabla f(x_{0})\|}\bigg)\Bigg)\notin S_{i}'
    \\&\beta_{0}=\inf_{t>0}\Bigg(x_{0}+t\bigg(\frac{\nabla f(x_{0})}{\|\nabla f(x_{0})\|}\bigg)\Bigg)\notin S_{i}'.
    \end{aligned}\end{equation}

    %point $x_{1}\in \ S_{i}$, making $x_{1}-x_{0}$ is paralle to $\nabla f(x_{0})$} and $\overline{x_{0}x_{1}}\subset S''_{i}$ ($\overline{x_{0}x_{1}}$ is a straight line connecting $x_{1}$ and $x_{0}$). Let     $x_{2}=\argmin_{x\in L_{i}}d(x,x_{0})$, where  $L_{i}=\{x|x\in \overline{x_{0}x_{1}},x\in S_{i}\}$. 

   Define function
    \begin{equation}\nonumber
        g(t)=\Bigg|f\bigg(x_{0}+\frac{\nabla f(x_{0})}{\|\nabla f(x_{0})\|}t\bigg)-f_{i}\Bigg|\ \ (t\in(\alpha_{0},\beta_{0})).
\end{equation}
   Since $x_{0}+t(\nabla f(x_{0})/\|\nabla f(x_{0})\|)\in S_{i}'\ (\alpha_{0}<t<\beta_{0})$, it holds that 
    \begin{equation}\nonumber
        f\bigg(x_{0}+\frac{\nabla f(x_{0})}{\|\nabla f(x_{0})\|}t\bigg)-f_{i}\neq 0.
\end{equation}
So $g(t)\in C^{1}$ ($t\in(\alpha_{0},\beta_{0})$). Then for $\forall \ t',\ t''\in (\alpha_{0},\beta_{0})$, from Newton-Leibniz formula, it follows that  }
	\begin{equation}\label{ozmbv}
	\begin{aligned}
	g(t')-g(t'')={\int_{t''}^{t'}g'(x)dx=\int_{t''}^{t'}\big(g'(x)-g'(t')+g'(t')\big)dx=\int_{t''}^{t'}\big(g'(x)-g'(t')\big)dx+\int_{t''}^{t'}g'(t')dx}.
	\end{aligned}
	\end{equation}
Next we will prove $g'(x)$ satisfies the Lipschitz  condition. According to the definition of $S_i^{'}$,   $f(x_{0}+t\nabla f(x_{0})/\|\nabla f(x_{0})\|)-f_{i}$ keeps the same sign over $ t\in(\alpha_{0},\beta_{0})$. Thus,   $\forall \ \tau_{1},\ \tau_{2}\in(\alpha_{0},\beta_{0})$, it holds that
	\begin{equation}\nonumber
	\begin{aligned}
	&\big|g'(\tau_{1})-g'(\tau_{2})\big|=\Bigg|\bigg(\frac{\nabla f(x_{0})}{\|\nabla f(x_{0})\|}\bigg)^{T}\Bigg(\nabla f\bigg(x_{0}+\tau_{1}\frac{\nabla f(x_{0})}{\|\nabla f(x_{0})\|}\bigg)-\nabla f\bigg(x_{0}+\tau_{2}\frac{\nabla f(x_{0})}{\|\nabla f(x_{0})\|}\bigg)\Bigg)\Bigg|\\&\le \bigg\|\frac{\nabla f(x_{0})}{\|\nabla f(x_{0})\|}\bigg\|\Bigg\|\nabla f\bigg(x_{0}+\tau_{1}\frac{\nabla f(x_{0})}{\|\nabla f(x_{0})\|}\bigg)-\nabla f\bigg(x_{0}+\tau_{2}\frac{\nabla f(x_{0})}{\|\nabla f(x_{0})\|}\bigg)\Bigg\|\\&=\Bigg\|\nabla f\bigg(x_{0}+\tau_{1}\frac{\nabla f(x_{0})}{\|\nabla f(x_{0})\|}\bigg)-\nabla f\bigg(x_{0}+\tau_{2}\frac{\nabla f(x_{0})}{\|\nabla f(x_{0})\|}\bigg)\Bigg\|.
	\end{aligned}
	\end{equation}From the Lipschitz   condition of $\nabla f$, we have that
	\begin{equation}\nonumber
	\begin{aligned}
	&\big|g'(\tau_{1})-g'(\tau_{2})\big|\le \Bigg\|\nabla f\bigg(x_{0}+\tau_{1}\frac{\nabla f(x_{0})}{\|\nabla f(x_{0})\|}\bigg)-\nabla f\bigg(x_{0}+\tau_{2}\frac{\nabla f(x_{0})}{\|\nabla f(x_{0})\|}\bigg)\Bigg\|\le c|\tau_{1}-\tau_{2}|\bigg\|\frac{\nabla f(x_{0})}{\|\nabla f(x_{0})\|}\bigg\|\\&=c|\tau_{1}-\tau_{2}|.
	\end{aligned}
	\end{equation}From the above analsis, we obtain the Lipschitz   condition of $g'(x)$, that is, $\forall\ \tau_{1},\ \tau_{2}$, there is $|g'(\tau_{1})-g'(\tau_{2})|\le c|\tau_{1}-\tau_{2}|$. By using the absolute value inequality, we get that $-c|\tau_{1}-\tau_{2}|\le g'(\tau_{1})-g'(\tau_{2})\le c|\tau_{1}-\tau_{2}| $. Then it follows from  \eqref{ozmbv} that
	\begin{equation}\nonumber
	\begin{aligned}
	&g(t')-g(t'')=\int_{t''}^{t'}\big(g'(x)-g'(t')\big)dx+\int_{t''}^{t'}g'(t')dx\geq \int_{t''}^{t'}-c|x-t'|dx+\int_{t''}^{t'}g'(t')dx\\&=-\frac{1}{2c}(t'-t'')^{2}+g'(t')(t'-t''),
	\end{aligned}
	\end{equation}Let $t'=0, t''=-g'(0)/c$. {So $t'=0\in(\alpha_{0},\beta_{0})$. Next, we will prove  $t''\in(\alpha_{0},\beta_{0})$. We separate it into two cases. First, we assume that $g'(0)>0$. In this case, because
	\begin{equation}\nonumber\begin{aligned}
    &\alpha_{0}=\sup_{t<0}\Bigg(x_{0}+t\bigg(\frac{\nabla f(x_{0})}{\|\nabla f(x_{0})\|}\bigg)\Bigg)\notin S_{i}',
    \end{aligned}\end{equation}
	we just need to prove
	\begin{equation}\label{,.kjmnhgbvf}
	\begin{aligned}
	&x_{0}+t_{0}\frac{f(x_{0})}{\|\nabla f(x_{0})\|}\in S_{i}', \forall t_{0}\in\Big(-\frac{g'(0)}{c},0\Big).
	\end{aligned}
	\end{equation}Note that 
	\begin{equation}\nonumber
	\begin{aligned}
	&d\bigg(x_{0}+t_{0}\frac{f(x_{0})}{\|\nabla f(x_{0})\|},S_{i}\bigg)\le d(x_{0},S_{i})+\bigg\|t_{0}\frac{f(x_{0})}{\|\nabla f(x_{0})\|}\bigg\|\le \epsilon_{0}'+|t_{0}|.
	\end{aligned}
	\end{equation}From $x_{0}\in S_{i}''$, we get $\|f(x_{0})\|<c\epsilon_{1}'/4$ and $d(x_{0},S_{i})< \epsilon_{1}'$. Then we have
	\begin{equation}\nonumber
	\begin{aligned}
	&d\bigg(x_{0}+t_{0}\frac{f(x_{0})}{\|\nabla f(x_{0})\|},S_{i}\bigg)\le \epsilon_{0}'+|t_{0}|\le \epsilon_{0}'+\Big|\frac{g'(0)}{c}\Big|= \epsilon_{0}'+\frac{\|\nabla f(x_{0})\|}{c}\le \epsilon_{0}'+\frac{\epsilon_{1}'}{4}\le \frac{\epsilon_{1}'}{2}<\epsilon_{1}'.
	\end{aligned}
	\end{equation}Thus, \eqref{,.kjmnhgbvf} holds, meaning when $g'(0)>0$, $t''\in(\alpha_{0},\beta_{0})$. 
	
	Secondly, when $g'(0)<0$, we can also prove $t''\in(\alpha_{0},\beta_{0})$. It follows that
	\begin{equation}\nonumber
	\begin{aligned}
	&g(0)=g(t')\geq g(t')-g(t'')\geq-\frac{1}{2c}(t'-t'')^{2}+g'(t')(t'-t'')=-\frac{1}{2c}\big(g'(0)\big)^{2}+\frac{1}{c}\big(g'(0)\big)^{2}=\frac{1}{2c}\big(g'(0)\big)^{2}\\&=\frac{1}{2c}\Bigg(\bigg(\frac{\nabla f(x_{0})}{\|\nabla f(x_{0})\|}\bigg)^{T}\nabla f(x_{0})\Bigg)^{2}=\frac{1}{2c}\|\nabla f(x_{0})\|^{2}.
	\end{aligned}
	\end{equation}That is
	\begin{equation}\nonumber
	\begin{aligned}
	&\big|f(x_{0})-f_{i}\big|\geq \frac{1}{2c}\|\nabla f(x_{0})\|^{2}.
	\end{aligned}
	\end{equation}}Because of the arbitrariness of $x_{0}$, we concludes that $\exists \epsilon_{0}'=\min\{\epsilon_{1}'/4,\epsilon_{2}'\}$, $\forall S_{i}\ (i=1,2,...,m)$, if $d(x,S_{i})<\epsilon_{0}'$, there is
		\begin{equation}\nonumber\begin{aligned}
		\big\|\nabla f(x)\big\|^{2}\le {2c}\big|f(x)-f_{i}\big|.
		\end{aligned}
		\end{equation}

	\subsection{Proof of Lemma~\ref{lem_summation}}
	For $\forall \delta>0$, we have
	\begin{equation*} 
	\begin{aligned}
	&P\Bigg(\bigcap_{n=1}^{+\infty}\bigcup_{m=n}^{+\infty}\bigg(\sum_{t=n}^{m}X_{t}\geq\delta\bigg)\Bigg)\\
	&=\lim_{n\rightarrow+\infty}P\Bigg(\bigcup_{m=n}^{+\infty}\bigg(\sum_{t=n}^{m}X_{t}\geq\delta\bigg)\Bigg)=\lim_{n\rightarrow+\infty}P\Bigg(\sum_{t=n}^{+\infty}X_{t}\geq\delta\Bigg)
	\leq\lim_{n\rightarrow+\infty}\frac{1}{\delta}\Expect\bigg(\sum_{t=n}^{+\infty}X_{t}\bigg),
	\end{aligned}
	\end{equation*}
	where the last inequality is due to Markov inequality.
	Since $\sum_{t=1}^{\infty}\Expect\big(X_{t}\big)<\infty$, it follows that
% 	\begin{equation} 
% 	\begin{aligned}
% 	\lim_{n\rightarrow+\infty}\frac{1}{\delta}\Expect\bigg(\sum_{t=n}^{+\infty}X_{t}\bigg)=0.
% 	\end{aligned}
% 	\end{equation}
% 	Hence it holds that
	\begin{equation} 
	\begin{aligned}
	P\Bigg(\bigcap_{n=1}^{+\infty}\bigcup_{m=n}^{+\infty}\bigg(\sum_{t=n}^{m}X_{t}\geq\delta\bigg)\Bigg)\leq\lim_{n\rightarrow+\infty}\frac{1}{\delta}\Expect\bigg(\sum_{t=n}^{+\infty}X_{t}\bigg)=0.
	\end{aligned}
	\end{equation}
	By Cauchy's  convergence  test, we have
	$			\sum_{t=1}^{\infty}X_{t}<\infty,\quad   a.s.$

	\subsection{Proof of Lemma~\ref{lem_bound}}\label{pf_lem_bound}
	
	Recall the  mSGD algorithm in \eqref{mSGD}
	\begin{align}
	%			\begin{aligned}
	v_n&=\alpha v_{n-1}+\epsilon_{n} \nabla_{\theta_n} g(\theta_{n} ,\xi_{n})\label{msgd_moment}\\ 
	\theta_{n+1}&=\theta_{n}-v_n.\label{msgd_estimate}
	%			\end{aligned} 
	\end{align}
	%			And we make a transformation
	Equation \eqref{msgd_moment}  is equivalent to
	\begin{equation}\nonumber
	\begin{aligned}
	v_n&=\alpha v_{n-1}+\epsilon_{n} \nabla_{\theta_n} g(\theta_{n})+\epsilon_{n}\big(\nabla_{\theta_n} g(\theta_{n} ,\xi_{n})-\nabla_{\theta_n} g(\theta_{n})\big).
	%			\theta_{n+1}&=\theta_{n}-v_n.
	\end{aligned}
	\end{equation}
% 	This transformation borrows ideas from stochastic approximation. 
% 	For a stochastic approximation problem, it can often  be divided into two parts, the deterministic part and the stochastic part. 
	Under Assumption~\ref{ass_g} 2), it follows from 
	Lemma \ref{lem_ggg} that
	\begin{equation}\label{expansion}
	\begin{aligned}
	&-\nabla_{\theta_{t}}g(\theta_{t})^{\sf T}v_{t}-\frac{c}{2}\|v_{t}\|^{2}\le g(\theta_{t+1})-g(\theta_{t})\le-\nabla_{\theta_{t}}g(\theta_{t})^{\sf T}v_{t}+\frac{c}{2}\|v_{t}\|^{2}.\end{aligned} \end{equation}In this subsection, we just use the right side of \eqref{expansion}. The left side   will be used in the next subsection.
	Consider $\nabla_{\theta_{t}}g(\theta_{t})^{\sf T}v_{t}$ in the following
	\begin{equation}\nonumber
	\begin{aligned}
	\nabla_{\theta_{t}}g(\theta_{t})^{\sf T}v_{t}&=(\nabla_{\theta_{t}}g(\theta_{t}))^{\sf T}(\alpha v_{t-1}+\epsilon_{t}\nabla_{\theta_{n}}g(\theta_{t},\xi_{t}))\\
	&=\alpha(\nabla_{\theta_{t-1}}g(\theta_{t-1})+\nabla_{\theta_{t}}g(\theta_{t})-\nabla_{\theta_{t-1}}g(\theta_{t-1}))^{\sf T}v_{t-1}+\epsilon_{t}\nabla_{\theta_{t}}g(\theta_{t})^{\sf T}\nabla_{\theta_{t}}g(\theta_{t},\xi_{t})\\
	&=\alpha\nabla_{\theta_{t-1}}g(\theta_{t-1})^{\sf T}v_{t-1}+\alpha(\nabla_{\theta_{t}}g(\theta_{t})-\nabla_{\theta_{t-1}}g(\theta_{t-1}))^{T}v_{t-1}+\epsilon_{t}\nabla_{\theta_{t}}g(\theta_{t})^{\sf T} \nabla_{\theta_{t}}g(\theta_{t},\xi_{t})
	%\\&\le \alpha\nabla_{\theta_{t-1}}g(\theta_{t-1})^{\sf T}v_{t-1}+\alpha\big\|\nabla_{\theta_{t}}g(\theta_{t})-\nabla_{\theta_{t-1}}g(\theta_{t-1})\big\|\|v_{t-1}\|+\epsilon_{t}\nabla_{\theta_{t}}g(\theta_{t})^{\sf T} \nabla_{\theta_{t}}g(\theta_{t},\xi_{t})\\&\le \alpha\nabla_{\theta_{t-1}}g(\theta_{t-1})^{\sf T}v_{t-1}+\alpha c\|v_{t-1}\|^{2}+\epsilon_{t}\nabla_{\theta_{t}}g(\theta_{t})^{\sf T} \nabla_{\theta_{t}}g(\theta_{t},\xi_{t}).
	\end{aligned}
	\end{equation}
	Recursively applying the above equation yields \begin{equation}\label{9i8u7y6}\begin{aligned}&\nabla_{\theta_{t}}g(\theta_{t})^{\sf T}v_{t}=\alpha^{t-1}\nabla_{\theta_{1}}g(\theta_{1})^{\sf T}v_{1}+\sum_{i=1}^{t-1}\alpha^{t-i}(\nabla_{\theta_{i}}g(\theta_{i})-\nabla_{\theta_{i-1}}g(\theta_{i-1}))^{T}v_{i-1}\\&+\sum_{i=2}^{t}\alpha^{t-i}\epsilon_{i}\nabla_{\theta_{i}}g(\theta_{i})^{\sf T} \nabla_{\theta_{i}}g(\theta_{i},\xi_{i}).\end{aligned}\end{equation}
	By substituting the above equation into \eqref{expansion}   and noting  $-(\nabla_{\theta_{i}}g(\theta_{i})-\nabla_{\theta_{i-1}}g(\theta_{i-1}))^{\sf T}v_{i-1}\le\big\|\nabla_{\theta_{i}}g(\theta_{i})-\nabla_{\theta_{i-1}}g(\theta_{i-1})\big\|\|v_{i-1}\|\le c\|v_{t-1}\|^{2}$, we obtain
	\begin{align}\label{19}
	\begin{split}
	&g(\theta_{t+1})-g(\theta_{t})\\
	\le&-\alpha^{t-1}\nabla_{\theta_{1}}g(\theta_{1})^{\sf T}v_{1}-\sum_{i=2}^{t}\alpha^{t-i}\epsilon_{i}\nabla_{\theta_{i}}g(\theta_{i})^{\sf T} \nabla_{\theta_{i}}g(\theta_{i},\xi_{i})+\frac{c}{2}\|v_{t}\|^{2} \\
	&-\sum_{i=1}^{t-1}\alpha^{t-i}(\nabla_{\theta_{t}}g(\theta_{t})-\nabla_{\theta_{t-1}}g(\theta_{t-1}))^{T}v_{t-1}\\
	<&-\alpha^{t-1}\nabla_{\theta_{1}}g(\theta_{1})^{\sf T}v_{1}-\sum_{i=2}^{t}\alpha^{t-i}\epsilon_{i}\nabla_{\theta_{i}}g(\theta_{i})^{\sf T} \nabla_{\theta_{i}}g(\theta_{i},\xi_{i})+c\sum_{i=1}^{t}\alpha^{t-i}\|v_{i}\|^{2}.
	\end{split} 
	\end{align}
	Denote  $	Z(t):=\prod_{k=t}^{+\infty}(1+M_{0}\epsilon_{k}^{2}),$ 
	where 
	\begin{equation}\nonumber\begin{aligned}
	&M_{0}=\frac{cM}{\alpha^{1-\delta}(1-\alpha^{\delta})(1-\alpha)^{2}},
	\end{aligned}
	\end{equation} where $\delta>0$ is a constant and $M$ is introduced in Assumption~\ref{ass_g} 4). Here 
	we define $M_{0}$, $Z(t)$ and $\delta$   to facilitate the proof \eqref{33}. From Assumption \ref{assum8}, it holds that $\sum_{t=1}^{+\infty}\epsilon_{t}^{2}<+\infty$. Thus,   $\sum_{t=1}^{+\infty}M_{0}\epsilon_{t}^{2}<+\infty$. From a general inequality $\ln(1+x)\le x$ for $x>-1$, we   get
	\begin{equation}\nonumber\begin{aligned}
	Z(t)\le\prod_{k=1}^{+\infty}(1+M_{0}\epsilon_{k}^{2})=\exp\bigg\{\sum_{k=1}^{\infty}\ln(1+M_{0}\epsilon_{k}^{2})\bigg\} \le \exp\bigg\{\sum_{k=1}^{\infty}M_{0}\epsilon_{k}^{2}\bigg\}<+\infty,\end{aligned}\end{equation}
	which means that $Z(t)$ is uniformly upper bounded. 
	Then  multiplying $Z(t+1)$ on   both sides of \eqref{19} and taking the mathematical expectation yield
	\begin{equation}\label{20}
	\begin{aligned}
	Z(t+1)\Expect \big(g(\theta_{t+1})-g(\theta_{t})\big) \leq&-Z(t+1)\alpha^{t-1}\Expect\big(\nabla_{\theta_{1}}g(\theta_{1})^{\sf T}v_{1}\big)+c\sum_{i=1}^{t}\alpha^{t-i}Z(i+1)\Expect\big(\|v_{i}\|^{2}\big)\\
	&\quad -Z(t+1)\sum_{i=2}^{t}\alpha^{t-i}\epsilon_{i}\Expect\big\|\nabla_{\theta_{i}}g(\theta_{i})\big\|^{2}\\
	<&-Z(t+1)\alpha^{t-1}\Expect\big(\nabla_{\theta_{1}}g(\theta_{1})^{\sf T}v_{1}\big)+c\sum_{i=1}^{t}\alpha^{t-i}Z(i+1)\Expect\big(\|v_{i}\|^{2}\big),
	\end{aligned} 
	\end{equation}
	where the first and second inequalities are respectively due to $Z(i+1)<Z(i)$ and    $Z(t+1)\sum_{i=2}^{t}\alpha^{t-i}\epsilon_{i}\Expect\big\|\nabla_{\theta_{i}}g(\theta_{i})\big\|^{2}>0.$
	
	Next, we aim to analyze $c\sum_{i=1}^{t}\alpha^{t-i}Z(i+1)\Expect\big(\|v_{i}\|^{2}\big)$ in \eqref{20}. 		
	It is proved in Appendix \ref{pf_27} that
	\begin{equation}\label{27}\begin{aligned}
	Z(i+1)\Expect\|v_{i}\|^{2}
	\le& \alpha^{(1+\delta)i}Z(1)\Expect\|v_{0}\|^{2}
	+\frac{1}{\alpha^{1-\delta}}\sum_{k=1}^{i}\alpha^{(1+\delta)(i-k)}\epsilon_{k}^{2}Z(k+1)\Expect\big(\|\nabla_{\theta_{k}}g(\theta_{k})-\nabla_{\theta_{k}}g(\theta_{k},\xi_{k})\|^{2}\big)\\
	&-\frac{2}{\alpha^{1-\delta}}\sum_{k=1}^{i}\alpha^{(1+\delta)(i-k)}Z(k+1)\bigg(\Expect\big(\epsilon_{k}g(\theta_{k+1})\big)-\Expect\big(\epsilon_{k-1}g(\theta_{k})\big)\bigg).\end{aligned}
	\end{equation}
Taking a weighted sum of \eqref{27} yields 
	\begin{equation}\label{bht}\begin{aligned}
	&\sum_{i=1}^{t}\alpha^{t-i}Z(i+1)\Expect\|v_{i}\|^{2}\\
	\le&\sum_{i=1}^{t}\alpha^{t-i}\alpha^{(1+\delta)i}Z(1)\Expect\big(\|v_{0}\|^{2}\big)
	\\
	&+\sum_{i=1}^{t}\alpha^{t-i}\bigg(\frac{1}{\alpha^{1-\delta}}\sum_{k=1}^{i}\alpha^{(1+\delta)(i-k)}\epsilon_{k}^{2}Z(k+1)\Expect\big(\|\nabla_{\theta_{k}}g(\theta_{k})-\nabla_{\theta_{k}}g(\theta_{k},\xi_{k})\|^{2}\big)\bigg)\\
	&-\sum_{i=1}^{t}\alpha^{t-i}\Bigg(\frac{2}{\alpha^{1-\delta}}\sum_{k=1}^{i}\alpha^{(1+\delta)(i-k)}Z(k+1)\bigg(\Expect\big(\epsilon_{k}g(\theta_{k+1})\big)-\Expect\big(\epsilon_{k-1}g(\theta_{k})\big)\bigg)\Bigg):=A+B+C.
	\end{aligned}
	\end{equation}
	We derive that 
	\begin{align}\label{bht1}
	A&=\bigg(\sum_{i=1}^{t}\alpha^{\delta i}\bigg)\alpha^{t}Z(1)\Expect\big(\|v_{0}\|^{2}\big)\leq \frac{\alpha^{t-1}\alpha^{\delta}}{1-\alpha^{\delta}}Z(1)\Expect\big(\|v_{0}\|^{2}\big),
	\end{align}
	\begin{align}\label{bht2}
	B&= \frac{1}{\alpha^{1-\delta}}\sum_{i=1}^{t}\sum_{k=1}^{i}\alpha^{t-k+\delta(i-k)}\bigg(\epsilon_{k}^{2}Z(k+1)\Expect\big(\|\nabla_{\theta_{k}}g(\theta_{k})-\nabla_{\theta_{k}}g(\theta_{k},\xi_{k})\|^{2}\big)\bigg)\nonumber\\
	&\leq \frac{1}{\alpha^{1-\delta}(1-\alpha^{\delta})}\sum_{k=1}^{t}\alpha^{t-k}Z(k+1)\epsilon_{k}^{2}\Expect\big(\|\nabla_{\theta_{k}}g(\theta_{k})-\nabla_{\theta_{k}}g(\theta_{k},\xi_{k})\|^{2}\big),
	\end{align}
	\begin{align}\label{bht3}
	C&=-\frac{2}{\alpha^{1-\delta}}\sum_{k=1}^{t}\sum_{i=k}^{t}\alpha^{t-k+\delta(i-k)}Z(k+1)\bigg(\Expect\big(\epsilon_{k}g(\theta_{k+1})\big)-\Expect\big(\epsilon_{k-1}g(\theta_{k})\big)\bigg)\nonumber\\
	&= -\frac{2}{\alpha^{1-\delta}}\sum_{k=1}^{t}\bigg(\sum_{i=0}^{t-k}\alpha^{\delta i}\bigg)\alpha^{t-k}Z(k+1)\bigg(\Expect\big(\epsilon_{k}g(\theta_{k+1})\big)-\Expect\big(\epsilon_{k-1}g(\theta_{k})\big)\bigg).
	\end{align}

	Substituting \eqref{bht1}--\eqref{bht3} into \eqref{20} yields \begin{equation}\label{28}\begin{aligned}
	&Z(t+1)\Expect\bigg(\big(g(\theta_{t+1})-g(\theta_{t})\big)\bigg)\\\le&-Z(t+1)\alpha^{t-1}\Expect\big(\nabla_{\theta_{1}}g(\theta_{1})^{\sf T}v_{1}\big)+\frac{\alpha^{t-1}c\alpha^{\delta}}{1-\alpha^{\delta}}Z(1)\Expect\big(\|v_{0}\|^{2}\big)\\&+\frac{c}{\alpha^{1-\delta}(1-\alpha^{\delta})}\sum_{i=1}^{t}\alpha^{t-i}Z(i+1)\epsilon_{i}^{2}\Expect\big(\|\nabla_{\theta_{i}}g(\theta_{i})-\nabla_{\theta_{i}}g(\theta_{i},\xi_{i})\|^{2}\big)\\&-\frac{2c}{\alpha^{1-\delta}}\sum_{i=1}^{t}\bigg(\sum_{k=0}^{t-i}\alpha^{\delta k}\bigg)\alpha^{t-i}Z(i+1)\bigg(\Expect\big(\epsilon_{i}g(\theta_{i+1})\big)-\Expect\big(\epsilon_{i-1}g(\theta_{i})\big)\bigg).\end{aligned}\end{equation}
	
	Construct a sequence $\{V_{n}\}$ as follows \begin{equation}\label{29}\begin{aligned}
	V_{n}=\sum_{t=1}^{n}\bigg(\frac{1}{2-\alpha}\bigg)^{n-t}Z(t+1)\Expect\bigg(\big(g(\theta_{t+1})-g(\theta_{t})\big)\bigg).
	\end{aligned}
	\end{equation}
	By substituting \eqref{28} into \eqref{29} following the   way of \eqref{bht1}--\eqref{bht3}, we have \begin{equation}\nonumber\label{30}\begin{aligned}
	V_{n}\le&\frac{\alpha^{n}(2-\alpha)}{1-\alpha}Z(1)\Big|\Expect\big(\nabla_{\theta_{1}}g(\theta_{1})^{\sf T}v_{1}\big)\Big|+\frac{c\alpha^{n}\alpha^{\delta}(2-\alpha)}{(1-\alpha)(1-\alpha^{\delta})}Z(1)\Expect\big(\|v_{0}\|^{2}\big)\\&+\frac{c}{\alpha^{1-\delta}(1-\alpha^{\delta})(1-\alpha)^{2}}\sum_{t=1}^{n}\bigg(\frac{1}{2-\alpha}\bigg)^{n-i}Z(t+1)\epsilon_{t}^{2}\Expect\big(\|\nabla_{\theta_{t}}g(\theta_{t})-\nabla_{\theta_{t}}g(\theta_{t},\xi_{t})\|^{2}\big)\\&-\frac{2c}{\alpha^{1-\delta}(1-\alpha^{\delta})}\sum_{t=1}^{n}f(n-t)\bigg(\frac{1}{2-\alpha}\bigg)^{n-t}Z(t+1)\bigg(\Expect\big(\epsilon_{t}g(\theta_{t+1})\big)-\Expect\big(\epsilon_{t-1}g(\theta_{t})\big)\bigg),\end{aligned}\end{equation}
	where $f(n-t)$ is defined as follows
	\begin{equation}\nonumber\begin{aligned}
	f(n-t)=\sum_{k=1}^{n-t}\big(\alpha(2-\alpha)\big)^{k}-\alpha^{\delta}\sum_{k=1}^{n-t}\big(\alpha^{1+\delta}(2-\alpha)\big)^{k}.\end{aligned}\end{equation}
	Move the last term to the left-hand side of the above inequality, then we have
	\begin{equation}\label{31}\begin{aligned}
	&\sum_{t=1}^{n}\bigg(\frac{1}{2-\alpha}\bigg)^{n-t}Z(t+1)\Expect\bigg(e_{t+1}^{(n)}g(\theta_{t+1})-e_{t}^{(n-1)}g(\theta_{t})\big)\bigg)\\\le&\frac{\alpha^{n}(2-\alpha)}{1-\alpha}Z(1)\Big|\Expect\big(\nabla_{\theta_{1}}g(\theta_{1})^{\sf T}v_{1}\big)\Big|+\frac{c\alpha^{n}\alpha^{\delta}(2-\alpha)}{(1-\alpha)(1-\alpha^{\delta})}Z(1)\Expect\big(\|v_{0}\|^{2}\big)\\&+\frac{c}{\alpha^{1-\delta}(1-\alpha^{\delta})(1-\alpha)^{2}}\sum_{t=1}^{n}\bigg(\frac{1}{2-\alpha}\bigg)^{n-t}Z(t+1)\epsilon_{t}^{2}\Expect\big(\|\nabla_{\theta_{t}}g(\theta_{t})-\nabla_{\theta_{t}}g(\theta_{t},\xi_{t})\|^{2}\big),\end{aligned}\end{equation}where
	\begin{equation}\label{98uj}\begin{aligned}
	e_{t+1}^{(n)}=\bigg(1+\frac{2c\epsilon_{t}f(n-t)}{\alpha^{1-\delta}(1-\alpha^{\delta})}\bigg).
	\end{aligned}\end{equation}
	Because of $\alpha<1$, it holds that $f(n-t)>0$ and $e_{t+1}^{(n)}>1$. It follows from Assumption \ref{assum8} that
	\begin{equation}\label{33}
	\begin{aligned}
	&\frac{c}{\alpha^{1-\delta}(1-\alpha^{\delta})(1-\alpha)^{2}}\sum_{t=1}^{n}\bigg(\frac{1}{2-\alpha}\bigg)^{n-t}Z(t+1)\epsilon_{t}^{2}\Expect\big(\|\nabla_{\theta_{t}}g(\theta_{t})-\nabla_{\theta_{t}}g(\theta_{t},\xi_{t})\|^{2}\big)\\
	\le&\sum_{t=1}^{n}\bigg(\frac{1}{2-\alpha}\bigg)^{n-t}Z(t+1)M_{0}\epsilon_{t}^{2}e_{t}^{(n-1)}\big(1+\Expect(g(\theta_{t}))\big),
	\end{aligned}
	\end{equation}
	where
	\begin{equation}\nonumber\begin{aligned}
	&M_{0}=\frac{cM}{\alpha^{1-\delta}(1-\alpha^{\delta})(1-\alpha)^{2}}.
	\end{aligned}
	\end{equation}
	Calculate $f(n-t)$, then we obtain 
	\begin{equation}\nonumber\begin{aligned}
	f(n-t)=\sum_{k=1}^{n-t}\big(\alpha(2-\alpha)\big)^{k}-\alpha^{\delta}\sum_{k=1}^{n-t}\big(\alpha^{1+\delta}(2-\alpha)\big)^{k}>0.
	\end{aligned}
	\end{equation}
	It holds that 
	\begin{equation}\label{33.5}\begin{aligned}
	&\sum_{t=1}^{n}\bigg(\frac{1}{2-\alpha}\bigg)^{n-t}Z(t+1)\Expect\bigg(e_{t+1}^{(n)}g(\theta_{t+1})-e_{t}^{(n-1)}g(\theta_{t})\big)\bigg)\\&-\sum_{t=1}^{n}\bigg(\frac{1}{2-\alpha}\bigg)^{n-t}Z(t+1)M_{0}\epsilon_{t}^{2}e_{t}^{(n-1)}\big(1+\Expect(g(\theta_{t}))\big)\\&=\sum_{t=1}^{n}\bigg(\frac{1}{2-\alpha}\bigg)^{n-t}Z(t+1)\Expect\bigg(e_{t+1}^{(n)}\big(1+g(\theta_{t+1})\big)\bigg)\\&-\sum_{t=1}^{n}\bigg(\frac{1}{2-\alpha}\bigg)^{n-t}Z(t+1)(1+M_{0}\epsilon_{t}^{2})\Expect\bigg(e_{t}^{(n-1)}\big(1+g(\theta_{t})\big)\bigg)\\&-\frac{c}{\alpha^{1-\delta}(1-\alpha^{\delta})}\sum_{t=1}^{n}f(n-t)\bigg(\frac{1}{2-\alpha}\bigg)^{n-t}\big(\epsilon_{t}-\epsilon_{t-1}\big)\\&>\sum_{t=1}^{n}\bigg(\frac{1}{2-\alpha}\bigg)^{n-t}Z(t+1)\Expect\bigg(e_{t+1}^{(n)}\big(1+g(\theta_{t+1})\big)\bigg)\\&-\sum_{t=1}^{n}\bigg(\frac{1}{2-\alpha}\bigg)^{n-t}Z(t+1)(1+M_{0}\epsilon_{t}^{2})\Expect\bigg(e_{t}^{(n-1)}\big(1+g(\theta_{t})\big)\bigg).
	\end{aligned}\end{equation}
	Substituting \eqref{33} and \eqref{33.5} into \eqref{31} yields 
	\begin{equation}\label{34}\begin{aligned}
	&\sum_{t=1}^{n}\bigg(\frac{1}{2-\alpha}\bigg)^{n-t}Z(t+1)\Expect\bigg(e_{t+1}^{(n)}\big(1+g(\theta_{t+1})\big)\bigg)\\&-\sum_{t=1}^{n}\bigg(\frac{1}{2-\alpha}\bigg)^{n-t}Z(t+1)(1+M_{0}\epsilon_{t}^{2})\Expect\bigg(e_{t}^{(n-1)}\big(1+g(\theta_{t})\big)\bigg)\\\le&\frac{\alpha^{n}(2-\alpha)}{1-\alpha}Z(1)\Big|\Expect\big(\nabla_{\theta_{1}}g(\theta_{1})^{\sf T}v_{1}\big)\Big|+\frac{c\alpha^{n}\alpha^{\delta}(2-\alpha)}{2(1-\alpha)(1-\alpha^{\delta})}Z(1)\Expect\big(\|v_{0}\|^{2}\big).
	\end{aligned}\end{equation}
	
	According to the definition of $Z(t)$, we have  
	\begin{equation}\nonumber\begin{aligned}Z(t)=\prod_{k=t}^{+\infty}(1+M_{0}\epsilon_{k}^{2})=(1+M_{0}\epsilon_{t}^{2})\prod_{k=t+1}^{+\infty}(1+M_{0}\epsilon_{k}^{2})=(1+M_{0}\epsilon_{t}^{2})Z(t+1).\end{aligned}\end{equation}
	Let 
	\begin{equation}\nonumber\begin{aligned}
	F_{n}:=\sum_{t=1}^{n}\bigg(\frac{1}{2-\alpha}\bigg)^{n-t}Z(t+1)\Expect\bigg(e_{t+1}^{(n)}\big(1+g(\theta_{t+1})\big)\bigg),\end{aligned}\end{equation}
	and it follows from  \eqref{34} that
	\begin{equation}\label{diff_F}
	\begin{aligned}
	F_{n}-F_{n-1}\le& Z(1)\bigg(\frac{1}{2-\alpha}\bigg)^{n-1}\Expect\big(g(\theta_{1})\big)+\frac{\alpha^{n}(2-\alpha)}{1-\alpha}Z(1)\Big|\Expect\big(\nabla_{\theta_{1}}g(\theta_{1})^{\sf T}v_{1}\big)\Big|\\+&\frac{c\alpha^{n}\alpha^{\delta}(2-\alpha)}{2(1-\alpha)(1-\alpha^{\delta})}Z(1)\Expect\big(\|v_{0}\|^{2}\big).
	\end{aligned}\end{equation}

	Denote $p=\exp\bigg\{\sum_{k=1}^{\infty}M_{0}\epsilon_{k}^{2}\bigg\}$.
	By taking  the summation of \eqref{diff_F}, we obtain
	\begin{equation}\nonumber\begin{aligned}
	F_{n}\le& Z(1)\Expect\big(e_{1}\big(1+g(\theta_{1})\big)\big)\sum_{t=1}^{n}\bigg(\frac{1}{2-\alpha}\bigg)^{t-1}+\frac{p(2-\alpha)}{1-\alpha}\bigg|\Expect\big(\nabla_{\theta_{1}}g(\theta_{1})^{\sf T}v_{1}\big)\bigg|\sum_{t=1}^{n}\alpha^{t}\\&+\frac{c\alpha^{\delta}(2-\alpha)}{2(1-\alpha)(1-\alpha^{\delta})}Z(1)\Expect\big(\|v_{0}\|^{2}\big)\sum_{t=1}^{n}\alpha^{t}\\&<\frac{p(2-\alpha)}{1-\alpha}\Expect\big(e_{1}\big(1+g(\theta_{1})\big)\big)+\frac{p\alpha(2-\alpha)}{(1-\alpha)^{2}}\bigg|\Expect\big(\nabla_{\theta_{1}}g(\theta_{1})^{\sf T}v_{1}\big)\bigg|\\&+\frac{c\alpha^{1+\delta}(2-\alpha)}{2(1-\alpha)^{2}(1-\alpha^{\delta})}Z(1)\Expect\big(\|v_{0}\|^{2}\big).
	\end{aligned}\end{equation}
	Define
	\begin{equation}\nonumber\begin{aligned}
	T(\theta_{1},v_{0})&=1+\frac{p(2-\alpha)}{1-\alpha}\Expect\big(e_{1}\big(1+g(\theta_{1})\big)\big)+\frac{p\alpha(2-\alpha)}{(1-\alpha)^{2}}\bigg|\Expect\big(\nabla_{\theta_{1}}g(\theta_{1})^{\sf T}v_{1}\big)\bigg|\\&+\frac{c\alpha^{1+\delta}(2-\alpha)}{2(1-\alpha)^{2}(1-\alpha^{\delta})}Z(1)\Expect\big(\|v_{0}\|^{2}\big).
	\end{aligned}\end{equation}
	It follows from the relationship between $g(\theta_{n+1})$ and $F_{n}$ that
	\begin{equation}\nonumber\begin{aligned}
	\Expect\big(1+g(\theta_{n+1})\big)&<p\bigg(\frac{1}{2-\alpha}\bigg)^{n-n}Z(n+1)\Expect\big(e_{n+1}^{(n)}\big(1+g(\theta_{n+1})\big)\big)\\
	&<p\sum_{t=1}^{n}\bigg(\frac{1}{2-\alpha}\bigg)^{n-t}Z(t+1)\Expect\big(e_{t+1}^{(n)}\big(1+g(\theta_{t+1})\big)\big)\\
	&<pF_{n}<1+T(\theta_{1},v_{0}),
	\end{aligned}\end{equation}
	which leads to $\Expect\big(g(\theta_{n})\big)<T(\theta_{1},v_{0})$.

	\subsection{Proof of \eqref{27}}\label{pf_27}
	%			For this goal, we made some calculations as follows
	We consider
	\begin{equation}\label{21}
	\begin{aligned}
	&\quad\epsilon_{i}g(\theta_{i+1})-\epsilon_{i-1}g(\theta_{i})\\
	&=\epsilon_{i}\big(g(\theta_{i+1}-g(\theta_{i})\big)+(\epsilon_{i}-\epsilon_{i-1})g(\theta_{i})\\&\le \epsilon_{i}\big(g(\theta_{i+1})-g(\theta_{i})\big)\\
	&\le-\epsilon_{i}\nabla_{\theta_{i}}g(\theta_{i})^{\sf T}v_{i}+\frac{c}{2}\epsilon_{i}\|v_{i}\|^2\\&=(-\epsilon_{i}\nabla_{\theta_{i}}g(\theta_{i},\xi_{i})^{\sf T}v_{i}-(\epsilon_{i}\nabla_{\theta_{i}}g(\theta_{i})-\epsilon_{i}\nabla_{\theta_{i}}g(\theta_{i},\xi_{i}))^{\sf T}v_{i}+\frac{c}{2}\epsilon_{i}\|v_{i}\|^2\\
	&=\alpha v_{i}^{\sf T}v_{i-1}-\|v_{i}\|^{2}+\epsilon_{i}^{2}\big\|\nabla_{\theta_{i}}g(\theta_{i})-\nabla_{\theta_{i}}g(\theta_{i},\xi_{i})\big\|^{2}\\
	&\quad-\big(\epsilon_{i}\alpha v_{i-1}+\epsilon_{i}^{2}\nabla_{\theta_{i}}g(\theta_{i})\big)^{\sf T}\big(\nabla_{\theta_{i}}g(\theta_{i})-\nabla_{\theta_{i}}g(\theta_{i},\xi_{i}))\big)+\frac{c}{2}\epsilon_{i}\|v_{i}\|^{2},
	\end{aligned}\end{equation}
	where the first inequality is due to $\epsilon_{i}\leq\epsilon_{i-1}$ in Assumption \ref{assum8}, and the last equality is from \eqref{msgd_estimate}.
	Since $\xi_{i}$ and $\theta_{i}$ are independent,
taking the mathematical expectation of \eqref{21} and noting that\begin{equation}\nonumber\begin{aligned}
	&\Expect\bigg(\big(\epsilon_{i}\alpha v_{i-1}+\epsilon_{i}^{2}\nabla_{\theta_{i}}g(\theta_{i})\big)^{\sf T}\big(\nabla_{\theta_{i}}g(\theta_{i})-\nabla_{\theta_{i}}g(\theta_{i},\xi_{i}))\big)\bigg)=0,
	\end{aligned}\end{equation}
  yield
  \begin{equation}\label{22}
	\begin{aligned}&\Expect\big(\epsilon_{i}g(\theta_{i+1})\big)-\Expect\big(\epsilon_{i-1}g(\theta_{i})\big)\\\le&\alpha \Expect(v_{i}^{\sf T}v_{i-1})-\Expect\big(\|v_{i}\|^{2}\big)+\epsilon_{i}^{2}\Expect\bigg(\big\|\nabla_{\theta_{i}}g(\theta_{i})-\nabla_{\theta_{i}}g(\theta_{i},\xi_{i})\big\|^{2}\bigg)+\frac{c}{2}\epsilon_{i}\Expect\big(\|v_{i}\|^{2}\big).\end{aligned}\end{equation}
Moreover, it holds   that 
\begin{equation}\label{23}
	\begin{aligned}&\Expect\big(\|\nabla_{\theta_{i}}g(\theta_{i})-\nabla_{\theta_{i}}g(\theta_{i},\xi_{i})\|^{2}\big)\\=&\Expect\big(\|\nabla_{\theta_{i}}g(\theta_{i},\xi_{i})\|^{2}\big)-\Expect\big(\|\nabla_{\theta_{i}}g(\theta_{i})\|^{2}\big)\\=&\frac{1}{\epsilon_{i}^{2}}\Expect\big(\|v_{i}-\alpha v_{i-1}\|^{2}\big)-\Expect\big(\|\nabla_{\theta_{i}}g(\theta_{i})\|^{2}\big)\\=&\frac{1}{\epsilon_{i}^{2}}\bigg(\Expect\big(\|v_{i}\|^2\big)+\alpha^{2}\Expect\big(\|v_{i-1}\|^2\big)-2\alpha \Expect(v_{i}^{\sf T}v_{i-1})\bigg)-\Expect\big(\|\nabla_{\theta_{i}}g(\theta_{i})\|^{2}\big).\end{aligned}\end{equation}
	Combining \eqref{22} and \eqref{23}, we get\begin{equation}\label{24}\begin{aligned}&\Expect\big(\epsilon_{i}g(\theta_{i+1})\big)-\Expect\big(\epsilon_{i-1}g(\theta_{i})\big)\\\le&-\frac{1}{2}\bigg(\Expect\big(\|v_{i}\|^2\big)-\alpha^{2}\Expect\big(\|v_{i-1}\|^{2}\big)\bigg)-\frac{\epsilon_{i}^{2}}{2}\Expect\big(\|\nabla_{\theta_{i}}g(\theta_{i})\|^{2}\big)+\frac{c}{2}\epsilon_{i}\Expect\big(\|v_{i}\|^{2}\big)\\&+\frac{\epsilon_{i}^{2}}{2}\Expect\big(\|\nabla_{\theta_{i}}g(\theta_{i})-\nabla_{\theta_{i}}g(\theta_{i},\xi_{i})\|^{2}\big).\end{aligned}\end{equation}
	Since $\epsilon_{i}\rightarrow0$,   given any $\delta>0$, there is an integer $i_{0}\geq 0$, such that for $i\ge i_{0}$, $1-c\epsilon_{i}>\alpha^{1-\delta}(\delta>0)$.
	Since $i_{0}$ is finite, without loss of generality, we assume $  i_{0}=0$ for convenience, i.e.,  $1-c\epsilon_{i}>\alpha^{1-\delta}(\delta>0)(i\ge 1)$. Thus, we have\begin{equation}\label{25}\begin{aligned}
	&\Expect\big(\epsilon_{i}g(\theta_{i+1})\big)-\Expect\big(\epsilon_{i-1}g(\theta_{i})\big)\\\le&-\frac{\alpha^{1-\delta}}{2}\bigg(\Expect\big(\|v_{i}\|^{2}\big)-\alpha^{1+\delta}\Expect\big(\|v_{i-1}\|^2\big)\bigg)+\frac{\epsilon_{i}^{2}}{2}\Expect\big(\|\nabla_{\theta_{i}}g(\theta_{i})-\nabla_{\theta_{i}}g(\theta_{i},\xi_{i})\|^{2}\big).\end{aligned}\end{equation}
	Multiplying both sides of \eqref{25} by $Z(i+1)$, and noticing that $Z(i)>Z(i+1)$, we have
	\begin{equation}\nonumber
	\begin{aligned}
	&\bigg(Z(i+1)\Expect\big(\|v_{i}\|^{2}\big)-\alpha^{1+\delta}Z(i)\Expect\big(\|v_{i-1}\|^2\big)\bigg)\\\le& -\frac{2}{\alpha^{1-\delta}}Z(i+1)\bigg(\Expect\big(\epsilon_{i}g(\theta_{i+1})\big)-\Expect\big(\epsilon_{i-1}g(\theta_{i})\big)\bigg) +\frac{\epsilon_{i}^{2}}{\alpha^{1-\delta}}Z(i+1)\Expect\big(\|\nabla_{\theta_{i}}g(\theta_{i})-\nabla_{\theta_{i}}g(\theta_{i},\xi_{i})\|^{2}\big).
	\end{aligned}\end{equation}
 Then \eqref{27} is obtained by recursively applying the above inequality.

	\subsection{Proof of Lemma~\ref{lem_mid}}\label{pf_lem_mid}
	From \eqref{24}, 
	% \begin{equation}\nonumber\begin{aligned}&\Expect\big(\epsilon_{t}g(\theta_{t+1})\big)-\Expect\big(\epsilon_{t-1}g(\theta_{t})\big)\\
	% \le&-\frac{1}{2}\bigg(\Expect\big(\|v_{t}\|^2\big)-\alpha^{2}\Expect\big(\|v_{t-1}\|^{2}\big)\bigg)-\frac{\epsilon_{t}^{2}}{2}\Expect\big(\|\nabla_{\theta_{t}}g(\theta_{t})\|^{2}\big)+\frac{c}{2}\epsilon_{t}\Expect\big(\|v_{t}\|^{2}\big)\\&+\frac{\epsilon_{t}^{2}}{2}\Expect\big(\|\nabla_{\theta_{t}}g(\theta_{t})-\nabla_{\theta_{t}}g(\theta_{t},\xi_{t})\|^{2}\big).\end{aligned}\end{equation}
	we have
	\begin{equation}\label{pf_compare}
	\begin{split}
	&\Expect\big(\epsilon_{n}g(\theta_{n+1})\big)-\Expect\big(\epsilon_{0}g(\theta_{1})\big)\\
	<&-\frac{1}{2}\sum_{t=1}^{n}\Expect\big(\|v_{t}\|^{2}\big)+\frac{\alpha^{2}}{2}\sum_{t=1}^{n}\Expect\big(\|v_{t-1}\|^{2}\big)-\sum_{t=1}^{n}\frac{\epsilon_{t}^{2}}{2}\Expect\big(\|\nabla_{\theta_{t}}g(\theta_{t})\|^{2}\big)+\frac{c}{2}\sum_{t=1}^{n}\epsilon_{t}\Expect\big(\|v_{t}\|^{2}\big)\\
	& +\sum_{t=1}^{n}\frac{\epsilon_{t}^{2}}{2}\Expect\big(\|\nabla_{\theta_{t}}g(\theta_{t})-\nabla_{\theta_{t}}g(\theta_{t},\xi_{t})\|^{2}\big).
	\end{split}\end{equation}
	%			Because $\{\nabla_{\theta_{t}}g(\theta_{t})-\nabla_{\theta_{t}}g(\theta_{t},\xi_{t})\}$ is a Martingale difference sequence and through 
	It follows from Assumption \ref{ass_g} 5) and Lemma \ref{lem_bound} that  \begin{equation}\nonumber\begin{aligned}&\Expect\big(\|\nabla_{\theta_{t}}g(\theta_{t})-\nabla_{\theta_{t}}g(\theta_{t},\xi_{t})\|^{2}\big)<M(1+T(\theta_{1},v_{0}))<+\infty.
	\end{aligned}\end{equation}
	Because of $\sum_{t=1}^{n}\epsilon_{t}^{2}<+\infty$, there is a scalar $\bar M>0$ such that for $\forall n$\begin{equation}\nonumber\begin{aligned}&\sum_{t=1}^{n}\frac{\epsilon_{t}^{2}}{2}\Expect\big(\|\nabla_{\theta_{t}}g(\theta_{t})-\nabla_{\theta_{t}}g(\theta_{t},\xi_{t})\|^{2}\big)<\bar M<+\infty.\end{aligned}
	\end{equation}
	Then it follows from \eqref{pf_compare} that \begin{equation}\nonumber\begin{aligned}
	&\frac{1}{2}\sum_{t=1}^{n}(1-\alpha^{2}-c\epsilon_{t})\Expect\big(\|v_{t}\|^2\big)\\
	\le& \bar M+\epsilon_{0}g(\theta_{1})-\epsilon_{n}\Expect\big(g(\theta_{n+1})\big) +\frac{\alpha^{2}}{2}\Expect\big(\|v_{0}\|^{2}-\|v_{n}\|^{2}\big)\\
	& -\sum_{t=1}^{n}\frac{\epsilon_{t}^{2}}{2}\Expect\big(\|\nabla_{\theta_{t}}g(\theta_{t})\|^{2}\big)\\
	<&\bar M+\epsilon_{0}g(\theta_{1})+\alpha^{2}\Expect\big(\|v_{0}\|^{2}\big)/2<K,
	\end{aligned}
	\end{equation}
	where $K$ is a positive scalar.
	Since $\epsilon_{n}\rightarrow0$  when $n$ is large enough, it holds that $\frac{1}{5}(1-\alpha^{2})<1-\alpha^{2}-c\epsilon_{n}$. Without loss of generality, assume $\frac{1}{5}(1-\alpha^{2})<1-\alpha^{2}-c\epsilon_{n}^{2}$ for $n\ge 0$, so $\sum_{t=1}^{n}\Expect\big(\|v_{t}\|^2\big)<\frac{10K}{1-\alpha^{2}}<+\infty.$
	By Lemma~\ref{lem_summation}, we obtain $	\sum_{t=1}^{n}\|v_{t}\|^{2}<+\infty.$
	
	\subsection{Proof of Lemma~\ref{lem_deriv}}\label{pf_lem_deriv}
	Through Taylor expansion, we derive that
	\begin{equation}\label{41.5}\begin{aligned}
    &g(\theta_{t+1})-g(\theta_{t})\\
	=&\nabla_{\theta_{\zeta_{t}}}g(\theta_{\zeta_{t}})^{T}(\theta_{t+1}-\theta_{t})=-\nabla_{\theta_{t}}g(\theta_{t})^{T}v_{t}+\Big(\nabla_{\theta_{\zeta_{t}}}g(\theta_{\zeta_{t}})-\nabla_{\theta_{t}}g(\theta_{t})\Big)^{T}(\theta_{t+1}-\theta_{t}),\end{aligned} \end{equation}where $\theta_{\zeta_{t}}$ means a point between $\theta_{t}$ and $\theta_{t+1}$. Substituting \eqref{9i8u7y6} into \eqref{41.5} yields
	\begin{equation}\nonumber\begin{aligned}
    &g(\theta_{t+1})-g(\theta_{t})
	=-\Big(\alpha^{t-1}\nabla_{\theta_{1}}g(\theta_{1})^{\sf T}v_{1}+\sum_{i=1}^{t-1}\alpha^{t-i}(\nabla_{\theta_{i}}g(\theta_{i})-\nabla_{\theta_{i-1}}g(\theta_{i-1}))^{T}v_{i-1}\\&+\sum_{i=2}^{t}\alpha^{t-i}\epsilon_{i}\nabla_{\theta_{i}}g(\theta_{i})^{\sf T} \nabla_{\theta_{i}}g(\theta_{i},\xi_{i})\Big)+\Big(\nabla_{\theta_{\zeta_{t}}}g(\theta_{\zeta_{t}})-\nabla_{\theta_{t}}g(\theta_{t})\Big)^{T}(\theta_{t+1}-\theta_{t}).\end{aligned} \end{equation}
	%\begin{equation}\nonumber\begin{aligned}
	% g(\theta_{t+1})-g(\theta_{t})=&-\alpha^{t-1}\nabla_{\theta_{1}}g(\theta_{1})^{\sf T}v_{1}-\sum_{i=2}^{t}\alpha^{t-i}\epsilon_{i}\nabla_{\theta_{i}}g(\theta_{i})\nabla_{\theta_{i}}g(\theta_{i},\xi_{i})\\
	% &+\frac{1}{2}v_{t}^{\sf T}H_{\theta\theta}(\zeta_{t})v_{t}+\sum_{i=1}^{t-1}\alpha^{t-i}v_{i}^{\sf T}H_{\theta\theta}(\zeta_{i})v_{i}.\end{aligned} \end{equation}We take the summation, and get that
It follows that
\begin{equation}\label{42}\begin{aligned}g(\theta_{n+1})=&g(\theta_{1})+\sum_{t=1}^{n}\big(g(\theta_{t+1})-g(\theta_{t})\big)\\
	=&g(\theta_{1})-\frac{1-\alpha^{n}}{1-\alpha}\nabla_{\theta_{1}}g(\theta_{1})^{\sf T}v_{1}+\frac{1-\alpha^{n}}{1-\alpha}\epsilon_{1}\nabla_{\theta_{1}}g(\theta_{1})\nabla_{\theta_{1}}g(\theta_{1},\xi_{1})\\&-\sum_{t=1}^{n}\frac{1-\alpha^{n-t+1}}{1-\alpha}\epsilon_{t}\nabla_{\theta_{t}}g(\theta_{t})\nabla_{\theta_{t}}g(\theta_{t},\xi_{t})+\sum_{t=1}^{n}\Big(\nabla_{\theta_{\zeta_{t}}}g(\theta_{\zeta_{t}})-\nabla_{\theta_{t}}g(\theta_{t})\Big)^{T}(\theta_{t+1}-\theta_{t})\\
	&-\sum_{t=1}^{n}\frac{1-\alpha^{n-t+1}}{1-\alpha}\sum_{i=1}^{t-1}\alpha^{t-i}(\nabla_{\theta_{i}}g(\theta_{i})-\nabla_{\theta_{i-1}}g(\theta_{i-1}))^{T}v_{i-1}.\end{aligned} \end{equation}
	Take the mathematical expectation of \eqref{42}, and notice Assumption \ref{ass_g} 3), then we   have\begin{equation}\nonumber\begin{aligned}
	\Expect\big(g(\theta_{n+1})\big)\le& \Expect\big(g(\theta_{1})\big)+\frac{\alpha}{1-\alpha}\bigg|\nabla_{\theta_{1}}g(\theta_{1})^{\sf T}v_{1}\bigg|+\frac{1}{1-\alpha}\epsilon_{1}\Expect\big(\|\nabla_{\theta_{1}}g(\theta_{1})\|^{2}\big)\\
	&-\sum_{t=1}^{n}\epsilon_{t}\Expect\big(\|\nabla_{\theta_{t}}g(\theta_{t})\|^{2}\big)+c\sum_{t=1}^{n}\Expect\big(\|v_{t}\|^{2}\big)+\frac{2c}{1-\alpha}\sum_{t=1}^{n}\Expect\big(\|v_{t}\|^{2}\big).\end{aligned}\end{equation}
	From Lemma \ref{lem_mid}, it follows that for some positive constant $Q$,  
	\begin{equation}\nonumber\begin{aligned}
	&c\sum_{t=1}^{n}\Expect\big(\|v_{t}\|^{2}\big)+\frac{2c}{1-\alpha}\sum_{t=1}^{n}\Expect\big(\|v_{t}\|^{2}\big)<Q.
	\end{aligned}\end{equation} 
	Hence,
	\begin{equation}\nonumber\begin{aligned}\Expect\big(g(\theta_{n+1})\big)<&Q+\Expect\big(g(\theta_{1})\big)+\frac{\alpha}{1-\alpha}\bigg|\nabla_{\theta_{1}}g(\theta_{1})^{\sf T}v_{1}\bigg|+\frac{1}{1-\alpha}\epsilon_{1}\Expect\big(\|\nabla_{\theta_{1}}g(\theta_{1})\|^{2}\big)\\
	&-\sum_{t=1}^{n}\epsilon_{t}\Expect\big(\|\nabla_{\theta_{t}}g(\theta_{t})\|^{2}\big).
	\end{aligned}\end{equation}As a result, \begin{equation}\nonumber\begin{aligned}
	\sum_{t=1}^{n}\epsilon_{t}\Expect\big(\|\nabla_{\theta_{t}}g(\theta_{t})\|^{2}\big)<&Q+\Expect\big(g(\theta_{1})\big) +\frac{\alpha}{1-\alpha}\bigg|\nabla_{\theta_{1}}g(\theta_{1})^{\sf T}v_{1}\bigg|+\frac{1}{1-\alpha}\epsilon_{1}\Expect\big(\|\nabla_{\theta_{1}}g(\theta_{1})\|^{2}\big)\\
	&-\Expect\big(g(\theta_{n+1})\big)\\
	<&+\infty.\end{aligned}\end{equation}
	From Lemma \ref{lem_summation}, we have $\sum_{t=1}^{n}\epsilon_{t}\|\nabla_{\theta_{t}}g(\theta_{t})\|^{2}<+\infty \ a.s..$
	
	\subsection{Proof of Lemma~\ref{lem_conve}}\label{pf_lem_conve}
	We divide the proof into three steps. 
	
	The first step is to prove 
	\begin{equation}\nonumber\begin{aligned}
	\liminf_{n\rightarrow +\infty}\big\|\nabla_{\theta_{n}}g(\theta_{n})\big\|^2=0 \ a.s..
	\end{aligned}
	\end{equation}
	Suppose the above conclusion does not hold, i.e., 
	\begin{equation}\nonumber\begin{aligned}
	\liminf_{n\rightarrow +\infty}\big\|\nabla_{\theta_{n}}g(\theta_{n})\big\|^2>s^{2}>0 \ a.s.,
	\end{aligned}
	\end{equation}
	where $s$ is a random variable depending on sample paths. Then it holds that
	\begin{equation}
	\nonumber\begin{aligned}
	P\Bigg(\bigcup_{n=1}^{+\infty}\bigcap_{m=n}^{+\infty}\bigg(\big\|\nabla_{\theta_{m}}g(\theta_{m})\big\|^2>\frac{1}{4}s^2\bigg)\Bigg)=1 .\end{aligned}\end{equation}
	From \eqref{yuandian} and $\lim_{n\rightarrow+\infty}\zeta_{n}=\zeta<\infty \ a.s.$,   if $\big\|\nabla_{\theta_{n}}g(\theta_{n})\big\|^2>\frac{1}{4}s^2 $, it holds that
	\begin{equation}\nonumber
	\begin{aligned}
	&g(\theta_{n+1})\le \zeta_{n}-b\sum_{i=n}^{+\infty}\epsilon_{i}=-\infty.
	\end{aligned}
    \end{equation}
	It follows that
	\begin{equation}\nonumber
	\begin{aligned}
	&P\big(\lim_{n\rightarrow +\infty}g(\theta_{n+1})=-\infty\big)\ge P\Bigg(\bigcup_{n=1}^{+\infty}\bigcap_{m=n}^{+\infty}\bigg(\big\|\nabla_{\theta_{m}}g(\theta_{m})\big\|^2>\frac{1}{4}s^2\bigg)\Bigg)=1.
	\end{aligned}
	\end{equation}
	As a result, $\lim_{n\rightarrow+\infty}g(\theta_{n})=-\infty$, meaning   a contradiction. Hence,
	\begin{equation}\label{uy}
	\begin{aligned}\liminf_{n\rightarrow +\infty}\big\|\nabla_{\theta_{n}}g(\theta_{n})\big\|^2=0 \ a.s.
	\end{aligned}
	\end{equation}
	
	Step 2 is to prove that the set $\{\theta_{n}\}$ has an accumulation point contained in $J$ with probability one. From \eqref{uy}, we have $\forall \epsilon>0$
	\begin{equation}\nonumber\begin{aligned}
	P\Bigg(\bigcup_{n=1}^{+\infty}\bigcap_{m=n}^{+\infty}\bigg(\big\|\nabla_{\theta_{m}}g(\theta_{m})\big\|^2>\epsilon\bigg)\Bigg)=0.\end{aligned}
	\end{equation}
	Since $\|\nabla_{\theta}g(\theta)\|^{2}$ is continuous,  $\forall \delta>0$, it holds that
	\begin{equation}\nonumber
	\begin{aligned}
	P\Bigg(\bigcup_{n=1}^{+\infty}\bigcap_{m=n}^{+\infty}\bigg(\inf_{\theta\in J}\big\|\theta_{m}-\theta\|^2>\delta\bigg)\Bigg)=0.
	\end{aligned}
	\end{equation}
	In addition, under the given conditions, $J$ is a closed set. It holds that
	\begin{equation}\nonumber
	\begin{aligned}
	P\Bigg(\bigcup_{n=1}^{+\infty}\bigcap_{m=n}^{+\infty}\bigg(\min_{\theta\in J}\big\|\theta_{m}-\theta\|^2>\delta\bigg)\Bigg)=0. 
	\end{aligned}
	\end{equation}
	For convenience, let $\theta^*_n := \arg\min_{\theta\in J} \|\theta_n - \theta\|^2$, then
	\begin{equation}\label{lkj}\begin{aligned}
	&P\Bigg(\bigcap_{n=1}^{+\infty}\bigcup_{m=n}^{+\infty}\Bigg(\bigg(\big\|\theta_{m}-\theta_{m}^{*}\|\le\sqrt{\delta}\bigg)\Bigg)\Bigg)=1.
	\end{aligned}\end{equation}
It follows that
\begin{equation}\label{King}\begin{aligned}
	&P\Bigg(\bigcup_{\{m_{i}\} \in \mathcal{S}_{\mathbb{N}}}\bigcap_{i=1}^{+\infty}\bigg(\big\|\theta_{m_{i}}-\theta^{*}_{m_{i}}\big\|\le\sqrt{\delta}\bigg)\Bigg)=1,
	\end{aligned}\end{equation}
	%We inspecting an event
	%\begin{equation}\nonumber\begin{aligned}
	%\bigcap_{i=1}^{+\infty}\bigg(\big\|\theta_{m_{i}}-\theta^{*}_{m_{i}}\big\|\le\sqrt{\delta}\bigg),
	%\end{aligned}\end{equation}finding this event is a $tail \ event$. So through $zero-one \ law$, we know for all group $\{m_{i}\}$, there is
	%\begin{equation}\nonumber\begin{aligned}
	%P\Bigg(\bigcap_{i=1}^{+\infty}\bigg(\big\|\theta_{m_{i}}-\theta^{*}_{m_{i}}\big\|\le\sqrt{\delta}\bigg)\Bigg)=0\ or\ 1.			\end{aligned}\end{equation}
	%And through \eqref{King}, we know exist one group $\{m_{i}\}$, let
	%\begin{equation}\label{rab}\begin{aligned}
	%P\Bigg(\bigcap_{i=1}^{+\infty}\bigg(\big\|\theta_{m_{i}}-\theta^{*}_{m_{i}}\big\|\le\sqrt{\delta}\bigg)\Bigg)=1.			\end{aligned}\end{equation}
	where $\mathcal{S}_{\mathbb{N}}$ is the set of all infinite subsequences of $\mathbb{N}$.
	Since $\{\theta_{m_{i}}^{*}\}\subset J$,   $\{\theta_{m_{i}}^*\}$ is bounded. From the accumulative point principle,   the event
	\begin{equation}\nonumber
	\begin{aligned}
	\bigcap_{n=1}^{+\infty}\bigcup_{i,j\ge n}\bigg(\|\theta^{*}_{m_{i}}-\theta^{*}_{m_{j}}\|\le\sqrt{\delta}\bigg),
	\end{aligned}
	\end{equation}
	is a deterministic event, i.e., being true for every sample path. Thus
	\begin{equation}\nonumber\begin{aligned}
	&\bigcup_{\{m_{k_{i}}\} \in \mathcal{S}_{\mathbb{N}}^{(i)}}\bigcap_{i,j\geq 1}\bigg(\big\|\theta_{m_{k_{i}}}-\theta_{m_{k_{j}}}\big\|\le\sqrt{\delta}\bigg)
	\end{aligned}\end{equation}
	is a deterministic event as well, where $\mathcal{S}_{\mathbb{N}}^{(i)}$ is the set of all infinite subsequences of $\{m_{i}\}$. Due to the arbitrariness of $\{m_{i}\}$, we get that 
	\begin{equation}\label{poi}\begin{aligned}
	&\bigcap_{\{m_{i}\}\in\mathcal{S}_{\mathbb{N}}}\bigcup_{\{m_{k_{i}}\} \in \mathcal{S}_{\mathbb{N}}^{(i)}}\bigcap_{i,j\geq 1}\bigg(\big\|\theta_{m_{k_{i}}}-\theta_{m_{k_{j}}}\big\|\le\sqrt{\delta}\bigg)
	\end{aligned}\end{equation}
	is a deterministic event. Therefore, 
	\begin{align}
	&P\Bigg(\bigcup_{\{m_{i}\}\in\mathcal{S}_{\mathbb{N}}}\bigcap_{i,j\ge 1}\bigg(\big\|\theta^{*}_{m_{i}}-\theta^{*}_{m_{j}}\big\|<\sqrt{\delta}\bigg)\bigcap\bigg(\big\|\theta_{m_{i}}-\theta^{*}_{m_{i}}\big\|<\sqrt{\delta}\bigg)\bigcap\bigg(\big\|\theta_{m_{j}}-\theta^{*}_{m_{j}}\big\|<\sqrt{\delta}\bigg)\Bigg)\\\nonumber
	=&P\Bigg({\bigcup_{\{m_{i}\}\in\mathcal{S}_{\mathbb{N}}}}
	\bigcup_{\{m_{k_{i}}\}\in\mathcal{S}_{\mathbb{N}}^{(i)}}\bigcap_{i,j\ge 1}			
	\bigg(\big\|\theta^{*}_{m_{k_{i}}}-\theta^{*}_{m_{k_{j}}}\big\|<\sqrt{\delta}\bigg)\bigcap\bigg(\big\|\theta_{m_{k_{i}}}-\theta^{*}_{m_{k_{i}}}\big\|<\sqrt{\delta}\bigg)\\\nonumber
	&\quad\bigcap\bigg(\big\|\theta_{m_{k_{j}}}-\theta^{*}_{m_{k_{j}}}\big\|<\sqrt{\delta}\bigg)\Bigg)\\\nonumber
	\ge &P\Bigg(\bigcup_{\{m_{i}\}\in\mathcal{S}_{\mathbb{N}}}\bigcup_{\{m_{k_{i}}\}\in\mathcal{S}_{\mathbb{N}}^{(i)}}\bigcap_{i,j\ge 1}\bigg(\big\|\theta^{*}_{m_{k_{i}}}-\theta^{*}_{m_{k_{j}}}\big\|<\sqrt{\delta}\bigg)\bigcap\bigg(\bigcap_{i=1}^{+\infty}\big\|\theta_{m_{i}}-\theta^{*}_{m_{i}}\big\|<\sqrt{\delta}\bigg)\Bigg)\\\nonumber
	=&P\Bigg({\bigcup_{\{m_{i}\}\in\mathcal{S}_{\mathbb{N}}}}{\bigg(\bigcap_{i=1}^{+\infty}\big\|\theta_{m_{i}}-\theta^{*}_{m_{i}}\big\|<\sqrt{\delta}\bigg)\bigcap{\bigcup_{\{m_{k_{i}}\}\in\mathcal{S}_{\mathbb{N}}^{()}}}\bigcap_{i,j\ge 1}\bigg(\big\|\theta^{*}_{m_{k_{i}}}-\theta^{*}_{m_{k_{j}}}\big\|<\sqrt{\delta}\bigg)}\Bigg)\\\nonumber
	\ge& P\Bigg(\bigcup_{\{m_{i}\}\in\mathcal{S}_{\mathbb{N}}}\bigg(\bigcap_{i=1}^{+\infty}\big\|\theta_{m_{i}}-\theta^{*}_{m_{i}}\big\|<\sqrt{\delta}\bigg)\\&\nonumber\bigcap\bigg(\bigcap_{\{m_{i}\}\in\mathcal{S}_{\mathbb{N}}}\bigcup_{\{m_{k_{i}}\}\in\mathcal{S}_{\mathbb{N}}^{(i)}}\bigcap_{i,j\ge 1}\bigg(\big\|\theta^*_{m_{k_{i}}}-\theta^{*}_{m_{k_{j}}}\big\|<\sqrt{\delta}\bigg)\bigg)\Bigg)\\\nonumber
	=&P\Bigg(\bigg(\bigcap_{\{m_{i}\}\in\mathcal{S}_{\mathbb{N}}}\bigcup_{\{m_{k_{i}}\}\in\mathcal{S}_{\mathbb{N}}^{(i)}}\bigcap_{i,j\ge 1}\bigg(\big\|\theta^{*}_{m_{k_{i}}}-\theta^{*}_{m_{k_{j}}}\big\|<\sqrt{\delta}\bigg)\bigg)\\&\nonumber\bigcap\bigg(\bigcup_{\{m_{i}\}\in\mathcal{S}_{\mathbb{N}}}\bigg(\bigcap_{i=1}^{+\infty}\big\|\theta_{m_{i}}-\theta^{*}_{m_{i}}\big\|<\sqrt{\delta}\bigg)\bigg)\Bigg) \\\nonumber
	\end{align}
	Combine (\ref{King}) and (\ref{poi}), then we get that
	\begin{equation}\nonumber\begin{aligned}
	&P\Bigg(\bigg(\bigcap_{\{m_{k_{i}}\} \in \mathcal{S}_{\mathbb{N}}}\bigcup_{\{m_{i}\} \in \mathcal{S}_{\mathbb{N}}^{(i)}}\bigcap_{i,j\ge 1}\bigg(\big\|\theta^{*}_{m_{k_{i}}}-\theta^{*}_{m_{k_{j}}}\big\|<\sqrt{\delta}\bigg)\bigg)\\&\bigcap\bigg(\bigcup_{\{m_{i}\} \in \mathcal{S}_{\mathbb{N}}}\bigg(\bigcap_{i=1}^{+\infty}\big\|\theta_{m_{i}}-\theta^{*}_{m_{i}}\big\|<\sqrt{\delta}\bigg)\bigg)\Bigg)=1.
	\end{aligned}\end{equation}
	So
	\begin{equation}\nonumber\begin{aligned}
	&P\Bigg(\bigcup_{\{m_{i}\} \in \mathcal{S}_{\mathbb{N}}}\bigcap_{i,j\ge 1}\bigg(\big\|\theta^{*}_{m_{i}}-\theta^{*}_{m_{j}}\big\|<\sqrt{\delta}\bigg)\bigcap\bigg(\big\|\theta_{m_{i}}-\theta^{*}_{m_{i}}\big\|<\sqrt{\delta}\bigg)\\&\bigcap\bigg(\big\|\theta_{m_{j}}-\theta^{*}_{m_{j}}\big\|<\sqrt{\delta}\bigg)\Bigg)=1.
	\end{aligned}\end{equation}
	Thus it holds that
	\begin{equation}\nonumber\begin{aligned}
	&P\Bigg(\bigcap_{t=1}^{+\infty}\bigcup_{i,j>t}\Bigg(\bigg(\big\|\theta^{*}_{i}-\theta^{*}_{j}\big\|\le\sqrt{\delta}\bigg)\bigcap\bigg(\big\|\theta_{i}-\theta_{j}\big\|\le3\sqrt{\delta}\bigg)\Bigg)\\
	\geq& P\Bigg(\bigcap_{t=1}^{+\infty}\bigcup_{i,j>t}\Bigg(\bigg(\big\|\theta^{*}_{i}-\theta^{*}_{j}\big\|\le\sqrt{\delta}\bigg)\bigcap\bigg(\big\|\theta_{i}-\theta^{*}_{i}\big\|\le\sqrt{\delta}\bigg)\bigcap\bigg(\big\|\theta_{j}-\theta^{*}_{j}\big\|\le\sqrt{\delta}\bigg)\Bigg)\\ 
	=& P\Bigg(\bigcup_{\{m_{k_{i}}\} \in \mathcal{S}_{\mathbb{N}}}\bigcap_{i,j\ge 1}\bigg(\big\|\theta^{*}_{m_{i}}-\theta^{*}_{m_{j}}\big\|<\sqrt{\delta}\bigg)\bigcap\bigg(\big\|\theta_{m_{i}}-\theta^{*}_{m_{i}}\big\|<\sqrt{\delta}\bigg)\\&\bigcap\bigg(\big\|\theta_{m_{j}}-\theta^{*}_{m_{j}}\big\|<\sqrt{\delta}\bigg)\Bigg)\\
	=&1.
	\end{aligned}\end{equation}
	This means we can find two "same" convergent subsequences $\{\theta_{k_{n}}\}$ and $\{\theta_{k_{n}}^{*}\}$ with probability one. Because $J$ is a closed set and $\{\theta_{k_{n}}^{*}\}\subset J$, we know that there exists $\theta^{\prime\prime} \in J$ such that
	\begin{equation}\label{qw}\begin{aligned}
	\lim_{n\rightarrow +\infty}\theta_{k_{n}}=\theta^{''} a.s.
	\end{aligned}\end{equation}
	This indicates that we can find an accumlative point of $\{\theta_{n}\}$ in $J$ with probability one. 
	Denote the connected component containing $\theta^{''}$ by $J^{*}$. 
	
	Step 3 is to prove that there is a connected component $J^{*}$ of $J$ such that $\theta_{n}\rightarrow J^{*}$ a.s. From \eqref{qw}, we have
	\begin{equation}\nonumber
	\begin{aligned}
	P\Bigg(\bigcap_{n=1}^{+\infty}\bigcup_{m=n}^{+\infty}\Big(\big\|\theta_{m}-\theta^{''}\big\|<\delta^{'}\Big)\Bigg)=1.
	\end{aligned}
	\end{equation}
	Since $g(\theta)$ is continuous, we have
	\begin{equation}\nonumber
	\begin{aligned}
	&P\Bigg(\bigcap_{n=1}^{+\infty}\bigcup_{m=n}^{+\infty}\Big(\big|g(\theta_{m})-g(\theta^{''})\big|<\delta^{''}\Big)\Bigg)=1.
	\end{aligned}
	\end{equation}
	Since $\big\{g(\theta_{n})\big\}$ converges a.s., let $M$ be the limit, i.e., $\lim_{n\rightarrow+\infty}g(\theta_{n})=M\ a.s$. Hence $g(\theta^{''})=M$, and $g(\theta)=M$ for all $\theta \in J^{*}$. Define $A=\{\theta|g(\theta)=M\}$. From $g(\theta_{n})\rightarrow M$ a.s. and the continuity of $g(\theta)$, it holds that $\theta_{n}\rightarrow A$ a.s. Obviously, $J^{*}\subseteq A$. If $J^{*}=A$, then the conclusion follows. Now assume $J^{*}\subsetneqq A$. From the definition of $J^{*}$, it follows that there are two disjoint open sets $V$ and $H$ such that $A\subset H\bigcup V$, $J^{*}\subset V$, $A/J^{*}\subset H$, $H^{'}\bigcap V^{'}=\emptyset$ ($H^{'}$ and $V^{'}$ are closures of $H$ and $V$). So we just need to prove that $\theta_{n}\rightarrow V$ a.s. A contradiction argument is to be used. Suppose that $\{\theta_{n}\}$ has accumulation points in $V$ and $H$ simultaneously with a probability larger than zero. That is,
	\begin{equation}\label{ml}\begin{aligned}
	&P\Bigg(\bigg(\bigcap_{n=1}^{+\infty}\bigcup_{m=n}^{+\infty}\big(\theta_{m}\in H\big)\bigg)\bigcap\bigg(\bigcap_{s=1}^{+\infty}\bigcup_{t=s}^{+\infty}\big(\theta_{t}\in V\big)\bigg)\bigcap\bigg(\bigcup_{r=1}^{+\infty}\bigcap_{v=r}^{+\infty}\big(\theta_{v}\in H\cup V\big)\bigg)\Bigg)>0.
	\end{aligned}\end{equation}
	Expanding \eqref{ml}, we get 
	\begin{equation}\nonumber\begin{aligned}
	&P\Bigg(\bigcup_{r=1}^{+\infty}\Bigg(\bigcap_{(m,n) \in \mathbb{N}\times \mathbb{N}}\bigcup_{i\geq m,j\geq n}\bigg(\big(\theta_{i}\in V\big)\bigcap\big(\theta_{j}\in H\big)\bigg)\Bigg)\bigcap\bigg(\bigcap_{v=r}^{+\infty}\big(\theta_{v}\in H\cup V\big)\bigg)\Bigg)\Bigg)>0.
	\end{aligned}\end{equation}
	Note that the event
	\begin{equation}\nonumber\begin{aligned}
	\bigcap_{(m,n)\in \mathbb{N}\times\mathbb{N}}\bigcup_{i\geq m,j\geq n}\bigg(\big(\theta_{i}\in V\big)\bigcap\big(\theta_{j}\in H\big)\bigg)
	\end{aligned}\end{equation} can be written as
	\begin{equation}\nonumber\begin{aligned}
	\bigcup_{\{\alpha_{i} \}\in \mathcal{S}_{\mathbb{N}},\{\beta_{i}\}\in \mathcal{S}_{\mathbb{N}}}\bigcap_{i=1}^{+\infty}\bigg(\big(\theta_{\alpha_{i}}\in V\big)\bigcap\big({\theta_{\beta_{i}}\in H}\big)\bigg),
	\end{aligned}\end{equation} 
	where $\{\alpha_{i}\}\cap\{\beta_{i}\}=\emptyset$. The event
	\begin{equation}\label{][p}\begin{aligned}
	&\Bigg(\bigcap_{(m,n) \in \mathbb{N}\times \mathbb{N}}\bigcup_{i\geq m,j\geq n}\bigg(\big(\theta_{i}\in V\big)\bigcap\big(\theta_{j}\in H\big)\bigg)\Bigg)\bigcap\bigg(\bigcap_{v=r}^{+\infty}\big(\theta_{v}\in H\cup V\big)\bigg)
	\end{aligned}\end{equation} means $\forall t\geq r$, $\theta_{t}$ must belong to one of $H$ and $V$. So $\{\alpha_{i}\}\cup\{\beta_{i}\}$ must include the set $\{r,r+1,r+2,...\}$. Denote $\{\alpha_{i}\}$ and $\{\beta_{i}\}$ as $\{\alpha_{i}^{(r)}\}$ and $\{\beta_{i}^{(r)}\}$, and then  \eqref{][p} can be written as 
	\begin{equation}\nonumber\begin{aligned}
	&\bigcup_{\{\alpha_{i}^{(r)}\}\in \mathcal{S}_{\mathbb{N}},\{\beta_{i}^{(r)}\}\in \mathcal{S}_{\mathbb{N}}}\bigcap_{i=1}^{+\infty}\bigg(\big(\theta_{\alpha_{i}^{(r)}}\in V\big)\bigcap\big({\theta_{\beta_{i}^{(r)}}\in H}\big)\bigg).
	\end{aligned}\end{equation}
	So we get
	\begin{equation}\label{quen}\begin{aligned}
	&P\Bigg(\bigcup_{r=1}^{+\infty}
	\bigcup_{\{\alpha_{i}^{(r)}\},\{\beta_{i}^{(r)}\}}\bigcap_{i=1}^{+\infty}\bigg(\big(\theta_{\alpha_{i}^{(r)}}\in V\big)\bigcap\big({\theta_{\beta_{i}^{(r)}}\in H}\big)\bigg)\Bigg)>0.
	\end{aligned}\end{equation}
	It follows from $V^{'}\bigcap H^{'}=\emptyset$ that
	\begin{equation}\label{yt}
	\inf_{\theta^{(1)}\in H,\theta^{(2)}\in V}\big\|\theta^{(1)}-\theta^{(2)}\big\|=s>0,
	\end{equation}
	where $s$ is a constant. By Lemma \ref{lem_mid}, it holds that $\forall s>0$
	\begin{equation}\label{vb}\begin{aligned}
	P\Bigg(\bigcap_{n=1}^{+\infty}\bigcup_{m=n}^{+\infty}\|\theta_{m}-\theta_{m+1}\|\geq s\Bigg)=0.
	\end{aligned}\end{equation}
	
	Now we consider $\{\alpha_{i}^{(r)}+1\}$. In fact, $\big\{\alpha_{i}^{(r)}+1\big\}$ and $\big\{\beta_{i}^{(r)}\big\}$ have a finite number of identical elements with probability one. If not, then
	\begin{equation}\nonumber
	P\Bigg(\bigcap_{n=1}^{+\infty}\bigcup_{m=n}^{+\infty}\bigg(\big(\theta_{\alpha_{m}^{(r)}}\in V\big)\bigcap\big(\theta_{\alpha_{m}^{(r)}+1}\in H\big)\bigg)\Bigg)>0.
	\end{equation}By \eqref{yt}, it follows that
	\begin{equation}\nonumber\begin{aligned}
	&P\Bigg(\bigcap_{n=1}^{+\infty}\bigcup_{m=n}^{+\infty}\big\|\theta_{m}-\theta_{m+1}\big\|\geq s\Bigg)\\
	\geq& P\Bigg(\bigcap_{n=1}^{+\infty}\bigcup_{m=n}^{+\infty}\Big\|\theta_{\alpha_{m}^{(r)}}-\theta_{\alpha_{m}^{(r)}+1}\Big\| \geq s\Bigg)
	\\
	\geq& P\Bigg(\bigcap_{n=1}^{+\infty}\bigcup_{m=n}^{+\infty}\bigg(\Big(\theta_{\alpha_{m}^{(r)}}\in V\Big)\bigcap\Big(\theta_{\alpha_{m}^{(r)}+1}\in H\Big)\bigg)\Bigg)>0,
	\end{aligned}\end{equation}
	which however contradicts with \eqref{vb}. Hence $\big\{\alpha_{i}^{(r)}+1\big\}$ and $\big\{\beta_{i}^{(r)}\big\}$ have a finite number of identical elements with probability one. From $\big\{\alpha_{i}^{(r)}\big\}\bigcup\big\{\beta_{i}^{(r)}\big\}\supset\{r,r+1,r+2,...\}$, we know that \begin{equation}\label{mlkj}\begin{aligned}
	P\Bigg(\bigcup_{n=1}^{+\infty}\bigcap_{m=n}^{+\infty}\big(\alpha_{m}^{(r)}+1\big)\in \big\{\alpha_{i}^{(r)}\big\}\Bigg)=1.
	\end{aligned}\end{equation}  Hence $\big\{\beta_{i}^{(r)}\big\}$ is a finite sequence with probability one, otherwise it contradicts with \eqref{quen}. Thus,  the assumption does not hold, and $\theta_{n}\rightarrow V$ a.s. Furthermore, $\theta_{n}\rightarrow J^{*}$. Therefore, there is a connected component $J^{*}$ of $J$ such that
	\begin{equation}\nonumber\begin{aligned}
	\lim\limits_{n\rightarrow\infty}d(\theta_{n},J^*)=0.
	\end{aligned}\end{equation}
	
			\subsection{Proof of Theorem \ref{thm_converg1}}\label{append_thm_converg}
	First of all we aim to prove that $g(\theta_{n+1})$ is convergent almost surely. Divide \eqref{42} into four parts as follows.
	\begin{equation}\label{47}\begin{aligned}
	&g(\theta_{n+1})=\underbrace{g(\theta_{1})-\frac{1-\alpha^{n}}{1-\alpha}\nabla_{\theta_{1}}g(\theta_{1})^{\sf T}v_{1}+\frac{1-\alpha^{n}}{1-\alpha}\epsilon_{1}\nabla_{\theta_{1}}g(\theta_{1})\nabla_{\theta_{1}}g(\theta_{1},\xi_{1})}_{(A)}\\&-\underbrace{\sum_{t=1}^{n}\frac{1-\alpha^{n-t+1}}{1-\alpha}\epsilon_{t}\nabla_{\theta_{t}}g(\theta_{t})^{\sf T}\nabla_{\theta_{t}}g(\theta_{t},\xi_{t})}_{(B)}+\underbrace{\sum_{t=1}^{n}\Big(\nabla_{\theta_{\zeta_{t}}}g(\theta_{\zeta_{t}})-\nabla_{\theta_{t}}g(\theta_{t})\Big)^{T}(\theta_{t+1}-\theta_{t})}_{(C)}\\
	&\underbrace{-\sum_{t=1}^{n}\frac{1-\alpha^{n-t+1}}{1-\alpha}\sum_{i=1}^{t-1}\alpha^{t-i}(\nabla_{\theta_{i}}g(\theta_{i})-\nabla_{\theta_{i-1}}g(\theta_{i-1}))^{T}v_{i-1}}_{(D)}.
	\end{aligned}\end{equation}Due to $\alpha<1$,  $\alpha^{n}$ is tending to zero, which ensures the convergence  of  part $(A)$. For $(C)$, we consider the absolute value of $\sum_{t=n}^{m}$. It follows from Assumption \ref{ass_g} 3) that
	\begin{equation}\nonumber\begin{aligned}
	&\bigg|\sum_{t=n}^{m}\Big(\nabla_{\theta_{\zeta_{t}}}g(\theta_{\zeta_{t}})-\nabla_{\theta_{t}}g(\theta_{t})\Big)^{T}(\theta_{t+1}-\theta_{t})\bigg|\le\sum_{t=n}^{m}\bigg|\Big(\nabla_{\theta_{\zeta_{t}}}g(\theta_{\zeta_{t}})-\nabla_{\theta_{t}}g(\theta_{t})\Big)^{T}(\theta_{t+1}-\theta_{t})\bigg|\\&\le\sum_{t=n}^{m}\Big\|\nabla_{\theta_{\zeta_{t}}}g(\theta_{\zeta_{t}})-\nabla_{\theta_{t}}g(\theta_{t})\Big\|\|v_{t}\|\le c\sum_{t=n}^{m}\|v_{t}\|^{2}.\end{aligned}\end{equation}
	Through Lemma \ref{lem_mid}, we get $\sum_{t=n}^{m}\|v_{t}\|^{2}\rightarrow 0 \ \ a.s.$,  leading to
	\begin{equation}\nonumber\begin{aligned}
	&\bigg|\sum_{t=n}^{m}\Big(\nabla_{\theta_{\zeta_{t}}}g(\theta_{\zeta_{t}})-\nabla_{\theta_{t}}g(\theta_{t})\Big)^{T}(\theta_{t+1}-\theta_{t})\bigg|\rightarrow 0\ \ a.s..\end{aligned}\end{equation}Through $ Cauchy's \ test\  for\  convergence$, we know that $(C)$ is convergent almost surely. By using the same function, it holds that $(D)$ is convergent almost surely. For $(B)$, we have
	\begin{equation}\label{48.9}\begin{aligned}
	(B)=&\sum_{t=1}^{n}\frac{1-\alpha^{n-t+1}}{1-\alpha}\epsilon_{t}\nabla_{\theta_{t}}g(\theta_{t})^{\sf T}\nabla_{\theta_{t}}g(\theta_{t},\xi_{t})\\
	=&\sum_{t=1}^{n}\frac{1-\alpha^{n-t+1}}{1-\alpha}\epsilon_{t}\|\nabla_{\theta_{t}}g(\theta_{t})\|^{2}+\sum_{t=1}^{n}\frac{1-\alpha^{n-t+1}}{1-\alpha}\epsilon_{t}\nabla_{\theta_{t}}g(\theta_{t})^{\sf T}\big(\nabla_{\theta_{t}}g(\theta_{t},\xi_{t})-\nabla_{\theta_{t}}g(\theta_{t})\big).\end{aligned}\end{equation}
	From Lemma \ref{lem_mid}, it follows that $(C)$ is convergent almost surely. For $(B)$, it holds that 
	\begin{equation}\label{48}\begin{aligned}
	(B)=&\sum_{t=1}^{n}\frac{1-\alpha^{n-t+1}}{1-\alpha}\epsilon_{t}\nabla_{\theta_{t}}g(\theta_{t})^{\sf T}\nabla_{\theta_{t}}g(\theta_{t},\xi_{t})\\
	=&\sum_{t=1}^{n}\frac{1-\alpha^{n-t+1}}{1-\alpha}\epsilon_{t}\|\nabla_{\theta_{t}}g(\theta_{t})\|^{2}+\sum_{t=1}^{n}\frac{1-\alpha^{n-t+1}}{1-\alpha}\epsilon_{t}\nabla_{\theta_{t}}g(\theta_{t})^{\sf T}\big(\nabla_{\theta_{t}}g(\theta_{t},\xi_{t})-\nabla_{\theta_{t}}g(\theta_{t})\big).\end{aligned}\end{equation}
	By Lemma \ref{lem_deriv}, we know\begin{equation}\nonumber
	\begin{aligned}\sum_{t=1}^{n}\frac{1-\alpha^{n-t+1}}{1-\alpha}\epsilon_{t}\|\nabla_{\theta_{t}}g(\theta_{t})\|^{2}<+\infty\ \ a.s..\end{aligned}\end{equation}From Lemmas \ref{lem_summation_MDS} and \ref{lem_deriv} it follows that \begin{equation}\nonumber\begin{aligned}&\sum_{t=1}^{n-1}\frac{1-\alpha^{n-t}}{1-\alpha}\epsilon_{t+1}\nabla_{\theta_{t}}g(\theta_{t})^{\sf T}\big(\nabla_{\theta_{t}}g(\theta_{t},\xi_{t})-\nabla_{\theta_{t}}g(\theta_{t})\big)\end{aligned}\end{equation}
	is convergent a.s. Thus $(B)$ is convergent a.s., and $g(\theta_{n+1})$ is convergent a.s.. Substituting \eqref{48} into \eqref{47} leads to \begin{equation}\nonumber\begin{aligned}
	&g(\theta_{n+1})\le\zeta_{n}'-\sum_{t=1}^{n}\epsilon_{t}\big\|\nabla_{\theta_{t}}g(\theta_{t})\big\|^{2},
	\end{aligned}\end{equation}
	where $\{\zeta_{n}'\}$ is defined as follows
	\begin{equation}\nonumber\begin{aligned}
	\zeta_{n}'=&g(\theta_{1})-\frac{1-\alpha^{n}}{1-\alpha}\nabla_{\theta_{1}}g(\theta_{1})^{\sf T}v_{1}+\frac{1-\alpha^{n}}{1-\alpha}\epsilon_{1}\nabla_{\theta_{1}}g(\theta_{1})\nabla_{\theta_{1}}g(\theta_{1},\xi_{1})\\
	&-\sum_{t=1}^{n}\frac{1-\alpha^{n-t+1}}{1-\alpha}\epsilon_{t}\nabla_{\theta_{t}}g(\theta_{t})^{\sf T}\big(\nabla_{\theta_{t}}g(\theta_{t},\xi_{t})-\nabla_{\theta_{t}}g(\theta_{t})\big)+\sum_{t=1}^{n}\Big(\nabla_{\theta_{\zeta_{t}}}g(\theta_{\zeta_{t}})-\nabla_{\theta_{t}}g(\theta_{t})\Big)^{T}(\theta_{t+1}-\theta_{t})\\&-\sum_{t=1}^{n}\frac{1-\alpha^{n-t+1}}{1-\alpha}\sum_{i=1}^{t-1}\alpha^{t-i}(\nabla_{\theta_{i}}g(\theta_{i})-\nabla_{\theta_{i-1}}g(\theta_{i-1}))^{T}v_{i-1}.
	\end{aligned}
	\end{equation}
	we know $\{\zeta_{n}'\}$ is convergent a.s. It follows from Lemma \ref{lem_conve}   that there exists a connected component $J^{*}$ of $J$ such that
	$						\lim\limits_{n\rightarrow\infty}d(\theta_{n},J^*)=0.$
	%			\begin{equation}\nonumber\begin{aligned}
	%			\inf_{\theta\in J^{*}}\big\|\theta_{n}-\theta\|^2\stackrel{n\rightarrow\infty}{\longrightarrow}0, \qquad \text{a.s.} \end{aligned}\end{equation}

	\subsection{Proof of Theorem \ref{thm_rate}}\label{append_pf_thm_rate}
% 	\begin{equation}\label{okm,ljh}\begin{aligned}
%     &g(\theta_{t+1})-g(\theta_{t})
% 	=-\Big(\alpha^{t-1}\nabla_{\theta_{1}}g(\theta_{1})^{\sf T}v_{1}+\sum_{i=1}^{t-1}\alpha^{t-i}(\nabla_{\theta_{i}}g(\theta_{i})-\nabla_{\theta_{i-1}}g(\theta_{i-1}))^{T}v_{i-1}\\&+\sum_{i=2}^{t}\alpha^{t-i}\epsilon_{i}\nabla_{\theta_{i}}g(\theta_{i})^{\sf T} \nabla_{\theta_{i}}g(\theta_{i},\xi_{i})\Big)+\Big(\nabla_{\theta_{\zeta_{t}}}g(\theta_{\zeta_{t}})-\nabla_{\theta_{t}}g(\theta_{t})\Big)^{T}(\theta_{t+1}-\theta_{t}).\end{aligned} \end{equation}

	First of all we can get that
	\begin{equation}\label{okm,ljh}\begin{aligned}
    &g(\theta_{t+1})-g(\theta_{t})
	=-\Big(\alpha^{t-1}\nabla_{\theta_{1}}g(\theta_{1})^{\sf T}v_{1}+\sum_{i=1}^{t-1}\alpha^{t-i}(\nabla_{\theta_{i}}g(\theta_{i})-\nabla_{\theta_{i-1}}g(\theta_{i-1}))^{T}v_{i-1}\\&+\sum_{i=2}^{t}\alpha^{t-i}\epsilon_{i}\nabla_{\theta_{i}}g(\theta_{i})^{\sf T} \nabla_{\theta_{i}}g(\theta_{i},\xi_{i})\Big)+\Big(\nabla_{\theta_{\zeta_{t}}}g(\theta_{\zeta_{t}})-\nabla_{\theta_{t}}g(\theta_{t})\Big)^{T}(\theta_{t+1}-\theta_{t}).\end{aligned} \end{equation}
	From Theorem \ref{thm_converg1},  it follows that $g(\theta_{n})$ is convergent a.s., and it is orbitally convergent to $g_{i} \ (i=1,2,...,N)$ a.s.. Then it holds that
	\begin{equation}\nonumber\begin{aligned}
    &\lim_{n\rightarrow+\infty}g(\theta_{n})=\sum_{i=1}^{N}I_{i}g_{i}\ \ a.s.,\end{aligned} \end{equation} where
    \begin{equation}\nonumber\begin{aligned}
	I_{i}=\left\{\begin{array}{rcl}
	1 & & {\lim_{n\rightarrow+\infty}g(\theta_{n}})=g_{i}\\ \\
	0 & & {\lim_{n\rightarrow+\infty}g(\theta_{n})\neq g_{i}}
	\end{array} \right.
	\end{aligned}\end{equation}For convenient, we let $g^{*}=\sum_{i=1}^{\infty}I_{i}g_{i}$. Then we make some transformation on \eqref{okm,ljh}
	\begin{equation}\label{okm,ljh..}\begin{aligned}
    &\big(g(\theta_{t+1})-g^{*}\big)-\big(g(\theta_{t})-g^{*}\big)
	=-\Big(\alpha^{t-1}\nabla_{\theta_{1}}g(\theta_{1})^{\sf T}v_{1}+\sum_{i=1}^{t-1}\alpha^{t-i}(\nabla_{\theta_{i}}g(\theta_{i})-\nabla_{\theta_{i-1}}g(\theta_{i-1}))^{T}v_{i-1}\\&+\sum_{i=2}^{t}\alpha^{t-i}\epsilon_{i}\nabla_{\theta_{i}}g(\theta_{i})^{\sf T} \nabla_{\theta_{i}}g(\theta_{i},\xi_{i})\Big)+\Big(\nabla_{\theta_{\zeta_{t}}}g(\theta_{\zeta_{t}})-\nabla_{\theta_{t}}g(\theta_{t})\Big)^{T}(\theta_{t+1}-\theta_{t}).\end{aligned} \end{equation}
	
	Then we make some transformations, take absolute values,   take the mathematical expectation, and use same techniques in Theorem \ref{thm_converg1}. Since the sampling noise follows a uniform distribution, there is $\Expect(\|\nabla_{\theta_{n}}g(\theta_{n},\xi_{n})\|^{2})\le M\Expect(\|\nabla_{\theta_{n}}g(\theta_{n})\|^{2})$. So it follows that  $F'_{n}-F'_{n-1}\le P_{n-1}-Q_{n-1}$,  where $P_{n-1}$, $Q_{n-1}$, and $F'_{n}$ are defined as follows
	\begin{equation}\label{ewnmg}\begin{aligned}
	P_{n-1}=&\frac{1}{(1-\alpha)^{2}}\bigg(\frac{1}{2-\alpha}\bigg)^{n-1}\epsilon_{1}\Expect\big(\|\nabla_{\theta_{1}}g(\theta_{1})\|^{2}\big)+Z(1)L\bigg(\frac{1}{2-\alpha}\bigg)^{n-1}\Expect\big(g(\theta_{1})\big)\\
	&-\frac{\alpha^{n}(2-\alpha)}{1-\alpha}Z(n+1)\Expect\big(\nabla_{\theta_{1}}g(\theta_{1})^{\sf T}v_{1}\big)+\frac{c\alpha^{n}\alpha^{\delta}(2-\alpha)}{2(1-\alpha)(1-\alpha^{\delta})}Z(1)\Expect\big(\|v_{0}\|^{2}\big),
	&%+\frac{cL}{\alpha^{1-\delta}(1-\alpha^{\delta})}\sum_{t=1}^{n}\bigg(\frac{1}{2-\alpha}\bigg)^{n-t}\big(\epsilon_{t}-\epsilon_{t-1}\big),
	\\F'_{n}=&\sum_{t=1}^{n}\bigg(\frac{1}{2-\alpha}\bigg)^{n-t}Z(t+1)\Expect\Big(e_{t+1}^{(n)}\big|g(\theta_{t+1})-g^{*}\big|\Big),
	\\
	Q_{n-1}=&\frac{1}{(1-\alpha)^{2}}\sum_{t=1}^{n}\bigg(\frac{1}{2-\alpha}\bigg)^{n-t}\epsilon_{t}\Expect\big(\|\nabla_{\theta_{t}}g(\theta_{t})\|^{2}\big),\end{aligned}\end{equation}
	where $g^{*}=\inf_{\theta\in \mathbb{R}^{N}}g(\theta)$. Then we have
	\begin{equation}\label{090909_a}\begin{aligned}
	F'_{n}\le \bigg(1-\frac{Q_{n-1}}{F'_{n-1}}\bigg)F'_{n-1}+P_{n-1}.
	\end{aligned}\end{equation}

	Derive ${Q_{n}}/{\epsilon_{n+1}F'_{n}}$ as follows
	\begin{equation}\nonumber
	\begin{aligned}
	&\frac{Q_{n}}{\epsilon_{n+1}F'_{n}}=\frac{1}{(1-\alpha)^{2}}\frac{\sum_{t=1}^{n+1}\big(\frac{1}{2-\alpha}\big)^{n+1-t}\epsilon_{t+1}\Expect\big(\|\nabla_{\theta_{t}}g(\theta_{t})\|^{2}\big)}{\sum_{t=1}^{n}\big(\frac{1}{2-\alpha}\big)^{-t}Z(t+1)Z(t+1)\Expect\Big(e_{t+1}^{(n)}\big|g(\theta_{t+1})-g^{*}\big|\Big)}\\
	&=\frac{1}{(1-\alpha)^{2}}\frac{\big(\frac{1}{2-\alpha}\big)^{n}\epsilon_{2}\Expect\big(\|\nabla_{\theta_{1}}g(\theta_{1})\|^{2}\big)+\sum_{t=1}^{n}\big(\frac{1}{2-\alpha}\big)^{-t}\Expect\big(\|\nabla_{\theta_{t+1}}g(\theta_{t+1})\|^{2}\big)}{\sum_{t=1}^{n}\big(\frac{1}{2-\alpha}\big)^{-t}Z(t+1)Z(t+1)\Expect\Big(e_{t+1}^{(n)}\big|g(\theta_{t+1})-g^{*}\big|\Big)}.
	\end{aligned}
	\end{equation} It follows that
	\begin{equation}\label{rvy}
	\begin{aligned}
	&\liminf_{n\rightarrow+\infty}\frac{Q_{n}}{\epsilon_{n+1}F'_{n}}\\
	&=\frac{1}{(1-\alpha)^{2}}\liminf_{n\rightarrow+\infty}\frac{\big(\frac{1}{2-\alpha}\big)^{n}\epsilon_{2}\Expect\big(\|\nabla_{\theta_{1}}g(\theta_{1})\|^{2}\big)+\sum_{t=1}^{n}\big(\frac{1}{2-\alpha}\big)^{-t}\Expect\big(\|\nabla_{\theta_{t+1}}g(\theta_{t+1})\|^{2}\big)}{\sum_{t=1}^{n}\big(\frac{1}{2-\alpha}\big)^{-t}Z(t+1)\Expect\Big(e_{t+1}^{(n)}\big|g(\theta_{t+1})-g^{*}\big|\Big)}\\
	&\ge\frac{1}{(1-\alpha)^{2}}\liminf_{n\rightarrow+\infty}\frac{\sum_{t=1}^{n}\big(\frac{1}{2-\alpha}\big)^{-t}\Expect\big(\|\nabla_{\theta_{t+1}}g(\theta_{t+1})\|^{2}\big)}{\sum_{t=1}^{n}\big(\frac{1}{2-\alpha}\big)^{-t}Z(t+1)\Expect\Big(e_{t+1}^{(n)}\big|g(\theta_{t+1})-g^{*}\big|\Big)}.
	\end{aligned}
	\end{equation}
	From \eqref{98uj}, we have
	\begin{equation}\label{rvy0}\begin{aligned}
	&e_{t+1}^{(n)}=1+\frac{2c\epsilon_{t}f(n-t)}{\alpha^{1-\delta}(1-\alpha^{\delta})}\\&=1+\frac{2c\epsilon_{t}}{\alpha^{1-\delta}(1-\alpha^{\delta})}\bigg(\sum_{k=1}^{n-t}\big(\alpha(2-\alpha)\big)^{k}-\alpha^{\delta}\sum_{k=1}^{n-t}\big(\alpha^{1+\delta}(2-\alpha)\big)^{k}\bigg)\\&\le
% 	1+\frac{2c\epsilon_{t}(2-\alpha)}{\alpha^{-\delta}(1-\alpha^{\delta})(1-\alpha(2-\alpha))}=
	1+C_{\alpha}\epsilon_{t},
	\end{aligned}\end{equation}where $C_{\alpha}=\frac{2c(2-\alpha)}{\alpha^{-\delta}(1-\alpha^{\delta})(1-\alpha(2-\alpha))}$. 
	Substituting \eqref{rvy0} into \eqref{rvy} yields
	\begin{equation}\label{rvyp}\begin{aligned}
	&\liminf_{n\rightarrow+\infty}\frac{Q_{n}}{\epsilon_{n+1}F'_{n}}\\&\ge\frac{1}{(1-\alpha)^{2}}\liminf_{n\rightarrow+\infty}\frac{\sum_{t=1}^{n}\Big(\frac{1}{2-\alpha}\Big)^{-t}\Expect\big(\|\nabla_{\theta_{t+1}}g(\theta_{t+1})\|^{2}\big)}{\sum_{t=1}^{n}\Big(\frac{1}{2-\alpha}\Big)^{-t}Z(t+1)\Expect\Big(e_{t+1}^{(n)}\big|g(\theta_{t+1})-g^{*}\big|\Big)}.
	\end{aligned}\end{equation}
	Then we proceed with the proof  under two different cases, namely,  $\sum_{t=1}^{n}\big(\frac{1}{2-\alpha}\big)^{-t}\Expect\big(\|\nabla_{\theta_{t+1}}g(\theta_{t+1})\|^{2}\big)=+\infty$ and $\sum_{t=1}^{n}\big(\frac{1}{2-\alpha}\big)^{-t}\Expect\big(\|\nabla_{\theta_{t+1}}g(\theta_{t+1})\|^{2}\big)<+\infty$.
	
	First, if $\sum_{t=1}^{n}\big(\frac{1}{2-\alpha}\big)^{-t}\Expect\big(\|\nabla_{\theta_{t+1}}g(\theta_{t+1})\|^{2}\big)=+\infty$ (the proof for   this condition is up to \eqref{.5}). It follows from Assumption \ref{ass_goiop} 2) and the uniform convergence and $O'stolz\  theorem$ that
	\begin{equation}\label{rvyp1}\begin{aligned}
	&\liminf_{n\rightarrow+\infty}\frac{Q_{n}}{\epsilon_{n+1 }F'_{n}}\\&\ge\frac{1}{(1-\alpha)^{2}}\liminf_{n\rightarrow+\infty}\frac{\sum_{t=1}^{n}\big(\frac{1}{2-\alpha}\big)^{-t}\Expect\big(\|\nabla_{\theta_{t+1}}g(\theta_{t+1})\|^{2}\big)}{\sum_{t=1}^{n}\big(\frac{1}{2-\alpha}\big)^{-t}Z(t+1)\Expect\Big(\big(1+C_{\alpha}\epsilon_{t}\big)\big|g(\theta_{t+1})-g^{*}\big|\Big)}\\&\ge\frac{1}{(1-\alpha)^{2}}\liminf_{n\rightarrow+\infty}\frac{\Expect\big(\|\nabla_{\theta_{n+1}}g(\theta_{n+1})\|^{2}\big)}{Z(n+1)\Expect\Big(\big(1+C_{\alpha}\epsilon_{t}\big)\big|g(\theta_{t+1})-g^{*}\big|\Big)}\\&\ge\frac{1}{(1-\alpha)^{2}}\liminf_{n\rightarrow+\infty}\frac{s}{Z(n+1)\big(1+C_{\alpha}\epsilon_{n}\big)} =\frac{s}{p(1-\alpha)^{2}},
	\end{aligned}\end{equation} 
	where $p=\exp\bigg\{\sum_{k=1}^{\infty}M\epsilon_{k}^{2}\bigg\}$. By  using O'stolz theorem on $Q_{n}/F_{n}^{'}$,  it follows that 
	\begin{equation}\nonumber\begin{aligned}
	&\lim_{n\rightarrow+\infty}\frac{Q_{n}}{F'_{n}}=\frac{1}{(1-\alpha)^{2}}\lim_{n\rightarrow+\infty}\frac{\sum_{t=1}^{n}\big(\frac{1}{2-\alpha}\big)^{-t}\epsilon_{t+1}\Expect\big(\|\nabla_{\theta_{t+1}}g(\theta_{t+1})\|^{2}\big)}{\sum_{t=1}^{n}\big(\frac{1}{2-\alpha}\big)^{-t}Z(t+1)\Expect\Big(e_{t+1}^{(n)}\big|g(\theta_{t+1})-g^{*})\big|\Big)}\\&\le\frac{1}{(1-\alpha)^{2}}\lim_{n\rightarrow+\infty}\frac{\sum_{t=1}^{n}\big(\frac{1}{2-\alpha}\big)^{-t}\epsilon_{t+1}\Expect\big(\|\nabla_{\theta_{t+1}}g(\theta_{t+1})\|^{2}\big)}{\sum_{t=1}^{n}\big(\frac{1}{2-\alpha}\big)^{-t}Z(t+1)\Expect\big|g(\theta_{t+1})-g^{*}\big|}\\&=\frac{1}{(1-\alpha)^{2}}\lim_{n\rightarrow+\infty}\frac{\epsilon_{n}\Expect\big(\|\nabla_{\theta_{n}}g(\theta_{n})\|^{2}\big)}{Z(n+1)\Expect\big|g(\theta_{n+1})-g^{*}\big|}\\&\le \frac{1}{(1-\alpha)^{2}}\lim_{n\rightarrow+\infty}\frac{t\epsilon_{n}\Expect\big(g(\theta_{n+1})-g^{*}\big)}{Z(n+1)\Expect\big|g(\theta_{n+1})-g^{*}\big|}=0.\end{aligned}\end{equation}
	So we   conclude that $\exists n_{0}\in\mathbb{N}_+$, such that $\forall n\ge n_{0}$,  
	\begin{equation}\label{eq_pf_ineq}
	\begin{aligned}
	&\frac{Q_{n}}{F'_{n}}<1-\frac{3-\alpha}{2(2-\alpha)}.
	\end{aligned}\end{equation}
	From \eqref{090909_a}, it follows that 
	\begin{equation}\label{0o9i8u7y}\begin{aligned}
	&F'_{n}\le \prod_{i=n_{0}}^{n-1}\bigg(1-\frac{Q_{i}}{F'_{i}}\bigg)F'_{n_0}+\sum_{t=n_{0}+1}^{n}\prod_{i=t}^{n-1}\bigg(1-\frac{Q_{i}}{F'_{i}}\bigg)P_{t-1}.
	\end{aligned}\end{equation}
	Using the inequality $\ln(1+x)\le x$ and \eqref{rvyp1}, we   get 
	\begin{equation}\label{plokijuh}\begin{aligned}
	&\prod_{i=n_{0}}^{n-1}\bigg(1-\frac{Q_{i}}{F'_{i}}\bigg)=\exp\Bigg(\sum_{i=n_{0}}^{n-1}\ln\bigg(1-\frac{Q_{i}}{F'_{i}}\bigg)\Bigg)\le\exp\bigg(-\sum_{i=n_{0}}^{n-1}\frac{Q_{i}}{F'_{i}}\bigg)\le ke^{-\sum_{i=1}^{n}\frac{s\epsilon_{i}}{p(1-\alpha)^{2}}}\\&=O\Big(e^{-\sum_{i=1}^{n}\frac{s\epsilon_{i}}{p(1-\alpha)^{2}}}\Big),
	\end{aligned}\end{equation}where $k$ is a constant. It follows that
	\begin{equation}\label{qawsedrf}\begin{aligned}
	&\frac{\sum_{t=n_{0}+1}^{n-1}\prod_{i=t}^{n-1}\big(1-\frac{Q_{i}}{F'_{i}}\big)P_{t-1}}{\prod_{i=n_{0}}^{n-1}\big(1-\frac{Q_{i}}{F'_{i}}\big)F'_{1}}=\sum_{t=n_{0}+1}^{n}\frac{P_{t-1}}{\prod_{i=n_{0}}^{t-1}\big(1-\frac{Q_{i}}{F'_{i}}\big)}.
	\end{aligned}\end{equation}
	From \eqref{ewnmg}, we have
	\begin{equation}\label{.,}\begin{aligned}
	P_{t-1}=&\frac{1}{(1-\alpha)^{2}}\bigg(\frac{1}{2-\alpha}\bigg)^{t-1}\epsilon_{1}\Expect\big(\|\nabla_{\theta_{1}}g(\theta_{1})\|^{2}\big)+Z(1)L\bigg(\frac{1}{2-\alpha}\bigg)^{t-1}\Expect\big(g(\theta_{1})\big)\\
	&+\frac{c\alpha^{t}\alpha^{\delta}(2-\alpha)}{2(1-\alpha)(1-\alpha^{\delta})}Z(1)\Expect\big(\|v_{0}\|^{2}\big)=\bar p\bigg(\frac{1}{2-\alpha}\bigg)^{t-1}+\bar q\alpha^{t-1},
	\end{aligned}\end{equation}
	where $\bar p=\frac{1}{(1-\alpha)^{2}}\epsilon_{1}\Expect\big(\|\nabla_{\theta_{1}}g(\theta_{1})\|^{2}\big)+Z(1)L\Expect\big(g(\theta_{1})\big)$ and $\bar q=\frac{c\alpha^{\delta+1}(2-\alpha)}{2(1-\alpha)(1-\alpha^{\delta})}Z(1)\Expect\big(\|v_{0}\|^{2}\big)$. Substituting \eqref{.,} into \eqref{qawsedrf} and noting \eqref{eq_pf_ineq} yield 
	\begin{equation}\label{.1}\begin{aligned}
	&\frac{\sum_{t=n_{0}+1}^{n-1}\prod_{i=t}^{n-1}\big(1-\frac{Q_{i}}{F'_{i}}\big)P_{t-1}}{\prod_{i=n_{0}}^{n-1}\big(1-\frac{Q_{i}}{F'_{i}}\big)F'_{1}}\\&=\sum_{t=n_{0}+1}^{n}\frac{P_{t-1}}{\prod_{i=n_{0}}^{t-1}\big(1-\frac{Q_{i}}{F'_{i}}\big)}<\sum_{t=n_{0}+1}^{n}\frac{\bar p\Big(\frac{1}{2-\alpha}\Big)^{t-1}}{\prod_{i=n_{0}}^{t-1}\big(1-\frac{Q_{i}}{F'_{i}}\big)}+\sum_{t=n_{0}+1}^{n}\frac{\bar q\alpha^{t-1}}{\prod_{i=n_{0}}^{t-1}\big(1-\frac{Q_{i}}{F'_{i}}\big)}\\&\le\sum_{t=n_{0}+1}^{n}\frac{\bar p\big(\frac{1}{2-\alpha}\big)^{t-1}}{\big(\frac{3-\alpha}{2(2-\alpha)}\big)^{t-n_{0}}}+\sum_{t=n_{0}+1}^{n}\frac{\bar q\alpha^{t-1}}{\big(\frac{3-\alpha}{2(2-\alpha)}\big)^{t-n_{0}}}\\&<\bigg(\frac{3-\alpha}{2(2-\alpha)}\bigg)^{n_{0}-1}(\bar p+\bar q)\sum_{t=n_{0}+1}^{n}\bigg(\frac{2}{3-\alpha}\bigg)^{t-1}<c_0,
%	\bigg(\frac{3-\alpha}{1-\alpha}\bigg)^{n_{0}}(\bar p+\bar q).
	\end{aligned}\end{equation}
	where $c_0$ is a positive constant.
It follows  that
	\begin{equation}\label{.2}\begin{aligned}
	&\sum_{t=n_{0}+1}^{n-1}\prod_{i=t}^{n-1}\bigg(1-\frac{Q_{i}}{F'_{i}}\bigg)P_{t-1}<c_0\prod_{i=n_{0}}^{n-1}\bigg(1-\frac{Q_{i}}{F'_{i}}\bigg)F'_{1}=O\Bigg(\prod_{i=n_{0}}^{n-1}\bigg(1-\frac{Q_{i}}{F'_{i}}\bigg)F'_{1}\Bigg)=O\Big(e^{-\sum_{i=1}^{n}\frac{s\epsilon_{i}}{p(1-\alpha)^{2}}}\Big).
	\end{aligned}\end{equation}
	Combine \eqref{0o9i8u7y}, \eqref{plokijuh} and \eqref{.2}, then we get
	\begin{equation}\label{.3}\begin{aligned}
	&F'_{n}\le \prod_{i=n_{0}}^{n-1}\bigg(1-\frac{Q_{i}}{F'_{i}}\bigg)F'_{1}+\sum_{t=n_{0}+1}^{n-1}\prod_{i=t}^{n-1}\bigg(1-\frac{Q_{i}}{F'_{i}}\bigg)P_{t-1}=O\Big(e^{-\sum_{i=1}^{n}\frac{s\epsilon_{i}}{p(1-\alpha)^{2}}}\Big).
	\end{aligned}\end{equation}
	In addition, we have
	\begin{equation}\label{.5}\begin{aligned}
	\Expect\big(g(\theta_{t+1})-g^{*}\big)<F'_{n}=O\Big(e^{-\sum_{i=1}^{n}\frac{s\epsilon_{i}}{p(1-\alpha)^{2}}}\Big).
	\end{aligned}\end{equation}
	
	If $\sum_{t=1}^{n}\Big(\frac{1}{2-\alpha}\Big)^{-t}\Expect\big(\|\nabla_{\theta_{t+1}}g(\theta_{t+1})\|^{2}\big)<+\infty$, it holds that
	\begin{equation}\nonumber\begin{aligned}
	\lim_{n\rightarrow+\infty}\bigg(\frac{1}{2-\alpha}\bigg)^{-n}\Expect\big(\|\nabla_{\theta_{n+1}}g(\theta_{n+1})\|^{2}=0,
	\end{aligned}\end{equation}that is
	\begin{equation}\nonumber\begin{aligned}
	\Expect\big(\|\nabla_{\theta_{n+1}}g(\theta_{n+1})\|^{2}\big)=O\Bigg(\bigg(\frac{1}{2-\alpha}\bigg)^{n}\Bigg).
	\end{aligned}\end{equation}
	Under Assumption \ref{ass_goiop} 2), we have
	\begin{equation}\nonumber\begin{aligned}
	&\limsup_{n\rightarrow+\infty}\frac{\Expect\big(g(\theta_{n+1})-g^{*}\big)}{\big(\frac{1}{2-\alpha}\big)^{n}}\le\limsup_{n\rightarrow+\infty}\frac{\Expect \big(\big\|\nabla_{\theta_{n+1}}g(\theta_{n+1})\big\|^{2}\big)}{s\big(\frac{1}{2-\alpha}\big)^{n}}\le\frac{1}{s}\frac{\Expect\big(\big\|\nabla_{\theta{n+1}}g(\theta_{n+1})\big\|^{2}\big)}{\big(\frac{1}{2-\alpha}\big)^{n}}=0. \end{aligned}\end{equation}
	It follows that
	\begin{equation}\nonumber\begin{aligned}
	&\Expect\big|g(\theta_{n+1})-g^{*}\big|=O\Bigg(\bigg(\frac{1}{2-\alpha}\bigg)^{n}\Bigg). \end{aligned}\end{equation}
	Now we compare $\big(\frac{1}{2-\alpha}\big)^{n}$ with $e^{-\sum_{i=1}^{n}\frac{s\epsilon_{i}}{p(1-\alpha)^{2}}}$. It holds that
	\begin{equation}\nonumber\begin{aligned}
	&\frac{\big(\frac{1}{2-\alpha}\big)^{n}}{\exp{\big(-\sum_{i=1}^{n}\frac{s\epsilon_{i}}{p(1-\alpha)^{2}}}\big)}=\frac{\exp{(-n\ln(2-\alpha)})}{\exp{\big(-\sum_{i=1}^{n}\frac{s\epsilon_{i}}{p(1-\alpha)^{2}}}\big)}=\exp{\bigg(\sum_{i=1}^{n}\frac{s\epsilon_{i}}{p(1-\alpha)^{2}}-n\ln(2-\alpha)\bigg)}. 
	\end{aligned}\end{equation}
It follows	from Assumption~\ref{assum8} that  $\epsilon_{n}\rightarrow 0$, indicating  $\sum_{i=1}^{n}\frac{s\epsilon_{i}}{p(1-\alpha)^{2}}-n\ln(2-\alpha)<0$ when $n$ is sufficiently large. Thus, we have
	\begin{equation}\nonumber\begin{aligned}
	&\limsup_{n\rightarrow+\infty}\frac{\big(\frac{1}{2-\alpha}\big)^{n}}{\exp{\big(-\sum_{i=1}^{n}\frac{s\epsilon_{i}}{p(1-\alpha)^{2}}}\big)}=\limsup_{n\rightarrow+\infty}\exp{\bigg(\sum_{i=1}^{n}\frac{s\epsilon_{i}}{p(1-\alpha)^{2}}-n\ln(2-\alpha)\bigg)}<1. 
	\end{aligned}\end{equation}
Then it holds that
	\begin{equation}\label{.90}\begin{aligned}
	&\Expect\big|g(\theta_{n+1})-g^{*}\big|=O\Bigg(\bigg(\frac{1}{2-\alpha}\bigg)^{n}\Bigg)=O\Big(e^{-\sum_{i=1}^{n}\frac{s\epsilon_{i}}{p(1-\alpha)^{2}}}\Big). \end{aligned}\end{equation}
	Combining \eqref{.90} and \eqref{.5} leads to
	\begin{equation}\nonumber\begin{aligned}
	&\Expect\big|g(\theta_{n+1})-g^{*}\big|=O\Big(e^{-\sum_{i=1}^{n}\frac{s\epsilon_{i}}{p(1-\alpha)^{2}}}\Big), \end{aligned}\end{equation}
	if $\sum_{t=1}^{n}\big(\frac{1}{2-\alpha}\big)^{-t}\Expect\big(\|\nabla_{\theta_{t+1}}g(\theta_{t+1})\|^{2}\big)<+\infty$. The above bound holds trivially if $\sum_{t=1}^{n}\big(\frac{1}{2-\alpha}\big)^{-t}\Expect\big(\|\nabla_{\theta_{t+1}}g(\theta_{t+1})\|^{2}\big)=+\infty$.
	It follows from Lemma \ref{pouikm,l} that
	\begin{equation}\nonumber\begin{aligned}
	&\Expect\big(\big\|\nabla_{\theta_{n}}g(\theta_{n})\big\|^{2}\big)=O\Big(e^{-\sum_{i=1}^{n}\frac{s\epsilon_{i}}{p(1-\alpha)^{2}}}\Big) \end{aligned}\end{equation}

\section{Convergence of AdaGrad}\label{append_pf_adgrad}
    The following lemmas are used for the proof of Theorem~\ref{thm_converg_adagrad}.
    \begin{lem}\label{lemiuy8}
    {Suppose $f(x) \in C^{1}\ \ (x\in \mathbb{R}^{N})$ with $f(x)>-\infty$ and its gradient satisfying the following Lipschitz condition
		\begin{equation}\nonumber\begin{aligned}
		\big\|\nabla f(x)-\nabla f(y)\big\|\le c\|x-y\|, 
		\end{aligned}
		\end{equation} then $\forall \ x_{0}\in \ \mathbb{R}^{N}$, there is
		\begin{equation}\nonumber\begin{aligned}
		\big\|\nabla f(x_{0})\big\|^{2}\le {2c}\big(f(x_{0})-f^{*}\big), 
		\end{aligned}
		\end{equation}where $f^{*}=inf_{x\in \ \mathbb{R}^{N}}f(x)$}

	\end{lem}
	\begin{lem}\label{lem6}
		Suppose $\{\theta_{n}\}$ is a sequence generated by AdaGrad in  \eqref{AdaGrad}, and Assumptions \ref{ass_noise} and   \ref{ass_g_poi} hold. If $\big\|\nabla_{\theta_{n}}g(\theta)\big\|^{2}>a$ where $a$ is given in Assumption~\ref{ass_g_poi} 3), then for any $ n\in\mathbb{N}_+,\theta_{1}\in \mathbb{R}^{N}$, and  $\epsilon\in(0,\frac{1}{2})$,     it holds that 
		\begin{equation}\nonumber
		\begin{aligned}
		\frac{g(\theta_{n+1})}{S_{n+1}^{\epsilon}}-\frac{g(\theta_{n})}{S_{n}^{\epsilon}}
		\le&\frac{\alpha_0}{2}(M+1)\Bigg(\frac{\big\|\nabla_{\theta_{n-1}}g(\theta_{n-1})\big\|^{2}}{S_{n-1}^{\frac{1}{2}+\epsilon}}-\frac{\big\|\nabla_{\theta_{n}}g(\theta_{n})\big\|^{2}}{S_{n}^{\frac{1}{2}+\epsilon}}\Bigg)\\
		&-\frac{\alpha_0}{20}\frac{\big\|\nabla_{\theta_{n}}g(\theta_{n})\big\|^{2}}{S_{n-1}^{\frac{1}{2}+\epsilon}}
		+\frac{\alpha_0}{20}\Bigg(\frac{\big\|\nabla_{\theta_{n-1}}g(\theta_{n-1})\big\|^{2}}{S_{n-2}^{\frac{1}{2}+\epsilon}}
		-\frac{\big\|\nabla_{\theta_{n}}g(\theta_{n})\big\|^{2}}{S_{n-1}^{\frac{1}{2}+\epsilon}}\Bigg)\\
		&+4M^{2}\alpha_0^{3}c^{2}\frac{\big\|\nabla_{\theta_{n-1}}g(\theta_{n-1},\xi_{n-1})\big\|^{2}}{S_{n-1}^{\frac{3}{2}+\epsilon}}+\frac{c\alpha_0^{2}}{2}\frac{\big\|\nabla_{\theta_{n}}g(\theta_{n},\xi_{n})\big\|^{2}}{S_{n}^{1+\epsilon}}+X_{n}^{(\epsilon)}+Y_{n}^{(\epsilon)},
		\end{aligned}\end{equation}
		where 
% 		$X_{n}^{(\epsilon)}$ and $Y_{n}^{(\epsilon)}$ are defined as follow
		\begin{equation}\nonumber\begin{aligned}
		&X_{n}^{(\epsilon)}=\frac{\alpha_0}{2}\frac{1}{S_{n-1}^{\frac{1}{2}+\epsilon}}\nabla_{\theta_{n}}g(\theta_{n})^{T}\big(\nabla_{\theta_{n}}g(\theta_{n})-\nabla_{\theta_{n}}g(\theta_{n},\xi_{n})\big)
		\\&Y_{n}^{(\epsilon)}=\frac{\alpha_0}{2}\Bigg(\frac{1}{M+1}\frac{\Expect\Big(\big\|\nabla_{\theta_{n}}g(\theta_{n},\xi_{n})\big\|^{2}\Big|\mathscr{F}_{n-1}\Big)}{S_{n-1}^{\frac{1}{2}+\epsilon}}-\frac{1}{M+1}\frac{\big\|\nabla_{\theta_{n}}g(\theta_{n},\xi_{n})\big\|^{2}}{S_{n-1}^{\frac{1}{2}+\epsilon}}\Bigg).
		\end{aligned}\end{equation}
	\end{lem}
	
	\begin{lem}\label{lem7}
		Suppose $\{\theta_{n}\}$ is a sequence generated by AdaGrad in  \eqref{AdaGrad}, and Assumptions \ref{ass_noise} and   \ref{ass_g_poi} hold.  If $\big\|\nabla_{\theta_{n}}g(\theta)\big\|^{2}\leq a$ where $a$ is given in Assumption~\ref{ass_g_poi} 3), then for any $ n\in\mathbb{N}_+,\theta_{1}\in \mathbb{R}^{N}$, and  $\epsilon\in(0,\frac{1}{2})$,     it holds that 
		\begin{equation}\nonumber\begin{aligned}
		\frac{g(\theta_{n+1})}{S_{n+1}^{\epsilon}}-\frac{g(\theta_{n})}{S_{n}^{\epsilon}}\le&-\frac{\alpha_0\big\|\nabla_{\theta_{n}}g(\theta_{n})\big\|^{2}}{20S_{n-1}^{\frac{1}{2}+\epsilon}}+\frac{\alpha_0^{3}c^{2}(M+1)^{2}}{2}\frac{\big\|\nabla_{\theta_{n-1}}g(\theta_{n-1},\xi_{n-1})\big\|^{2}}{S_{n-1}^{1+\epsilon}}\\&+\frac{(M+1)\alpha_0^{3}c^{2}}{2S_{n-1}^{\frac{3}{2}+\epsilon}}\big\|\nabla_{\theta_{n-1}}g(\theta_{n-1},\xi_{n-1})\big\|^{2}+\frac{\alpha_0}{20}\Bigg(\frac{\big\|\nabla_{\theta_{n-1}}g(\theta_{n-1})\big\|^{2}}{S_{n-2}^{\frac{1}{2}+\epsilon}}-\frac{\big\|\nabla_{\theta_{n}}g(\theta_{n})\big\|^{2}}{S_{n-1}^{\frac{1}{2}+\epsilon}}\Bigg)\\&+\frac{\alpha_0(M+1)}{2}\Bigg(\frac{\big\|\nabla_{\theta_{n-1}}g(\theta_{n-1})\big\|^{2}}{S_{n-1}^{\frac{1}{2}+\epsilon}}-\frac{\big\|\nabla_{\theta_{n}}g(\theta_{n})\big\|^{2}}{S_{n}^{\frac{1}{2}+\epsilon}}\Bigg)+\frac{\alpha_0 a(M+1)}{2}\Bigg(\frac{1}{S_{n-1}^{\frac{1}{2}+\epsilon}}-\frac{1}{S_{n}^{\frac{1}{2}+\epsilon}}\Bigg)\\&+\frac{c\alpha_0^{2}}{2}\frac{\big\|\nabla_{\theta_{n}}g(\theta_{n},\xi_{n})\big\|^{2}}{S_{n}^{1+\epsilon}}+A_{n}^{(\epsilon)}+B_{n}^{(\epsilon)},
		\end{aligned}\end{equation}
		where 
% 		$A_{n}^{(\epsilon)}$ and $B_{n}^{(\epsilon)}$ are defined as follow
		\begin{equation}\nonumber\begin{aligned}
		&A_{n}^{(\epsilon)}=\frac{\alpha_0}{S_{n-1}^{\frac{1}{2}+\epsilon}}\Big(\big\|\nabla_{\theta_{n}}g(\theta_{n})\big\|^{2}-\nabla_{\theta_{n}}g(\theta_{n})^{T}\nabla_{\theta_{n}}g(\theta_{n},\xi_{n})\Big)
		\\&B_{n}^{(\epsilon)}=\frac{\alpha_0 \big\|\nabla_{\theta_{n}}g(\theta_{n})\big\|^{2}}{2a(M+1)S_{n-1}^{\frac{1}{2}+\epsilon}}\bigg(\big\|\nabla_{\theta_{n}}g(\theta_{n},\xi_{n})\big\|^{2}-\Expect\Big(\big\|\nabla_{\theta_{n}}g(\theta_{n})\big\|^{2}\Big|\mathscr{F}_{n-1}\Big)\bigg).
		\end{aligned}\end{equation}
	\end{lem}
	
	\begin{lem}\label{lem8}
		Suppose $\{\theta_{n}\}$ is a sequence generated by AdaGrad in  \eqref{AdaGrad}, and Assumptions \ref{ass_noise} and   \ref{ass_g_poi} hold. Then $\forall n\in\mathbb{N}_+,$ $\forall \theta_{1}\in \mathbb{R}^{N}$ $\forall\epsilon\in(0,\frac{1}{2})$, it holds that
		\begin{equation}\nonumber\begin{aligned}
		&\sum_{k=3}^{n}\frac{\big\|\nabla_{\theta_{k}}g(\theta_{k})\big\|^{2}}{S_{k-1}^{\frac{1}{2}+\epsilon}}<+\infty\ \ a.s..
		\end{aligned}\end{equation}
	\end{lem}
	
	\begin{lem}\label{lem9}
		Suppose $\{\theta_{n}\}$ is a sequence generated by AdaGrad in  \eqref{AdaGrad}, and Assumptions \ref{ass_noise} and  \ref{ass_g_poi} hold. Then $\forall n\in\mathbb{N}_+$ $\forall \theta_{1}\in \mathbb{R}^{N}$ $\forall\epsilon\in(0,\frac{1}{2})$, $\exists \zeta<+\infty$, it holds that 
		\begin{equation}\nonumber\begin{aligned}
		&\frac{g(\theta_{n+1})-g^{*}}{S_{n+1}^{\epsilon}}\le\zeta<+\infty\ \ a.s.,
		\end{aligned}\end{equation}which $g^{*}=\inf_{\theta\in \mathbb{R}^{N}}g(\theta)$.
	\end{lem}
	
	\begin{lem}\label{lem10}
		Suppose $\{\theta_{n}\}$ is a sequence generated by AdaGrad in  \eqref{AdaGrad} subject to $S_{n}=\sum_{k=1}^{n}\|\nabla_{\theta_{k}}g(\theta_{k},\xi_{k})\|^{2}=+\infty\ \ a.s.$. Under Assumptions \ref{ass_noise} and   \ref{ass_g_poi},  $\forall n\in\mathbb{N}_+$ $\forall \theta_{1}\in \mathbb{R}^{N}$, $\forall\epsilon_{0}\in(0,\frac{3}{8})$, it holds that 
		\begin{equation}\nonumber\begin{aligned}
		\frac{\big\|\nabla_{\theta_{n}}g(\theta_{n})\big\|^{2}}{S_{n-1}^{\epsilon_{0}}}\rightarrow 0 \ \ a.s..
		\end{aligned}\end{equation}
	\end{lem}
	
    \begin{lem}\label{lem5}
        Suppose that $\{X_{n}\}\in \mathbb{R}^{N}$ is a vector sequence and $f(x)\in C^{1}$ is a monotonically non-increasing non-negative function with $\int_{a}^{+\infty}f(x)dx<\infty$ ($\forall a>0$).  Then  $\forall N\in\mathbb{N}_+$, it holds that
        \begin{equation}\nonumber\begin{aligned}
        \sum_{n=1}^{N}\|X_{n}\|^{2}f\bigg(\sum_{k=1}^{n}\|X_{k}\|^{2}\bigg)<\int_{\|X_{1}\|^{2}}^{\sum_{k=1}^{N}\|X_{k}\|^{2}}f(x)dx<\sum_{n=1}^{N}\|X_{n}\|^{2}f\bigg(\sum_{k=1}^{n-1}\|X_{k}\|^{2}\bigg).
        \end{aligned}\end{equation}
    \end{lem}
    {
    \subsection{Proof Outline of Theorem 3}\label{pf_outline2}
    Like the proof of mSGD, the proof for AdaGrad is also in light of the Lyapunov method. We aim to prove $\nabla g(\theta_{n})\rightarrow 0\ a.s.$, and then to get  $\theta_{n}\rightarrow J^{*}\ a.s.$
    The key step to prove convergence of AdaGrad is to show
    \begin{equation}\nonumber\begin{aligned}
        -\Expect\bigg(\frac{\alpha_{0}\nabla_{\theta_{n}}g(\theta_{n},\xi_{n})^{T}\nabla_{\theta_{n}}g(\theta_{n},\xi_{n})}{\sqrt{S_{n}}}\bigg|\mathscr{F}_{n}\bigg)\le 0.
        \end{aligned}\end{equation}
    but the learning rate of AdaGrad is a random variable and it is not conditionally independent of $\nabla_{\theta_{n}}g(\theta_{n},\xi_{n})$, meaning that
    \begin{equation}\nonumber\begin{aligned}
        -\Expect\bigg(\frac{\alpha_{0}\nabla_{\theta_{n}}g(\theta_{n},\xi_{n})^{T}\nabla_{\theta_{n}}g(\theta_{n},\xi_{n})}{\sqrt{S_{n}}}\bigg|\mathscr{F}_{n}\bigg)\neq- \frac{\alpha_{0}}{\sqrt{S_{n}}}\big\|\nabla_{\theta_{n}}g(\theta_{n})\big\|^{2}\le 0.
        \end{aligned}\end{equation}

In the following, we provide the proof outline of Theorem 3.

    Step 1: This step is to ensure
    \begin{equation}\nonumber\begin{aligned}
        -\Expect\bigg(\frac{\alpha_{0}\nabla_{\theta_{n}}g(\theta_{n},\xi_{n})^{T}\nabla_{\theta_{n}}g(\theta_{n},\xi_{n})}{\sqrt{S_{n}}}\bigg|\mathscr{F}_{n}\bigg)\le 0.
        \end{aligned}\end{equation} 
        {We are able to obtain the following equation
        \begin{equation}\nonumber\begin{aligned}
        &-\frac{\nabla_{\theta_{n}}g(\theta_{n},\xi_{n})^{T}\nabla_{\theta_{n}}g(\theta_{n})}{\sqrt{S_{n}}}=\underbrace{\frac{1}{2\sqrt{S_{n}}}\bigg\|\frac{1}{\sqrt{M+1}}\nabla_{\theta_{n}}g(\theta_{n},\xi_{n})-\sqrt{M+1}\nabla_{\theta_{n}}g(\theta_{n})\bigg\|^{2}}_{(K)}\\&-\underbrace{\frac{M+1}{2\sqrt{S_{n}}}\big\|\nabla_{\theta_{n}}g(\theta_{n})\big\|^{2}-\frac{1}{2(M+1)\sqrt{S_{n}}}\big\|\nabla_{\theta_{n}}g(\theta_{n},\xi_{n})\big\|^{2}}_{(L)}.
        \end{aligned}\end{equation}
        }
        For $(K)$, due to $S_{n-1}\le S_{n}$, it follows that \begin{equation}\nonumber\begin{aligned}
        (K)&=\frac{1}{2\sqrt{S_{n}}}\bigg\|\frac{1}{\sqrt{M+1}}\nabla_{\theta_{n}}g(\theta_{n},\xi_{n})-\sqrt{M+1}\nabla_{\theta_{n}}g(\theta_{n})\bigg\|^{2}\\&\le  \frac{1}{2\sqrt{S_{n-1}}}\bigg\|\frac{1}{\sqrt{M+1}}\nabla_{\theta_{n}}g(\theta_{n},\xi_{n})-\sqrt{M+1}\nabla_{\theta_{n}}g(\theta_{n})\bigg\|^{2} .
        \end{aligned}\end{equation}Note that $S_{n-1}$ is conditional independent on $\nabla_{\theta_{n}}g(\theta_{n},\xi_{n})$, thus it holds that
        \begin{equation}\nonumber\begin{aligned}
        &\Expect{\Bigg(\frac{1}{\sqrt{S_{n-1}}}\bigg\|\frac{1}{\sqrt{M+1}}\nabla_{\theta_{n}}g(\theta_{n},\xi_{n})-\sqrt{M+1}\nabla_{\theta_{n}}g(\theta_{n})\bigg\|^{2}\Bigg|\mathscr{F}_{n}\Bigg)}\\
        &=\frac{\alpha_0}{2}\frac{1}{\sqrt{S_{n-1}}}\bigg(\frac{1}{M+1}\Expect\Big(\big\|\nabla_{\theta_{n}}g(\theta_{n},\xi_{n})\big\|^{2}\Big |\mathscr{F}_{n}\Big)+(M-1)\big\|\nabla_{\theta_{n}}g(\theta_{n})\big\|^{2}\bigg).
        \end{aligned}\end{equation}  Next we prove $(K)$ can be controlled by $(L)$ according to Lemmas \ref{lem6} and \ref{lem7}. These two lemmas deal with the cases of $\|\nabla_{\theta_{n}}g(\theta_{n})\|\le a$ and $\|\nabla_{\theta_{n}}g(\theta_{n})\|>a$, respectively. In these two lemmas, we introduce a constant $\epsilon$ for the reason as stated in Step 2 in the following.  
    
    Step 2: Through Lemmas \ref{lem6} and \ref{lem7}, we obtain that \[\sum_{k=3}^{+\infty}\|\nabla_{k}g(\theta_{k})\|^{2}/S_{k-1}^{1/2+\epsilon}<\zeta +\sum_{k=3}^{+\infty}\|\nabla_{k}g(\theta_{k},\xi_{k})\|^{2}/S_{k}^{1+2\epsilon}.\] From Lemma \ref{lem5}, we obtain Lemma \ref{lem8}, stating that $\sum_{k=3}^{+\infty}\|\nabla_{k}g(\theta_{k},\xi_{k})\|^{2}/S_{k}^{1+2\epsilon}<+\infty$. Note that if we do not introduce $\epsilon$, the term $\sum_{k=3}^{+\infty}\|\nabla_{k}g(\theta_{k},\xi_{k})\|^{2}/S_{k}^{1+2\epsilon}$ will become $\sum_{k=3}^{+\infty}\|\nabla_{k}g(\theta_{k},\xi_{k})\|^{2}/S_{k} =O(\ln S_{n})$, and this term may not be bounded.
    
    Step 3:  Lemma \ref{lem8}   ensures that $\nabla_{\theta_{n}}g(\theta_{n})$ has a subsequence satisfying $\nabla_{\theta_{k_{n}}}g(\theta_{k_{n}})\rightarrow 0$ a.s. By using the recursion formula $g(\theta_{n+1})-g(\theta_{n})\le 1/S_{n}+P_{n}$, where $\sum_{n=1}^{+\infty}P_{n}<+\infty$ a.s., and Lemma \ref{pouikm,l}, we obtain $\nabla_{\theta_{n}}g(\theta_{n})\rightarrow0$ a.s. Consequently, $\theta_{n}\rightarrow J^{*}$.
    }

\subsection{Proof of Lemma~\ref{lemiuy8}}
	{For $\forall x\in \mathbb{R}^{N}$,	we define function
	\begin{equation}\nonumber
	\begin{aligned}
	g(t)=f\bigg(x+t\frac{x'-x}{\|x'-x\|}\bigg),
	\end{aligned}
	\end{equation}where $x'$ is a constant point such that   $x'-x$ is parallel to $\nabla f(x)$. By taking the derivative, we obtain
	\begin{equation}\label{qcxzd}
	\begin{aligned}
	g'(t)=\nabla_{x+t\frac{x'-x}{\|x'-x\|}}f\bigg(x+t\frac{x'-x}{\|x'-x\|}\bigg)^{T}\frac{x'-x}{\|x'-x\|}.
	\end{aligned}
	\end{equation}Through the Lipschitz condition of $\nabla f(x)$, we get $\forall t_{1}, \ t_{2}$
	\begin{equation}\nonumber
	\begin{aligned}
	&\big|g'(t_{1})-g'(t_{2})\big|=\Bigg|\Bigg(\nabla_{x+t\frac{x'-x}{\|x'-x\|}}f\bigg(x+t_{1}\frac{x'-x}{\|x'-x\|}\bigg)-\nabla_{x+t\frac{x'-x}{\|x'-x\|}}f\bigg(x+t_{2}\frac{x'-x}{\|x'-x\|}\bigg)\Bigg)^{T}\frac{x'-x}{\|x'-x\|}\\&\le\Bigg\|\nabla_{x+t\frac{x'-x}{\|x'-x\|}}f\bigg(x+t_{1}\frac{x'-x}{\|x'-x\|}\bigg)-\nabla_{x+t\frac{x'-x}{\|x'-x\|}}f\bigg(x+t_{2}\frac{x'-x}{\|x'-x\|}\bigg)\Bigg\|\bigg\|\frac{x'-x}{\|x'-x\|}\bigg\|\le c |t_{1}-t_{2}|.
	\end{aligned}
	\end{equation}So $g'(t)$ satisfies the Lipschitz  condition, and we have $\inf_{t\in \mathbb{R}}g(t)\geq\inf_{x\in\mathbb{R}^{N}}f(x)>-\infty$. Let $g^{*}=\inf{x\in_{\mathbb{R}}}g(x)$, then it holds that for $\forall \ t_{0}\in \ \mathbb{R},$
	\begin{equation}\label{qwcdfs}
	\begin{aligned}
	g(0)-g^{*}\geq g(0)-g(t_{0}).
	\end{aligned}
	\end{equation}By using  the Newton-Leibniz's  formula, we get that
	\begin{equation}\nonumber
	\begin{aligned}
	g(0)-g(t_{0})=\int_{t_{0}}^{0}g'(\alpha)d\alpha=\int_{t_{0}}^{0}\big(g'(\alpha)-g'(0)\big)d\alpha+\int_{t_{0}}^{0}g'(0)d\alpha.
	\end{aligned}
	\end{equation}Through the $Lipschitz \ condition$ of $g'$, we get that
	\begin{equation}\nonumber
	\begin{aligned}
	g(0)-g(t_{0})\geq\int_{t_{0}}^{0}-c|\alpha-0|d\alpha+\int_{t_{0}}^{0}g'(0)d\alpha=\frac{1}{2c}\big(g'(0)\big)^{2}.
	\end{aligned}
	\end{equation}
{	Then we take a special value of $t_{0}$. Let $t_{0}=-g'(0)/c$, then we get}
	\begin{equation}\label{8unii}
	\begin{aligned}
	&g(0)-g(t_{0})\geq-\int_{t_{0}}^{0}c|\alpha|d\alpha+\int_{t_{0}}^{0}g(0)dt=-\frac{c}{2}(0-t_{0})^{2}+g'(0)(-t_{0})\\&=-\frac{1}{2c}\big(g'(0)\big)^{2}+\frac{1}{c}\big(g'(0)\big)^{2}=\frac{1}{2c}\big(g'(0)\big)^{2}.
	\end{aligned}
	\end{equation}Substituting \eqref{8unii} into \eqref{qwcdfs}, we get   
	\begin{equation}\nonumber
	\begin{aligned}
	g(0)-g^{*}\geq\frac{1}{2c}\big(g'(0)\big)^{2}.
	\end{aligned}
	\end{equation}Due to $g^{*}\geq f^{*}$ and $\big(g'(0)\big)^{2}=\|\nabla f(x)\|^{2}$, it follows that
	\begin{equation}\nonumber
	\begin{aligned}
	\big\|\nabla f(x)\big\|^{2}\le 2c\big(f(x)-f^{*}\big).
	\end{aligned}
	\end{equation}}

	\subsection{Proof of Lemma \ref{lem6}}\label{pf_lemma10}
	First of all, it follows from Lemma \ref{lem_ggg} that
	\begin{equation}\label{98}\begin{aligned}
	&g(\theta_{n+1})-g(\theta_{n})\le\nabla _{\theta_{n}} g(\theta_{n})^{T}(\theta_{n+1}-\theta_{n})+\frac{c\alpha_0^{2}}{2}\frac{\big\|\nabla_{\theta_{n}}g(\theta_{n},\xi_{n})\big\|^{2}}{S_{n}} \\
	&=-\frac{\alpha_0\nabla_{\theta_{n}}g(\theta_{n})^{T}\nabla_{\theta_{n}}g(\theta_{n},\xi_{n})}{\sqrt{S_{n}}}+\frac{c\alpha_0^{2}}{2}\frac{\big\|\nabla_{\theta_{n}}g(\theta_{n},\xi_{n})\big\|^{2}}{S_{n}},
	\end{aligned}\end{equation}
	where 
	\begin{equation}\nonumber\begin{aligned}
	S_{n}=\sum_{k=1}^{n}\big\|\nabla_{\theta_{n}}g(\theta_{n},\xi_{n})\big\|^{2}.\end{aligned}\end{equation}
Note that
	\begin{equation}\label{Kolm}\begin{aligned}
	&\bigg\|\frac{1}{\sqrt{M+1}}\nabla_{\theta_{n}}g(\theta_{n},\xi_{n})-\sqrt{M+1}\nabla_{\theta_{n}}g(\theta_{n})\bigg\|^{2}\\&=\frac{1}{M+1}\big\|\nabla_{\theta_{n}}g(\theta_{n},\xi_{n})\big\|^{2}+(M+1)\big\|\nabla_{\theta_{n}}g(\theta_{n})\big\|^{2}-2\nabla_{\theta_{n}}g(\theta_{n})^{T}\nabla_{\theta_{n}}g(\theta_{n},\xi_{n}),
	\end{aligned}\end{equation}
	where
	$M=M'+2$ and $M'$ is defined in Assumption \ref{ass_g_poi} 3). Substitute \eqref{Kolm} into \eqref{98}, then we   get that %$\{\nabla_{\theta_{n}}g(\theta_{n},\xi_{n})-\nabla_{\theta_{n}}g(\theta_{n})\}$ is a martingale difference column. So we through Lemma , we know that
	\begin{equation}\label{890}\begin{aligned}
	&g(\theta_{n+1})-g(\theta_{n})\\&\le-\frac{\alpha_0}{2}\Bigg(\frac{1}{M+1}\frac{\big\|\nabla_{\theta_{n}}g(\theta_{n},\xi_{n})\big\|^{2}}{\sqrt{S_{n}}}+(M+1)\frac{\big\|\nabla_{\theta_{n}}g(\theta_{n})\big\|^{2}}{\sqrt{S_{n}}}\Bigg)\\&+\frac{\alpha_0}{2}\frac{1}{\sqrt{S_{n}}}\Big\|\frac{1}{\sqrt{M+1}}\nabla_{\theta_{n}}g(\theta_{n},\xi_{n})-\sqrt{M+1}\nabla_{\theta_{n}}g(\theta_{n})\Big\|^{2}+\frac{c\alpha_0^{2}}{2}\frac{\big\|\nabla_{\theta_{n}}g(\theta_{n},\xi_{n})\big\|^{2}}{S_{n}}.
	\end{aligned}\end{equation}
	Due to $S_{n}\geq S_{n-1}$, it follows that
	\begin{equation}\label{1000000}\begin{aligned}
	&\frac{\alpha_0}{2}\frac{1}{\sqrt{S_{n}}}\Big\|\frac{1}{\sqrt{M+1}}\nabla_{\theta_{n}}g(\theta_{n},\xi_{n})-\sqrt{M+1}\nabla_{\theta_{n}}g(\theta_{n})\Big\|^{2}\\&\le\frac{\alpha_0}{2}\frac{1}{\sqrt{S_{n-1}}}\Big\|\frac{1}{\sqrt{M+1}}\nabla_{\theta_{n}}g(\theta_{n},\xi_{n})-\sqrt{M+1}\nabla_{\theta_{n}}g(\theta_{n})\Big\|^{2}.\end{aligned}\end{equation}
	Substitute \eqref{1000000} into \eqref{890}, then we have
	\begin{equation}\label{1000001}\begin{aligned}
	&g(\theta_{n+1})-g(\theta_{n})\\&\le-\frac{\alpha_0}{2}\Bigg(\frac{1}{M+1}\frac{\big\|\nabla_{\theta_{n}}g(\theta_{n},\xi_{n})\big\|^{2}}{\sqrt{S_{n}}}+(M+1)\frac{\big\|\nabla_{\theta_{n}}g(\theta_{n})\big\|^{2}}{\sqrt{S_{n}}}\Bigg)\\&+\frac{\alpha_0}{2}\frac{1}{\sqrt{S_{n-1}}}\Big\|\frac{1}{\sqrt{M+1}}\nabla_{\theta_{n}}g(\theta_{n},\xi_{n})-\sqrt{M+1}\nabla_{\theta_{n}}g(\theta_{n})\Big\|^{2}+\frac{c\alpha_0^{2}}{2}\frac{\big\|\nabla_{\theta_{n}}g(\theta_{n},\xi_{n})\big\|^{2}}{S_{n}}.
	\end{aligned}\end{equation}
	Notice that
	\begin{equation}\label{01000001}\begin{aligned}
	&\frac{\alpha_0}{2}\frac{1}{\sqrt{S_{n-1}}}\Big\|\frac{1}{\sqrt{M+1}}\nabla_{\theta_{n}}g(\theta_{n},\xi_{n})-\sqrt{M+1}\nabla_{\theta_{n}}g(\theta_{n})\Big\|^{2}\\&=\frac{\alpha_0}{2}\frac{1}{\sqrt{S_{n-1}}}\bigg(\frac{1}{M+1}\big\|\nabla_{\theta_{n}}g(\theta_{n},\xi_{n})\big\|^{2}+(M+1)\big\|\nabla_{\theta_{n}}g(\theta_{n})\big\|^{2}-2\nabla_{\theta_{n}}g(\theta_{n},\xi_{n})^{T}\nabla_{\theta_{n}}g(\theta_{n})\bigg)\\&=\frac{\alpha_0}{2}\frac{1}{\sqrt{S_{n-1}}}\bigg(\frac{1}{M+1}\big\|\nabla_{\theta_{n}}g(\theta_{n},\xi_{n})\big\|^{2}+(M+1)\big\|\nabla_{\theta_{n}}g(\theta_{n})\big\|^{2}-2\big\|\nabla_{\theta_{n}}g(\theta_{n})\big\|^{2}\bigg)\\&+\frac{\alpha_0}{\sqrt{S_{n-1}}}\nabla_{\theta_{n}}g(\theta_{n})^{T}\big(\nabla_{\theta_{n}}g(\theta_{n})-\nabla_{\theta_{n}}g(\theta_{n},\xi_{n})\big).
	\end{aligned}\end{equation}
	Substitute \eqref{01000001} into \eqref{1000001}, and divide both sides of the inequality by $S_{n}^{\epsilon}$ $(\epsilon<\frac{1}{2})$, then we get 
	\begin{equation}\nonumber\begin{aligned}
	&\frac{g(\theta_{n+1})}{S_{n}^{\epsilon}}-\frac{g(\theta_{n})}{S_{n}^{\epsilon}}\\&\le-\frac{\alpha_0}{2}\Bigg(\frac{1}{M+1}\frac{\big\|\nabla_{\theta_{n}}g(\theta_{n},\xi_{n})\big\|^{2}}{S_{n}^{\frac{1}{2}+\epsilon}}+(M+1)\frac{\big\|\nabla_{\theta_{n}}g(\theta_{n})\big\|^{2}}{S_{n}^{\frac{1}{2}+\epsilon}}\Bigg)\\&+\frac{\alpha_0}{2}\frac{1}{S_{n-1}^{\frac{1}{2}+\epsilon}}\bigg(\frac{1}{M+1}\big\|\nabla_{\theta_{n}}g(\theta_{n},\xi_{n})\big\|^{2}+(M+1)\big\|\nabla_{\theta_{n}}g(\theta_{n})\big\|^{2}-2\big\|\nabla_{\theta_{n}}g(\theta_{n})\big\|^{2}\bigg)\\&+\frac{c\alpha_0^{2}}{2}\frac{\big\|\nabla_{\theta_{n}}g(\theta_{n},\xi_{n})\big\|^{2}}{S_{n}^{1+\epsilon}}+\frac{\alpha_0}{S_{n-1}^{\frac{1}{2}+\epsilon}}\nabla_{\theta_{n}}g(\theta_{n})^{T}\big(\nabla_{\theta_{n}}g(\theta_{n})-\nabla_{\theta_{n}}g(\theta_{n},\xi_{n})\big).
	\end{aligned}\end{equation}
	%Notice that $\xi_{n}$ and $\{\theta_{1}, \xi_{1}, \xi_{2},...,\xi_{n-1}\}$is independence. We can conclude that
	%\begin{equation}\label{1000003}\begin{aligned}
	%&\frac{\alpha_0}{2}\Expect\Bigg(\frac{1}{S_{n-1}^{\frac{1}{2}+\epsilon}}\Big\|\frac{1}{\sqrt{M+1}}\nabla_{\theta_{n}}g(\theta_{n},\xi_{n})-\sqrt{M+1}\nabla_{\theta_{n}}g(\theta_{n})\Big\|^{2}\Bigg)\\&=\frac{\alpha_0}{2}\Expect\Bigg(\Expect\Bigg(\frac{1}{S_{n-1}^{\frac{1}{2}+\epsilon}}\Big\|\frac{1}{\sqrt{M+1}}\nabla_{\theta_{n}}g(\theta_{n},\xi_{n})-\sqrt{M+1}\nabla_{\theta_{n}}g(\theta_{n})\Big\|^{2}\Bigg|\mathscr{F}_{n-1}\Bigg)\\&=\frac{\alpha_0}{2}\Expect\Bigg(\frac{1}{S_{n-1}^{\frac{1}{2}+\epsilon}}\Expect\Bigg(\Big\|\frac{1}{\sqrt{M+1}}\nabla_{\theta_{n}}g(\theta_{n},\xi_{n})-\sqrt{M+1}\nabla_{\theta_{n}}g(\theta_{n})\Big\|^{2}\Bigg|\mathscr{F}_{n-1}\Bigg)\\&=\frac{\alpha_0}{2}\Expect\Bigg(\frac{1}{M+1}\frac{\big\|\nabla_{\theta_{n}}g(\theta_{n},\xi_{n})\big\|^{2}}{S_{n-1}^{\frac{1}{2}+\epsilon}}+(M-1)\frac{\big\|\nabla_{\theta_{n}}g(\theta_{n})\big\|^{2}}{S_{n-1}^{\frac{1}{2}+\epsilon}}\Bigg).
	%\end{aligned}\end{equation}
	Notice that $\frac{g(\theta_{n+1})}{S_{n}^{\epsilon}}>\frac{g(\theta_{n+1})}{S_{n+1}^{\epsilon}}$, then we obtain
	\begin{equation}\label{1000002}\begin{aligned}
	&\frac{g(\theta_{n+1})}{S_{n+1}^{\epsilon}}-\frac{g(\theta_{n})}{S_{n}^{\epsilon}}\\&\le-\frac{\alpha_0}{2}\Bigg(\frac{1}{M+1}\frac{\big\|\nabla_{\theta_{n}}g(\theta_{n},\xi_{n})\big\|^{2}}{S_{n}^{\frac{1}{2}+\epsilon}}+(M+1)\frac{\big\|\nabla_{\theta_{n}}g(\theta_{n})\big\|^{2}}{S_{n}^{\frac{1}{2}+\epsilon}}\Bigg)\\&+\frac{\alpha_0}{2}\frac{1}{S_{n-1}^{\frac{1}{2}+\epsilon}}\bigg(\frac{1}{M+1}\big\|\nabla_{\theta_{n}}g(\theta_{n},\xi_{n})\big\|^{2}+(M+1)\big\|\nabla_{\theta_{n}}g(\theta_{n})\big\|^{2}-2\big\|\nabla_{\theta_{n}}g(\theta_{n})\big\|^{2}\bigg)\\&+\frac{c\alpha_0^{2}}{2}\frac{\big\|\nabla_{\theta_{n}}g(\theta_{n},\xi_{n})\big\|^{2}}{S_{n}^{1+\epsilon}}+\frac{\alpha_0}{S_{n-1}^{\frac{1}{2}+\epsilon}}\nabla_{\theta_{n}}g(\theta_{n})^{T}\big(\nabla_{\theta_{n}}g(\theta_{n})-\nabla_{\theta_{n}}g(\theta_{n},\xi_{n})\big).
	\end{aligned}
	\end{equation}
	
	Rearrange the above inequality, then it holds that
	\begin{equation}\label{1000004}\begin{aligned}
	&\frac{g(\theta_{n+1})}{S_{n+1}^{\epsilon}}-\frac{g(\theta_{n})}{S_{n}^{\epsilon}}\\
	&\le-\frac{\alpha_0}{2}(M+1)\frac{\big\|\nabla_{\theta_{n}}g(\theta_{n})\big\|^{2}}{S_{n}^{\frac{1}{2}+\epsilon}}\\&+\frac{\alpha_0}{2}\Bigg(\frac{1}{M+1}\frac{\big\|\nabla_{\theta_{n}}g(\theta_{n},\xi_{n})\big\|^{2}}{S_{n-1}^{\frac{1}{2}+\epsilon}}+(M-1)\frac{\big\|\nabla_{\theta_{n}}g(\theta_{n})\big\|^{2}}{S_{n-1}^{\frac{1}{2}+\epsilon}}\Bigg)+\frac{c\alpha_0^{2}}{2}\frac{\big\|\nabla_{\theta_{n}}g(\theta_{n},\xi_{n})\big\|^{2}}{S_{n}^{1+\epsilon}}+X_{n}^{(\epsilon)}\\
	&=\frac{\alpha_0}{2}(M+1)\Bigg(\frac{\big\|\nabla_{\theta_{n-1}}g(\theta_{n-1})\big\|^{2}}{S_{n-1}^{\frac{1}{2}+\epsilon}}-\frac{\big\|\nabla_{\theta_{n}}g(\theta_{n})\big\|^{2}}{S_{n}^{\frac{1}{2}+\epsilon}}\Bigg)\\&+\frac{\alpha_0}{2}\Bigg(\frac{1}{M+1}\frac{\big\|\nabla_{\theta_{n}}g(\theta_{n},\xi_{n})\big\|^{2}}{S_{n-1}^{\frac{1}{2}+\epsilon}}+\frac{(M-1)\big\|\nabla_{\theta_{n}}g(\theta_{n})\big\|^{2}}{S_{n-1}^{\frac{1}{2}+\epsilon}}-\frac{(M+1)\big\|\nabla_{\theta_{n-1}}g(\theta_{n-1})\big\|^{2}}{S_{n-1}^{\frac{1}{2}+\epsilon}}\Bigg)\\&+\frac{c\alpha_0^{2}}{2}\frac{\big\|\nabla_{\theta_{n}}g(\theta_{n},\xi_{n})\big\|^{2}}{S_{n}^{1+\epsilon}}+X_{n}^{(\epsilon)}.
	\end{aligned}\end{equation}
	$X_{n}^{(\epsilon)}$ is defined as follow  
	\begin{equation}\nonumber\begin{aligned}
	X_{n}^{(\epsilon)}=\frac{\alpha_0}{S_{n-1}^{\frac{1}{2}+\epsilon}}\nabla_{\theta_{n}}g(\theta_{n})^{T}\big(\nabla_{\theta_{n}}g(\theta_{n})-\nabla_{\theta_{n}}g(\theta_{n},\xi_{n})\big).
	\end{aligned}\end{equation}
	Due to $\|\nabla_{\theta_{n}}g(\theta_{n})\|^{2}>a$, we have \begin{equation}\label{1000008}\begin{aligned}
	&\Expect\Big(\big\|\nabla_{\theta_{n}}g(\theta_{n},\xi_{n})\big\|^{2}\Big|\mathscr{F}_{n-1}\Big)\le M\big\|\nabla_{\theta_{n}}g(\theta_{n})\big\|^{2}+a\\&<(M+1)\big\|\nabla_{\theta_{n}}g(\theta_{n})\big\|^{2},
	\end{aligned}\end{equation}%we can find that
	Moreover, using the Taylor formula, we  obtain
	\begin{equation}\nonumber\begin{aligned}
	&\big\|\nabla_{\theta_{n}}g(\theta_{n})\big\|^{2}=\big\|\nabla_{\theta_{n-1}}g(\theta_{n-1})+\big(\nabla_{\theta_{n}}g(\theta_{n})-\nabla_{\theta_{n-1}}g(\theta_{n-1})\big)\big\|^{2}\\&=\big\|\nabla_{\theta_{n-1}}g(\theta_{n-1})\big\|^{2}+2\nabla_{\theta_{n-1}}g(\theta_{n-1})^{T}\big(\nabla_{\theta_{n}}g(\theta_{n})-\nabla_{\theta_{n-1}}g(\theta_{n-1})\big)\\&+\big\|\nabla_{\theta_{n}}g(\theta_{n})-\nabla_{\theta_{n-1}}g(\theta_{n-1})\big\|^{2}\le\big\|\nabla_{\theta_{n-1}}g(\theta_{n-1})\big\|^{2}\\&+2\big\|\nabla_{\theta_{n-1}}g(\theta_{n-1})\big\|\big\|\nabla_{\theta_{n}}g(\theta_{n})-\nabla_{\theta_{n-1}}g(\theta_{n-1})\big\|+\big\|\nabla_{\theta_{n}}g(\theta_{n})-\nabla_{\theta_{n-1}}g(\theta_{n-1})\big\|^{2}.
	\end{aligned}\end{equation}
	Under Assumption \ref{ass_g_poi} 3), we get that
	\begin{equation}\label{05}\begin{aligned}
	&\big\|\nabla_{\theta_{n}}g(\theta_{n})\big\|^{2}\le\big\|\nabla_{\theta_{n-1}}g(\theta_{n-1})\big\|^{2}\\&+2\big\|\nabla_{\theta_{n-1}}g(\theta_{n-1})\big\|\big\|\nabla_{\theta_{n}}g(\theta_{n})-\nabla_{\theta_{n-1}}g(\theta_{n-1})\big\|+\big\|\nabla_{\theta_{n}}g(\theta_{n})-\nabla_{\theta_{n-1}}g(\theta_{n-1})\big\|^{2}\\&\le\big\|\nabla_{\theta_{n-1}}g(\theta_{n-1})\big\|^{2}+\frac{2\alpha_0 c}{\sqrt{S_{n-1}}}\big\|\nabla_{\theta_{n-1}}g(\theta_{n-1})\big\|\big\|\nabla_{\theta_{n-1}}g(\theta_{n-1},\xi_{n-1})\big\|\\&+c^{2}\alpha_0^{2}\frac{\big\|\nabla_{\theta_{n-1}}g(\theta_{n-1},\xi_{n-1})\big\|^{2}}{S_{n-1}}.
	\end{aligned}\end{equation}
	
From   inequality $2a^{T}b\le \lambda \|a\|^{2}+\frac{1}{\lambda}\|b\|^{2}\ \ (\lambda>0)$, it follows that \begin{equation}\label{1000006}\begin{aligned}
	&(M-1)\big\|\nabla_{\theta_{n}}g(\theta_{n})\big\|^{2}+\big\|\nabla_{\theta_{n}}g(\theta_{n})\big\|^{2}\\&\le(M+1)\big\|\nabla_{\theta_{n-1}}g(\theta_{n-1})\big\|^{2}-\frac{M-1}{4M-3}\big\|\nabla_{\theta_{n}}g(\theta_{n})\big\|^{2}+4M^{2}\alpha_0^{2}c^{2}\frac{\big\|\nabla_{\theta_{n-1}}g(\theta_{n-1},\xi_{n-1})\big\|^{2}}{S_{n-1}}.
	\end{aligned}\end{equation}
	By substituting \eqref{1000008} into \eqref{1000006}, we have
	\begin{equation}\label{1000009}\begin{aligned}
	&(M-1)\big\|\nabla_{\theta_{n}}g(\theta_{n})\big\|^{2}+\frac{1}{M+1}\Expect\Big(\big\|\nabla_{\theta_{n}}g(\theta_{n},\xi_{n})\big\|^{2}\Big|\mathscr{F}_{n-1}\Big)\\&\le(M+1)\big\|\nabla_{\theta_{n-1}}g(\theta_{n-1})\big\|^{2}-\frac{M-1}{4M-3}\big\|\nabla_{\theta_{n}}g(\theta_{n})\big\|^{2}+4M^{2}\alpha_0^{2}c^{2}\frac{\big\|\nabla_{\theta_{n-1}}g(\theta_{n-1},\xi_{n-1})\big\|^{2}}{S_{n-1}},
	\end{aligned}\end{equation}
	Divide  both sides of \eqref{1000009} by $S_{n-1}^{\frac{1}{2}+\epsilon}$, and notice $\frac{M-1}{4M-3}>\frac{1}{5}$ from $M>2$, then it holds that
	\begin{equation}\label{090909}\begin{aligned}
	&\frac{1}{M+1}\frac{\big\|\nabla_{\theta_{n}}g(\theta_{n},\xi_{n})\big\|^{2}}{S_{n-1}^{\frac{1}{2}+\epsilon}}+\frac{\big\|\nabla_{\theta_{n}}g(\theta_{n})\big\|^{2}}{S_{n-1}^{\frac{1}{2}+\epsilon}}\\&\le(M+1)\frac{\big\|\nabla_{\theta_{n-1}}g(\theta_{n-1})\big\|^{2}}{S_{n-1}^{\frac{1}{2}+\epsilon}}-\frac{1}{5}\big\|\nabla_{\theta_{n}}g(\theta_{n})\big\|^{2}+4M^{2}\alpha_0^{2}c^{2}\frac{\big\|\nabla_{\theta_{n-1}}g(\theta_{n-1},\xi_{n-1})\big\|^{2}}{S_{n-1}}\\&+\frac{2}{\alpha_0}Y_{n}^{(\epsilon)},
	\end{aligned}\end{equation}
	where 
	\begin{equation}\nonumber\begin{aligned}
	Y_{n}^{(\epsilon)}=\frac{\alpha_0}{2}\Bigg(\frac{1}{M+1}\frac{\Expect\Big(\big\|\nabla_{\theta_{n}}g(\theta_{n},\xi_{n})\big\|^{2}\Big|\mathscr{F}_{n-1}\Big)}{S_{n-1}^{\frac{1}{2}+\epsilon}}-\frac{1}{M+1}\frac{\big\|\nabla_{\theta_{n}}g(\theta_{n},\xi_{n})\big\|^{2}}{S_{n-1}^{\frac{1}{2}+\epsilon}}\Bigg).
	\end{aligned}\end{equation}
	Making some  simple transformations on \eqref{090909} leads to 
	\begin{equation}\label{1000011}\begin{aligned}
	&\frac{\alpha_0}{2}\Bigg(\frac{1}{M+1}\frac{\big\|\nabla_{\theta_{n}}g(\theta_{n},\xi_{n})\big\|^{2}}{S_{n-1}^{\frac{1}{2}+\epsilon}}+\frac{(M-1)\big\|\nabla_{\theta_{n}}g(\theta_{n})\big\|^{2}}{S_{n-1}^{\frac{1}{2}+\epsilon}}-\frac{(M+1)\big\|\nabla_{\theta_{n-1}}g(\theta_{n-1})\big\|^{2}}{S_{n-1}^{\frac{1}{2}+\epsilon}}\Bigg)\\&\le-\frac{\alpha_0}{10}\frac{\big\|\nabla_{\theta_{n}}g(\theta_{n})\big\|^{2}}{S_{n-1}^{\frac{1}{2}+\epsilon}}+4M^{2}\alpha_0^{3}c^{2}\frac{\big\|\nabla_{\theta_{n-1}}g(\theta_{n-1},\xi_{n-1})\big\|^{2}}{S_{n-1}^{\frac{3}{2}+\epsilon}}+Y_{n}^{(\epsilon)}.
	\end{aligned}\end{equation}
	Substitute \eqref{1000011} into \eqref{1000004}, then we get
	\begin{equation}\label{1000013}\begin{aligned}
	&\frac{g(\theta_{n+1})}{S_{n+1}^{\epsilon}}-\frac{g(\theta_{n})}{S_{n}^{\epsilon}}\\&\le\frac{\alpha_0}{2}(M+1)\Bigg(\frac{\big\|\nabla_{\theta_{n-1}}g(\theta_{n-1})\big\|^{2}}{S_{n-1}^{\frac{1}{2}+\epsilon}}-\frac{\big\|\nabla_{\theta_{n}}g(\theta_{n})\big\|^{2}}{S_{n}^{\frac{1}{2}+\epsilon}}\Bigg)\\&-\frac{\alpha_0}{10}\frac{\big\|\nabla_{\theta_{n}}g(\theta_{n})\big\|^{2}}{S_{n-1}^{\frac{1}{2}+\epsilon}}+4M^{2}\alpha_0^{3}c^{2}\frac{\big\|\nabla_{\theta_{n-1}}g(\theta_{n-1},\xi_{n-1})\big\|^{2}}{S_{n-1}^{\frac{3}{2}+\epsilon}}+\frac{c\alpha_0^{2}}{2}\frac{\big\|\nabla_{\theta_{n}}g(\theta_{n},\xi_{n})\big\|^{2}}{S_{n}^{1+\epsilon}}\\&+X_{n}^{(\epsilon)}+Y_{n}^{(\epsilon)}.
	\end{aligned}\end{equation}
	%\begin{equation}\nonumber\begin{aligned}
	%G_{n-1}=\frac{\alpha_0}{8}\frac{1}{\sqrt{S_{n-1}}}-\frac{5}{4}\alpha_0^{3}c^{2}(M+1)\frac{1}{S_{n-1}^{\frac{3}{2}}},\end{aligned}\end{equation}
	%and
	%\begin{equation}\nonumber\begin{aligned}
	%H_{n-1}=\frac{\alpha_0}{8}\frac{1}{\sqrt{S_{n-1}}}-\frac{c\alpha_0^{2}}{2}\frac{1}{S_{n}}.\end{aligned}\end{equation}
	It follows that
	\begin{equation}\label{1000014}\begin{aligned}
	&-\frac{\alpha_0}{10}\frac{\big\|\nabla_{\theta_{n}}g(\theta_{n})\big\|^{2}}{S_{n-1}^{\frac{1}{2}+\epsilon}}=-\frac{\alpha_0}{20}\frac{\big\|\nabla_{\theta_{n}}g(\theta_{n})\big\|^{2}}{S_{n-1}^{\frac{1}{2}+\epsilon}}-\frac{\alpha_0}{20}\frac{\big\|\nabla_{\theta_{n}}g(\theta_{n})\big\|^{2}}{S_{n-1}^{\frac{1}{2}+\epsilon}}\\&=-\frac{\alpha_0}{20}\frac{\big\|\nabla_{\theta_{n}}g(\theta_{n})\big\|^{2}}{S_{n-1}^{\frac{1}{2}+\epsilon}}+\frac{\alpha_0}{20}\Bigg(\frac{\big\|\nabla_{\theta_{n-1}}g(\theta_{n-1})\big\|^{2}}{S_{n-2}^{\frac{1}{2}+\epsilon}}-\frac{\big\|\nabla_{\theta_{n}}g(\theta_{n})\big\|^{2}}{S_{n-1}^{\frac{1}{2}+\epsilon}}\Bigg)\\&-\frac{\alpha_0}{20}\frac{\big\|\nabla_{\theta_{n-1}}g(\theta_{n-1})\big\|^{2}}{S_{n-2}^{\frac{1}{2}+\epsilon}}\\&\le-\frac{\alpha_0}{20}\frac{\big\|\nabla_{\theta_{n}}g(\theta_{n})\big\|^{2}}{S_{n-1}^{\frac{1}{2}+\epsilon}}+\frac{\alpha_0}{20}\Bigg(\frac{\big\|\nabla_{\theta_{n-1}}g(\theta_{n-1})\big\|^{2}}{S_{n-2}^{\frac{1}{2}+\epsilon}}-\frac{\big\|\nabla_{\theta_{n}}g(\theta_{n})\big\|^{2}}{S_{n-1}^{\frac{1}{2}+\epsilon}}\Bigg).
	\end{aligned}\end{equation}
	Substituting \eqref{1000014} into \eqref{1000013} yields
	\begin{equation}\label{1000015}\begin{aligned}
	&\frac{g(\theta_{n+1})}{S_{n+1}^{\epsilon}}-\frac{g(\theta_{n})}{S_{n}^{\epsilon}}\\&\le\frac{\alpha_0}{2}(M+1)\Bigg(\frac{\big\|\nabla_{\theta_{n-1}}g(\theta_{n-1})\big\|^{2}}{S_{n-1}^{\frac{1}{2}+\epsilon}}-\frac{\big\|\nabla_{\theta_{n}}g(\theta_{n})\big\|^{2}}{S_{n}^{\frac{1}{2}+\epsilon}}\Bigg)\\&-\frac{\alpha_0}{20}\frac{\big\|\nabla_{\theta_{n}}g(\theta_{n})\big\|^{2}}{S_{n-1}^{\frac{1}{2}+\epsilon}}+\frac{\alpha_0}{20}\Bigg(\frac{\big\|\nabla_{\theta_{n-1}}g(\theta_{n-1})\big\|^{2}}{S_{n-2}^{\frac{1}{2}+\epsilon}}-\frac{\big\|\nabla_{\theta_{n}}g(\theta_{n})\big\|^{2}}{S_{n-1}^{\frac{1}{2}+\epsilon}}\Bigg)\\&+4M^{2}\alpha_0^{3}c^{2}\frac{\big\|\nabla_{\theta_{n-1}}g(\theta_{n-1},\xi_{n-1})\big\|^{2}}{S_{n-1}^{\frac{3}{2}+\epsilon}}+\frac{c\alpha_0^{2}}{2}\frac{\big\|\nabla_{\theta_{n}}g(\theta_{n},\xi_{n})\big\|^{2}}{S_{n}^{1+\epsilon}}+X_{n}^{(\epsilon)}+Y_{n}^{(\epsilon)}.
	\end{aligned}\end{equation}
	
	\subsection{Proof of Lemma \ref{lem7}}
	First of all, dividing      both sides of \eqref{98} by $S_{n}^{\epsilon}$  yields
	\begin{equation}\begin{aligned}
	&\frac{g(\theta_{n+1})}{S_{n}^{\epsilon}}-\frac{g(\theta_{n})}{S_{n}^{\epsilon}}\le-\frac{\alpha_0\nabla_{\theta_{n}}g(\theta_{n})^{T}\nabla_{\theta_{n}}g(\theta_{n},\xi_{n})}{S_{n}^{\frac{1}{2}+\epsilon}}+\frac{c\alpha_0^{2}}{2}\frac{\big\|\nabla_{\theta_{n}}g(\theta_{n},\xi_{n})\big\|^{2}}{S_{n}^{1+\epsilon}}.
	\end{aligned}\end{equation}
Due to $S_{n+1}\geq S_{n}$, it holds that
	\begin{equation}\label{01}\begin{aligned}
	&\frac{g(\theta_{n+1})}{S_{n+1}^{\epsilon}}-\frac{g(\theta_{n})}{S_{n}^{\epsilon}}\le-\frac{\alpha_0\nabla_{\theta_{n}}g(\theta_{n})^{T}\nabla_{\theta_{n}}g(\theta_{n},\xi_{n})}{S_{n}^{\frac{1}{2}+\epsilon}}+\frac{c\alpha_0^{2}}{2}\frac{\big\|\nabla_{\theta_{n}}g(\theta_{n},\xi_{n})\big\|^{2}}{S_{n}^{1+\epsilon}}.
	\end{aligned}\end{equation}Then we make some transformations to obtain that
	\begin{equation}\label{03}\begin{aligned}
	&-\frac{\alpha_0\nabla_{\theta_{n}}g(\theta_{n})^{T}\nabla_{\theta_{n}}g(\theta_{n},\xi_{n})}{S_{n}^{\frac{1}{2}+\epsilon}}\\&=-\frac{\alpha_0\nabla_{\theta_{n}}g(\theta_{n})^{T}\nabla_{\theta_{n}}g(\theta_{n},\xi_{n})}{S_{n-1}^{\frac{1}{2}+\epsilon}}+\alpha_0\nabla_{\theta_{n}}g(\theta_{n})^{T}\nabla_{\theta_{n}}g(\theta_{n},\xi_{n})\Bigg(\frac{1}{S_{n-1}^{\frac{1}{2}+\epsilon}}-\frac{1}{S_{n}^{\frac{1}{2}+\epsilon}}\Bigg)\\&\le-\frac{\alpha_0\nabla_{\theta_{n}}g(\theta_{n})^{T}\nabla_{\theta_{n}}g(\theta_{n},\xi_{n})}{S_{n-1}^{\frac{1}{2}+\epsilon}}+\alpha_0\bigg(\frac{(M+1)a}{2}+\frac{1}{2(M+1)a}\big\|\nabla_{\theta_{n}}g(\theta_{n})\big\|^{2}\big\|\nabla_{\theta_{n}}g(\theta_{n},\xi_{n})\big\|^{2}\bigg)\\&\Bigg(\frac{1}{S_{n-1}^{\frac{1}{2}+\epsilon}}-\frac{1}{S_{n}^{\frac{1}{2}+\epsilon}}\Bigg)\\
	&\le -\frac{\alpha_0\big\|\nabla_{\theta_{n}}g(\theta_{n})\big\|^{2}}{S_{n-1}^{\frac{1}{2}+\epsilon}}+\frac{(M+1)\alpha_0 a}{2}\Bigg(\frac{1}{S_{n-1}^{\frac{1}{2}+\epsilon}}-\frac{1}{S_{n}^{\frac{1}{2}+\epsilon}}\Bigg)\\&+\bigg(\frac{\alpha_0 }{2a(M+1)}\big\|\nabla_{\theta_{n}}g(\theta_{n})\big\|^{2}\Expect\Big(\big\|\nabla_{\theta_{n}}g(\theta_{n},\xi_{n})\big\|^{2}\Big|\mathscr{F}_{n-1}\Big)\bigg)\frac{1}{S_{n-1}^{\frac{1}{2}+\epsilon}}+A_{n}^{(\epsilon)}+B_{n}^{(\epsilon)}.
	\end{aligned}\end{equation}
where
% 	$A_{n}^{(\epsilon)}$ and $B_{n}^{(\epsilon)}$ is defined as follow
	\begin{equation}\label{def_A_B}
\begin{aligned}
	&A_{n}^{(\epsilon)}=\frac{\alpha_0}{S_{n-1}^{\frac{1}{2}+\epsilon}}\Big(\big\|\nabla_{\theta_{n}}g(\theta_{n})\big\|^{2}-\nabla_{\theta_{n}}g(\theta_{n})^{T}\nabla_{\theta_{n}}g(\theta_{n},\xi_{n})\Big)
	\\&B_{n}^{(\epsilon)}=\frac{\alpha_0 \big\|\nabla_{\theta_{n}}g(\theta_{n})\big\|^{2}}{2a(M+1)S_{n-1}^{\frac{1}{2}+\epsilon}}\bigg(\big\|\nabla_{\theta_{n}}g(\theta_{n},\xi_{n})\big\|^{2}-\Expect\Big(\big\|\nabla_{\theta_{n}}g(\theta_{n},\xi_{n})\big\|^{2}\Big|\mathscr{F}_{n-1}\Big)\bigg).
	\end{aligned}\end{equation}
	Due to $\big\|\nabla_{\theta_{n}}g(\theta_{n})\big\|^{2}\le a$, we get
	\begin{equation}\label{04}\begin{aligned}
	\Expect\Big(\big\|\nabla_{\theta_{n}}g(\theta_{n},\xi_{n})\big\|^{2}\Big|\mathscr{F}_{n-1}\Big)\le M\big\|\nabla_{\theta_{n}}g(\theta_{n})\big\|^{2}+a\le (M+1)a.
	\end{aligned}\end{equation}
	Substitute it into \eqref{03}, then we get
	\begin{equation}\label{02}\begin{aligned}
	&\frac{g(\theta_{n+1})}{S_{n+1}^{\epsilon}}-\frac{g(\theta_{n})}{S_{n}^{\epsilon}}\le-\frac{\alpha_0\big\|\nabla_{\theta_{n}}g(\theta_{n})\big\|^{2}}{S_{n-1}^{\frac{1}{2}+\epsilon}}+\frac{\alpha_0 a(M+1)}{2}\Bigg(\frac{1}{S_{n-1}^{\frac{1}{2}+\epsilon}}-\frac{1}{S_{n}^{\frac{1}{2}+\epsilon}}\Bigg)\\&+\frac{\alpha_0\big\|\nabla_{\theta_{n}}g(\theta_{n})\big\|^{2}}{2S_{n-1}^{\frac{1}{2}+\epsilon}}+\frac{c\alpha_0^{2}}{2}\frac{\big\|\nabla_{\theta_{n}}g(\theta_{n},\xi_{n})\big\|^{2}}{S_{n}^{1+\epsilon}}+A_{n}^{(\epsilon)}+B_{n}^{(\epsilon)}\\&=-\frac{\alpha_0\big\|\nabla_{\theta_{n}}g(\theta_{n})\big\|^{2}}{2S_{n-1}^{\frac{1}{2}+\epsilon}}+\frac{\alpha_0 a(M+1)}{2}\Bigg(\frac{1}{S_{n-1}^{\frac{1}{2}+\epsilon}}-\frac{1}{S_{n}^{\frac{1}{2}+\epsilon}}\Bigg)\\&+\frac{c\alpha_0^{2}}{2}\frac{\big\|\nabla_{\theta_{n}}g(\theta_{n},\xi_{n})\big\|^{2}}{S_{n}^{1+\epsilon}}+A_{n}^{(\epsilon)}+B_{n}^{(\epsilon)}.
	\end{aligned}\end{equation}
	We make some transformations on $-\frac{\alpha_0\big\|\nabla_{\theta_{n}}g(\theta_{n})\big\|^{2}}{2S_{n-1}^{\frac{1}{2}+\epsilon}}$ to obtain that
	\begin{equation}\label{04}\begin{aligned}
	&-\frac{\alpha_0\big\|\nabla_{\theta_{n}}g(\theta_{n})\big\|^{2}}{2S_{n-1}^{\frac{1}{2}+\epsilon}}\le-\frac{\alpha_0\big\|\nabla_{\theta_{n}}g(\theta_{n})\big\|^{2}}{20S_{n-1}^{\frac{1}{2}+\epsilon}}-\frac{\alpha_0\big\|\nabla_{\theta_{n}}g(\theta_{n})\big\|^{2}}{20S_{n-1}^{\frac{1}{2}+\epsilon}}\\&=-\frac{\alpha_0\big\|\nabla_{\theta_{n}}g(\theta_{n})\big\|^{2}}{20S_{n-1}^{\frac{1}{2}+\epsilon}}-\frac{\alpha_0}{20}\frac{\big\|\nabla_{\theta_{n-1}}g(\theta_{n-1})\big\|^{2}}{S_{n-2}^{\frac{1}{2}+\epsilon}}+\frac{\alpha_0}{20}\Bigg(\frac{\big\|\nabla_{\theta_{n-1}}g(\theta_{n-1})\big\|^{2}}{S_{n-2}^{\frac{1}{2}+\epsilon}}-\frac{\big\|\nabla_{\theta_{n}}g(\theta_{n})\big\|^{2}}{S_{n-1}^{\frac{1}{2}+\epsilon}}\Bigg).
	\end{aligned}\end{equation}
	Then we use inequality  $2a^{T}b\le \lambda \|a\|^{2}+\frac{1}{\lambda}\|b\|^{2}\ \ (\lambda>0)$ on \eqref{05} to get
	\begin{equation}\label{06}\begin{aligned}
	&\big\|\nabla_{\theta_{n}}g(\theta_{n})\big\|^{2}-\big\|\nabla_{\theta_{n-1}}g(\theta_{n-1})\big\|^{2}\le\frac{\big\|\nabla_{\theta_{n-1}}g(\theta_{n-1})\big\|^{2}}{10(M+1)}\\&+\frac{10\alpha_0^{2}c^{2}(M+1)}{S_{n-1}}\big\|\nabla_{\theta_{n-1}}g(\theta_{n-1},\xi_{n-1})\big\|^{2}+\frac{\alpha_0^{2}c^{2}}{S_{n-1}}\big\|\nabla_{\theta_{n-1}}g(\theta_{n-1},\xi_{n-1})\big\|^{2}.
	\end{aligned}\end{equation}
	Divide both sides of \eqref{06} by $S_{n-1}^{\frac{1}{2}+\epsilon}$ and notice $S_{n-2}\le S_{n-1}\le S_{n}$, then we have
	\begin{equation}\label{07}\begin{aligned}
	&\frac{\big\|\nabla_{\theta_{n}}g(\theta_{n})\big\|^{2}}{S_{n}^{\frac{1}{2}+\epsilon}}-\frac{\big\|\nabla_{\theta_{n-1}}g(\theta_{n-1})\big\|^{2}}{S_{n-1}^{\frac{1}{2}+\epsilon}}\le\frac{1}{M+1}\frac{\big\|\nabla_{\theta_{n-1}}g(\theta_{n-1})\big\|^{2}}{10S_{n-2}^{\frac{1}{2}+\epsilon}}\\&+\frac{10\alpha_0^{2}c^{2}(M+1)}{S_{n-1}^{1+\epsilon}}\big\|\nabla_{\theta_{n-1}}g(\theta_{n-1},\xi_{n-1})\big\|^{2}+\frac{\alpha_0^{2}c^{2}}{S_{n-1}^{\frac{3}{2}+\epsilon}}\big\|\nabla_{\theta_{n-1}}g(\theta_{n-1},\xi_{n-1})\big\|^{2}.
	\end{aligned}\end{equation}
	%\begin{equation}\label{08}\begin{aligned}
	%&\frac{\big\|\nabla_{\theta_{n}}g(\theta_{n})\big\|^{2}}{S_{n-1}^{\frac{1}{2}+\epsilon}}-\frac{\big\|\nabla_{\theta_{n-1}}g(\theta_{n-1})\big\|^{2}}{S_{n-2}^{\frac{1}{2}+\epsilon}}\le\frac{\alpha_0^{2}c^{2}}{S_{n-2}^{1+\epsilon}}\big\|\nabla_{\theta_{n-1}}g(\theta_{n-1})\big\|^{2}+\frac{\big\|\nabla_{\theta_{n-1}}g(\theta_{n-1},\xi_{n-1})\big\|^{2}}{S_{n-1}^{1+\epsilon}}\\&+\frac{\alpha_0^{2}c^{2}}{S_{n-1}^{\frac{3}{2}+\epsilon}}\big\|\nabla_{\theta_{n-1}}g(\theta_{n-1},\xi_{n-1})\big\|^{2}.
	%\end{aligned}\end{equation}
	Then we calculate $\frac{\alpha_0}{2}(M+1)\eqref{07}+\eqref{04}$
	\begin{equation}\nonumber\begin{aligned}
	&-\frac{\alpha_0\big\|\nabla_{\theta_{n}}g(\theta_{n})\big\|^{2}}{2S_{n-1}^{\frac{1}{2}+\epsilon}}+\frac{\alpha_0(M+1)}{2}\Bigg(\frac{\big\|\nabla_{\theta_{n}}g(\theta_{n})\big\|^{2}}{S_{n}^{\frac{1}{2}+\epsilon}}-\frac{\big\|\nabla_{\theta_{n-1}}g(\theta_{n-1})\big\|^{2}}{S_{n-1}^{\frac{1}{2}+\epsilon}}\Bigg)\\&\le-\frac{\alpha_0\big\|\nabla_{\theta_{n}}g(\theta_{n})\big\|^{2}}{20S_{n-1}^{\frac{1}{2}+\epsilon}}+5\alpha_0^{3}c^{2}(M+1)^{2}\frac{\big\|\nabla_{\theta_{n-1}}g(\theta_{n-1},\xi_{n-1})\big\|^{2}}{S_{n-1}^{1+\epsilon}}\\&+\frac{(M+1)\alpha_0^{3}c^{2}}{2S_{n-1}^{\frac{3}{2}+\epsilon}}\big\|\nabla_{\theta_{n-1}}g(\theta_{n-1},\xi_{n-1})\big\|^{2}+\frac{\alpha_0}{20}\Bigg(\frac{\big\|\nabla_{\theta_{n-1}}g(\theta_{n-1})\big\|^{2}}{S_{n-2}^{\frac{1}{2}+\epsilon}}-\frac{\big\|\nabla_{\theta_{n}}g(\theta_{n})\big\|^{2}}{S_{n-1}^{\frac{1}{2}+\epsilon}}\Bigg).\end{aligned}\end{equation}
	Move $\frac{\alpha_0(M+1)}{2}\Bigg(\frac{\big\|\nabla_{\theta_{n}}g(\theta_{n})\big\|^{2}}{S_{n}^{\frac{1}{2}+\epsilon}}-\frac{\big\|\nabla_{\theta_{n-1}}g(\theta_{n-1})\big\|^{2}}{S_{n-1}^{\frac{1}{2}+\epsilon}}\Bigg)$ to the right-hand side of the above inequality, then we have
	\begin{equation}\label{08}\begin{aligned}
	&-\frac{\alpha_0\big\|\nabla_{\theta_{n}}g(\theta_{n})\big\|^{2}}{2S_{n-1}^{\frac{1}{2}+\epsilon}}\le-\frac{\alpha_0\big\|\nabla_{\theta_{n}}g(\theta_{n})\big\|^{2}}{20S_{n-1}^{\frac{1}{2}+\epsilon}}+5\alpha_0^{3}c^{2}(M+1)^{2}\frac{\big\|\nabla_{\theta_{n-1}}g(\theta_{n-1},\xi_{n-1})\big\|^{2}}{S_{n-1}^{1+\epsilon}}\\&+\frac{(M+1)\alpha_0^{3}c^{2}}{2S_{n-1}^{\frac{3}{2}+\epsilon}}\big\|\nabla_{\theta_{n-1}}g(\theta_{n-1},\xi_{n-1})\big\|^{2}+\frac{\alpha_0}{20}\Bigg(\frac{\big\|\nabla_{\theta_{n-1}}g(\theta_{n-1})\big\|^{2}}{S_{n-2}^{\frac{1}{2}+\epsilon}}-\frac{\big\|\nabla_{\theta_{n}}g(\theta_{n})\big\|^{2}}{S_{n-1}^{\frac{1}{2}+\epsilon}}\Bigg)\\&+\frac{\alpha_0(M+1)}{2}\Bigg(\frac{\big\|\nabla_{\theta_{n-1}}g(\theta_{n-1})\big\|^{2}}{S_{n-1}^{\frac{1}{2}+\epsilon}}-\frac{\big\|\nabla_{\theta_{n}}g(\theta_{n})\big\|^{2}}{S_{n}^{\frac{1}{2}+\epsilon}}\Bigg).\end{aligned}\end{equation}
	Substitute \eqref{08} into \eqref{02}, then we have
	\begin{equation}\label{88}\begin{aligned}
	&\frac{g(\theta_{n+1})}{S_{n+1}^{\epsilon}}-\frac{g(\theta_{n})}{S_{n}^{\epsilon}}\\&\le-\frac{\alpha_0\big\|\nabla_{\theta_{n}}g(\theta_{n})\big\|^{2}}{20S_{n-1}^{\frac{1}{2}+\epsilon}}+5\alpha_0^{3}c^{2}(M+1)^{2}\frac{\big\|\nabla_{\theta_{n-1}}g(\theta_{n-1},\xi_{n-1})\big\|^{2}}{S_{n-1}^{1+\epsilon}}\\&+\frac{(M+1)\alpha_0^{3}c^{2}}{2S_{n-1}^{\frac{3}{2}+\epsilon}}\big\|\nabla_{\theta_{n-1}}g(\theta_{n-1},\xi_{n-1})\big\|^{2}+\frac{\alpha_0}{20}\Bigg(\frac{\big\|\nabla_{\theta_{n-1}}g(\theta_{n-1})\big\|^{2}}{S_{n-2}^{\frac{1}{2}+\epsilon}}-\frac{\big\|\nabla_{\theta_{n}}g(\theta_{n})\big\|^{2}}{S_{n-1}^{\frac{1}{2}+\epsilon}}\Bigg)\\&+\frac{\alpha_0(M+1)}{2}\Bigg(\frac{\big\|\nabla_{\theta_{n-1}}g(\theta_{n-1})\big\|^{2}}{S_{n-1}^{\frac{1}{2}+\epsilon}}-\frac{\big\|\nabla_{\theta_{n}}g(\theta_{n})\big\|^{2}}{S_{n}^{\frac{1}{2}+\epsilon}}\Bigg)+\frac{\alpha_0 a(M+1)}{2}\Bigg(\frac{1}{S_{n-1}^{\frac{1}{2}+\epsilon}}-\frac{1}{S_{n}^{\frac{1}{2}+\epsilon}}\Bigg)\\&+\frac{c\alpha_0^{2}}{2}\frac{\big\|\nabla_{\theta_{n}}g(\theta_{n},\xi_{n})\big\|^{2}}{S_{n}^{1+\epsilon}}+A_{n}^{(\epsilon)}+B_{n}^{(\epsilon)}.
	\end{aligned}\end{equation}
	
	\subsection{Proof of Lemma \ref{lem8}}
	First of all, it holds that
	\begin{equation}\nonumber\begin{aligned}
	&\frac{g(\theta_{n+1})}{S_{n+1}^{\epsilon}}-\frac{g(\theta_{n})}{S_{n}^{\epsilon}}\\&=I\Big(\big\|\nabla_{\theta_{n}}g(\theta_{n})\big\|^{2}\le a\Big)\bigg(\frac{g(\theta_{n+1})}{S_{n+1}^{\epsilon}}-\frac{g(\theta_{n})}{S_{n}^{\epsilon}}\bigg)+I\Big(\big\|\nabla_{\theta_{n}}g(\theta_{n})\big\|^{2}> a\Big)\bigg(\frac{g(\theta_{n+1})}{S_{n+1}^{\epsilon}}-\frac{g(\theta_{n})}{S_{n}^{\epsilon}}\bigg).
	\end{aligned}\end{equation}
	$I\Big(\big\|\nabla_{\theta_{n}}g(\theta_{n})\big\|^{2}\leq a\Big)\in \mathscr{F}_{n-1}$ is the indicator function such that
	\begin{equation}\nonumber\begin{aligned}
	I\Big(\big\|\nabla_{\theta_{n}}g(\theta_{n})\big\|^{2}\le a\Big)=\left\{\begin{array}{rcl}
	1 & & {\big\|\nabla_{\theta_{n}}g(\theta_{n})\big\|^{2}\leq a}\\ \\
	0 & & {\big\|\nabla_{\theta_{n}}g(\theta_{n})\big\|^{2}> a}
	\end{array} \right.
	\end{aligned}\end{equation}
	For convenient, we abbreviate $I\Big(\big\|\nabla_{\theta_{n}}g(\theta_{n})\big\|^{2}\le a\Big)$ as $I_{n}^{\le a}$ and $I\Big(\big\|\nabla_{\theta_{n}}g(\theta_{n})\big\|^{2}> a\Big)$ as $I_{n}^{>a}$ in the following. \\Through Lemma \ref{lem7}, we get that
	\begin{equation}\label{001}\begin{aligned}
	&I_{n}^{\le a}\bigg(\frac{g(\theta_{n+1})}{S_{n+1}^{\epsilon}}-\frac{g(\theta_{n})}{S_{n}^{\epsilon}}\bigg)\\
	&\le -I_{n}^{\le a}\frac{\alpha_0\big\|\nabla_{\theta_{n}}g(\theta_{n})\big\|^{2}}{20S_{n-1}^{\frac{1}{2}+\epsilon}}+5\alpha_0^{3}c^{2}(M+1)^{2}I_{n}^{\le a}\frac{\big\|\nabla_{\theta_{n-1}}g(\theta_{n-1},\xi_{n-1})\big\|^{2}}{S_{n-1}^{1+\epsilon}}\\&+I_{n}^{\le a}\frac{(M+1)\alpha_0^{3}c^{2}}{2S_{n-1}^{\frac{3}{2}+\epsilon}}\big\|\nabla_{\theta_{n-1}}g(\theta_{n-1},\xi_{n-1})\big\|^{2}+\frac{\alpha_0}{20}I_{n}^{\le a}\Bigg(\frac{\big\|\nabla_{\theta_{n-1}}g(\theta_{n-1})\big\|^{2}}{S_{n-2}^{\frac{1}{2}+\epsilon}}-\frac{\big\|\nabla_{\theta_{n}}g(\theta_{n})\big\|^{2}}{S_{n-1}^{\frac{1}{2}+\epsilon}}\Bigg)\\&+\frac{\alpha_0(M+1)}{2}I_{n}^{\le a}\Bigg(\frac{\big\|\nabla_{\theta_{n-1}}g(\theta_{n-1})\big\|^{2}}{S_{n-1}^{\frac{1}{2}+\epsilon}}-\frac{\big\|\nabla_{\theta_{n}}g(\theta_{n})\big\|^{2}}{S_{n}^{\frac{1}{2}+\epsilon}}\Bigg)+\frac{\alpha_0 a(M+1)}{2}I_{n}^{\le a}\Bigg(\frac{1}{S_{n-1}^{\frac{1}{2}+\epsilon}}-\frac{1}{S_{n}^{\frac{1}{2}+\epsilon}}\Bigg)\\&+\frac{c\alpha_0^{2}}{2}I_{n}^{\le a}\frac{\big\|\nabla_{\theta_{n}}g(\theta_{n},\xi_{n})\big\|^{2}}{S_{n}^{1+\epsilon}}+I_{n}^{\le a}A_{n}^{(\epsilon)}+I_{n}^{\le a}B_{n}^{(\epsilon)}.
	\end{aligned}\end{equation}
	Through Lemma \ref{lem6}, we get
	\begin{equation}\label{002}\begin{aligned}
	&I_{n}^{>a}\bigg(\frac{g(\theta_{n+1})}{S_{n+1}^{\epsilon}}-\frac{g(\theta_{n})}{S_{n}^{\epsilon}}\bigg)\\&\le-I_{n}^{>a}\frac{\alpha_0}{20}\frac{\big\|\nabla_{\theta_{n}}g(\theta_{n})\big\|^{2}}{S_{n-1}^{\frac{1}{2}+\epsilon}}+\frac{\alpha_0(M+1)}{2}`  I_{n}^{>a}\Bigg(\frac{\big\|\nabla_{\theta_{n-1}}g(\theta_{n-1})\big\|^{2}}{S_{n-1}^{\frac{1}{2}+{\epsilon}}}-\frac{\big\|\nabla_{\theta_{n}}g(\theta_{n})\big\|^{2}}{S_{n}^{\frac{1}{2}+{\epsilon}}}\Bigg)\\&+\frac{\alpha_0}{20}I_{n}^{>a}\Bigg(\frac{\big\|\nabla_{\theta_{n-1}}g(\theta_{n-1})\big\|^{2}}{S_{n-2}^{\frac{1}{2}+\epsilon}}-\frac{\big\|\nabla_{\theta_{n}}g(\theta_{n})\big\|^{2}}{S_{n-1}^{\frac{1}{2}+\epsilon}}\Bigg)\\&4M^{2}\alpha_0^{3}c^{2}I_{n}^{>a}\frac{\big\|\nabla_{\theta_{n-1}}g(\theta_{n-1}),\xi_{n-1}\big\|^{2}}{S_{n-1}^{\frac{3}{2}+\epsilon}}+\frac{c\alpha_0^{2}}{2}I_{n}^{>a}\frac{\big\|\nabla_{\theta_{n}}g(\theta_{n},\xi_{n})\big\|^{2}}{S_{n}^{1+\epsilon}}\\&+I_{n}^{>a}X_{n}^{(\epsilon)}+I_{n}^{>a}Y_{n}^{(\epsilon)}.
	\end{aligned}\end{equation}
	Calculate $\eqref{001}+\eqref{002}$, then it holds that
	\begin{equation}\label{00}\begin{aligned}
	&I_{n}^{\le a}\bigg(\frac{g(\theta_{n+1})}{S_{n+1}^{\epsilon}}-\frac{g(\theta_{n})}{S_{n}^{\epsilon}}\bigg)+I_{n}^{>a}\bigg(\frac{g(\theta_{n+1})}{S_{n+1}^{\epsilon}}-\frac{g(\theta_{n})}{S_{n}^{\epsilon}}\bigg)\\&\le\frac{\alpha_0}{20}(I_{n}^{\le a}+I_{n}^{>a})\Bigg(\frac{\big\|\nabla_{\theta_{n-1}}g(\theta_{n-1})\big\|^{2}}{S_{n-2}^{\frac{1}{2}+\epsilon}}-\frac{\big\|\nabla_{\theta_{n}}g(\theta_{n})\big\|^{2}}{S_{n-1}^{\frac{1}{2}+\epsilon}}\Bigg)\\&+\frac{\alpha_0}{2}(M+1)(I_{n}^{\le a}+I_{n}^{>a})\Bigg(\frac{\big\|\nabla_{\theta_{n-1}}g(\theta_{n-1})\big\|^{2}}{S_{n-1}^{\frac{1}{2}+{\epsilon}}}-\frac{\big\|\nabla_{\theta_{n}}g(\theta_{n})\big\|^{2}}{S_{n}^{\frac{1}{2}+{\epsilon}}}\Bigg)\\&-\frac{\alpha_0}{20}\big(I_{n}^{\le a}+I_{n}^{>a}\big)\frac{\big\|\nabla_{\theta_{n}}g(\theta_{n})\big\|^{2}}{S_{n-1}^{\frac{1}{2}+\epsilon}}+4M^{2}\alpha_0^{3}c^{2}I_{n}^{>a}\frac{\big\|\nabla_{\theta_{n-1}}g(\theta_{n-1}),\xi_{n-1}\big\|^{2}}{S_{n-1}^{\frac{3}{2}+\epsilon}}\\&+\frac{c\alpha_0^{2}}{2}I_{n}^{>a}\frac{\big\|\nabla_{\theta_{n}}g(\theta_{n},\xi_{n})\big\|^{2}}{S_{n}^{1+\epsilon}}+5\alpha_0^{3}c^{2}(M+1)^{2}I_{n}^{\le a}\frac{\big\|\nabla_{\theta_{n-1}}g(\theta_{n-1},\xi_{n-1})\big\|^{2}}{S_{n-1}^{1+\epsilon}}\\&+I_{n}^{\le a}\frac{(M+1)\alpha_0^{3}c^{2}}{2S_{n-1}^{\frac{3}{2}+\epsilon}}\big\|\nabla_{\theta_{n-1}}g(\theta_{n-1},\xi_{n-1})\big\|^{2}\\&+\frac{\alpha_0 a(M+1)}{2}I_{n}^{\le a}\Bigg(\frac{1}{S_{n-1}^{\frac{1}{2}+\epsilon}}-\frac{1}{S_{n}^{\frac{1}{2}+\epsilon}}\Bigg)+\frac{c\alpha_0^{2}}{2}I_{n}^{\le a}\frac{\big\|\nabla_{\theta_{n}}g(\theta_{n},\xi_{n})\big\|^{2}}{S_{n}^{1+\epsilon}}\\&+I_{n}^{\le a}A_{n}^{(\epsilon)}+I_{n}^{\le a}B_{n}^{(\epsilon)}+I_{n}^{>a}X_{n}^{(\epsilon)}+I_{n}^{>a}Y_{n}^{(\epsilon)}.
	\end{aligned}\end{equation}
	Notice $I_{n}^{\le a}\le 1$, then we get
	\begin{equation}\label{005}\begin{aligned}
	\frac{\alpha_0 a(M+1)}{2}I_{n}^{\le a}\Bigg(\frac{1}{S_{n-1}^{\frac{1}{2}+\epsilon}}-\frac{1}{S_{n}^{\frac{1}{2}+\epsilon}}\Bigg)\le \frac{\alpha_0 a(M+1)}{2}\Bigg(\frac{1}{S_{n-1}^{\frac{1}{2}+\epsilon}}-\frac{1}{S_{n}^{\frac{1}{2}+\epsilon}}\Bigg).
	\end{aligned}\end{equation}
	Substitute \eqref{005} into \eqref{00}, then we get
	\begin{equation}\label{004}\begin{aligned}
	&\frac{g(\theta_{n+1})}{S_{n+1}^{\epsilon}}-\frac{g(\theta_{n})}{S_{n}^{\epsilon}}\\&\le\frac{\alpha_0}{20}\Bigg(\frac{\big\|\nabla_{\theta_{n-1}}g(\theta_{n-1})\big\|^{2}}{S_{n-2}^{\frac{1}{2}+\epsilon}}-\frac{\big\|\nabla_{\theta_{n}}g(\theta_{n})\big\|^{2}}{S_{n-1}^{\frac{1}{2}+\epsilon}}\Bigg)\\&+\frac{\alpha_0}{2}(M+1)\Bigg(\frac{\big\|\nabla_{\theta_{n-1}}g(\theta_{n-1})\big\|^{2}}{S_{n-1}^{\frac{1}{2}+{\epsilon}}}-\frac{\big\|\nabla_{\theta_{n}}g(\theta_{n})\big\|^{2}}{S_{n}^{\frac{1}{2}+{\epsilon}}}\Bigg)\\&-\frac{\alpha_0}{20}\frac{\big\|\nabla_{\theta_{n}}g(\theta_{n})\big\|^{2}}{S_{n-1}^{\frac{1}{2}+\epsilon}}+4M^{2}\alpha_0^{3}c^{2}I_{n}^{>a}\frac{\big\|\nabla_{\theta_{n-1}}g(\theta_{n-1}),\xi_{n-1}\big\|^{2}}{S_{n-1}^{\frac{3}{2}+\epsilon}}+\frac{c\alpha_0^{2}}{2}I_{n}^{>a}\frac{\big\|\nabla_{\theta_{n}}g(\theta_{n},\xi_{n})\big\|^{2}}{S_{n}^{1+\epsilon}}\\&+5\alpha_0^{3}c^{2}(M+1)^{2}I_{n}^{\le a}\frac{\big\|\nabla_{\theta_{n-1}}g(\theta_{n-1},\xi_{n-1})\big\|^{2}}{S_{n-1}^{1+\epsilon}}+I_{n}^{\le a}\frac{(M+1)\alpha_0^{3}c^{2}}{2S_{n-1}^{\frac{3}{2}+\epsilon}}\big\|\nabla_{\theta_{n-1}}g(\theta_{n-1},\xi_{n-1})\big\|^{2}\\&+\frac{\alpha_0 a(M+1)}{2}\Bigg(\frac{1}{S_{n-1}^{\frac{1}{2}+\epsilon}}-\frac{1}{S_{n}^{\frac{1}{2}+\epsilon}}\Bigg)+\frac{c\alpha_0^{2}}{2}I_{n}^{\le a}\frac{\big\|\nabla_{\theta_{n}}g(\theta_{n},\xi_{n})\big\|^{2}}{S_{n}^{1+\epsilon}}+I_{n}^{\le a}A_{n}^{(\epsilon)}+I_{n}^{\le a}B_{n}^{(\epsilon)}+I_{n}^{>a}X_{n}^{(\epsilon)}\\&+I_{n}^{>a}Y_{n}^{(\epsilon)}.
	\end{aligned}\end{equation}
	We make a summation of \eqref{004} to get
	\begin{equation}\label{006}\begin{aligned}
	&\sum_{k=3}^{n}\bigg(\frac{g(\theta_{k+1})}{S_{k+1}^{\epsilon}}-\frac{g(\theta_{k})}{S_{k}^{\epsilon}}\bigg)\\&\le\frac{\alpha_0}{20}\sum_{k=3}^{n}\Bigg(\frac{\big\|\nabla_{\theta_{k-1}}g(\theta_{k-1})\big\|^{2}}{S_{k-2}^{\frac{1}{2}+\epsilon}}-\frac{\big\|\nabla_{\theta_{k}}g(\theta_{k})\big\|^{2}}{S_{k-1}^{\frac{1}{2}+\epsilon}}\Bigg)\\&+\frac{\alpha_0}{2}(M+1)\sum_{k=3}^{n}\Bigg(\frac{\big\|\nabla_{\theta_{k-1}}g(\theta_{k-1})\big\|^{2}}{S_{k-1}^{\frac{1}{2}+{\epsilon}}}-\frac{\big\|\nabla_{\theta_{k}}g(\theta_{k})\big\|^{2}}{S_{k}^{\frac{1}{2}+{\epsilon}}}\Bigg)\\&-\frac{\alpha_0}{20}\sum_{k=3}^{n}\frac{\big\|\nabla_{\theta_{k}}g(\theta_{k})\big\|^{2}}{S_{k-1}^{\frac{1}{2}+\epsilon}}+4M^{2}\alpha_0^{3}c^{2}\sum_{k=3}^{n}I_{k}^{>a}\frac{\big\|\nabla_{\theta_{k-1}}g(\theta_{k-1}),\xi_{k-1}\big\|^{2}}{S_{k-1}^{\frac{3}{2}+\epsilon}}\\&+\frac{c\alpha_0^{2}}{2}\sum_{k=3}^{n}I_{k}^{>a}\frac{\big\|\nabla_{\theta_{k}}g(\theta_{k},\xi_{k})\big\|^{2}}{S_{k}^{1+\epsilon}}+5\alpha_0^{3}c^{2}(M+1)^{2}\sum_{k=2}^{n}I_{k}^{\le a}\frac{\big\|\nabla_{\theta_{k-1}}g(\theta_{k-1},\xi_{k-1})\big\|^{2}}{S_{k-1}^{1+\epsilon}}\\&+\sum_{k=3}^{n}I_{k}^{\le a}\frac{(M+1)\alpha_0^{3}c^{2}}{2S_{k-1}^{\frac{3}{2}+\epsilon}}\big\|\nabla_{\theta_{k-1}}g(\theta_{k-1},\xi_{k-1})\big\|^{2}\\&+\frac{\alpha_0 a(M+1)}{2}\sum_{k=3}^{n}\Bigg(\frac{1}{S_{k-1}^{\frac{1}{2}+\epsilon}}-\frac{1}{S_{k}^{\frac{1}{2}+\epsilon}}\Bigg)+\frac{c\alpha_0^{2}}{2}\sum_{k=3}^{n}I_{k}^{\le a}\frac{\big\|\nabla_{\theta_{k}}g(\theta_{k},\xi_{k})\big\|^{2}}{S_{k}^{1+\epsilon}}\\&+\sum_{k=3}^{n}\Big(I_{k}^{\le a}A_{k}^{(\epsilon)}+I_{k}^{\le a}B_{k}^{(\epsilon)}+I_{k}^{>a}X_{k}^{(\epsilon)}+I_{k}^{>a}Y_{k}^{(\epsilon)}\Big).
	\end{aligned}\end{equation}
	It follows from Lemma \ref{lem5} that
	\begin{equation}\label{007}\begin{aligned}
	\sum_{k=3}^{n}\frac{\big\|\nabla_{\theta_{k}}g(\theta_{k},\xi_{k})\big\|^{2}}{S_{k}^{1+\epsilon}}\le \int_{S_{3}}^{+\infty}\frac{1}{x^{1+\epsilon}}dx=\frac{1}{\epsilon S_{3}^{\epsilon}},
	\end{aligned}\end{equation}
	and
	\begin{equation}\label{008}\begin{aligned}
	\sum_{k=3}^{n}\frac{\big\|\nabla_{\theta_{k-1}}g(\theta_{k-1},\xi_{k-1})\big\|^{2}}{S_{k-1}^{1+\epsilon}}\le  \int_{S_{2}}^{+\infty}\frac{1}{x^{1+\epsilon}}dx=\frac{1}{\epsilon S_{2}^{\epsilon}},
	\end{aligned}\end{equation}
	and
	\begin{equation}\label{009}\begin{aligned}
	\sum_{k=3}^{n}\frac{\big\|\nabla_{\theta_{k-1}}g(\theta_{k-1},\xi_{k-1})\big\|^{2}}{S_{k-1}^{\frac{3}{2}+\epsilon}}\le \int_{S_{2}}^{+\infty}\frac{1}{x^{\frac{3}{2}+\epsilon}}dx=\frac{2}{(1+2\epsilon) S_{2}^{\frac{1}{2}+\epsilon}}.
	\end{aligned}\end{equation}
	Due to $I_{k}^{\le a}\le 1$ and $I_{k}^{>a}\le 1$, we get that
	\begin{equation}\label{010}\begin{aligned}
	&4M^{2}\alpha_0^{3}c^{2}\sum_{k=3}^{n}I_{k}^{>a}\frac{\big\|\nabla_{\theta_{k-1}}g(\theta_{k-1}),\xi_{k-1}\big\|^{2}}{S_{k-1}^{\frac{3}{2}+\epsilon}}+\frac{c\alpha_0^{2}}{2}\sum_{k=3}^{n}I_{k}^{>a}\frac{\big\|\nabla_{\theta_{k}}g(\theta_{k},\xi_{k})\big\|^{2}}{S_{k}^{1+\epsilon}}\\&+5\alpha_0^{3}c^{2}(M+1)^{2}\sum_{k=3}^{n}I_{k}^{\le a}\frac{\big\|\nabla_{\theta_{k-1}}g(\theta_{k-1},\xi_{k-1})\big\|^{2}}{S_{k-1}^{1+\epsilon}}\\&+\sum_{k=3}^{n}I_{k}^{\le a}\frac{(M+1)\alpha_0^{3}c^{2}}{2S_{k-1}^{\frac{3}{2}+\epsilon}}\big\|\nabla_{\theta_{k-1}}g(\theta_{k-1},\xi_{k-1})\big\|^{2}+\frac{c\alpha_0^{2}}{2}\sum_{k=3}^{n}I_{k}^{\le a}\frac{\big\|\nabla_{\theta_{k}}g(\theta_{k},\xi_{k})\big\|^{2}}{S_{k}^{1+\epsilon}}\\&\le \frac{8M^{2}\alpha_0^{3}c^{2}}{(1+2\epsilon) S_{2}^{\frac{1}{2}+\epsilon}}+\frac{c\alpha_0^{2}}{2\epsilon S_{3}^{\epsilon}}+\frac{5\alpha_0^{3}c^{2}(M+1)^{2}}{\epsilon S_{2}^{\epsilon}}+\frac{\alpha_0^{3}c^{2}(M+1)}{(1+2\epsilon) S_{2}^{\frac{1}{2}+\epsilon}}+\frac{c\alpha_0^{2}}{2\epsilon S_{3}^{\epsilon}}:=K.
	\end{aligned}\end{equation} 
% 	For convenience, let
% 	\begin{equation}\nonumber\begin{aligned}
% 	K=\frac{8M^{2}\alpha_0^{3}c^{2}}{(1+2\epsilon) S_{2}^{\frac{1}{2}+\epsilon}}+\frac{c\alpha_0^{2}}{2\epsilon S_{3}^{\epsilon}}+\frac{5\alpha_0^{3}c^{2}(M+1)^{2}}{\epsilon S_{2}^{\epsilon}}+\frac{\alpha_0^{3}c^{2}(M+1)}{(1+2\epsilon) S_{2}^{\frac{1}{2}+\epsilon}}+\frac{c\alpha_0^{2}}{2\epsilon S_{3}^{\epsilon}}.
% 	\end{aligned}\end{equation}
	
	Substituting \eqref{010} into \eqref{006} leads to 
	\begin{equation}\nonumber\begin{aligned}
	&\frac{g(\theta_{n+1})}{S_{n+1}^{\epsilon}}-\frac{g(\theta_{3})}{S_{3}^{\epsilon}}\\&\le\frac{\alpha_0}{20}\Bigg(\frac{\big\|\nabla_{\theta_{2}}g(\theta_{2})\big\|^{2}}{S_{1}^{\frac{1}{2}+\epsilon}}-\frac{\big\|\nabla_{\theta_{n}}g(\theta_{n})\big\|^{2}}{S_{n-1}^{\frac{1}{2}+\epsilon}}\Bigg)+\frac{\alpha_0}{2}(M+1)\Bigg(\frac{\big\|\nabla_{\theta_{2}}g(\theta_{2})\big\|^{2}}{S_{2}^{\frac{1}{2}+{\epsilon}}}-\frac{\big\|\nabla_{\theta_{n}}g(\theta_{n})\big\|^{2}}{S_{n}^{\frac{1}{2}+{\epsilon}}}\Bigg)\\&-\frac{\alpha_0}{20}\sum_{k=3}^{n}\frac{\big\|\nabla_{\theta_{k}}g(\theta_{k})\big\|^{2}}{S_{k-1}^{\frac{1}{2}+\epsilon}}+\frac{\alpha_0 a(M+1)}{2}\Bigg(\frac{1}{S_{2}^{\frac{1}{2}+\epsilon}}-\frac{1}{S_{n}^{\frac{1}{2}+\epsilon}}\Bigg)+K\\&+\sum_{k=3}^{n}\Big(I_{k}^{\le a}A_{k}^{(\epsilon)}+I_{k}^{\le a}B_{k}^{(\epsilon)}+I_{k}^{>a}X_{k}^{(\epsilon)}+I_{k}^{>a}Y_{k}^{(\epsilon)}\Big).
	\end{aligned}\end{equation}
	It is obvious that
	\begin{equation}\nonumber\begin{aligned}
	&\frac{\alpha_0}{20}\Bigg(\frac{\big\|\nabla_{\theta_{0}}g(\theta_{0})\big\|^{2}}{S_{1}^{\frac{1}{2}+\epsilon}}-\frac{\big\|\nabla_{\theta_{n}}g(\theta_{n})\big\|^{2}}{S_{n-1}^{\frac{1}{2}+\epsilon}}\Bigg)+\frac{\alpha_0}{2}(M+1)\Bigg(\frac{\big\|\nabla_{\theta_{2}}g(\theta_{2})\big\|^{2}}{S_{2}^{\frac{1}{2}+{\epsilon}}}-\frac{\big\|\nabla_{\theta_{n}}g(\theta_{n})\big\|^{2}}{S_{n}^{\frac{1}{2}+{\epsilon}}}\Bigg)\\&+\frac{\alpha_0 a(M+1)}{2}\Bigg(\frac{1}{S_{2}^{\frac{1}{2}+\epsilon}}-\frac{1}{S_{n}^{\frac{1}{2}+\epsilon}}\Bigg)+K\le\frac{\alpha_0}{20}\frac{\big\|\nabla_{\theta_{2}}g(\theta_{2})\big\|^{2}}{S_{1}^{\frac{1}{2}+\epsilon}}+\frac{\alpha_0(M+1)}{2}\frac{\big\|\nabla_{\theta_{2}}g(\theta_{2})\big\|^{2}}{S_{2}^{\frac{1}{2}+{\epsilon}}}\\&+\frac{\alpha_0 a(M+1)}{2S_{2}^{\frac{1}{2}+\epsilon}}+K:=L.
	\end{aligned}\end{equation}
% 	For convenience, let
% 	\begin{align*}
% 	L=&\frac{\alpha_0}{20}\frac{\big\|\nabla_{\theta_{2}}g(\theta_{2})\big\|^{2}}{S_{1}^{\frac{1}{2}+\epsilon}}+\frac{\alpha_0(M+1)}{2}\frac{\big\|\nabla_{\theta_{2}}g(\theta_{2})\big\|^{2}}{S_{2}^{\frac{1}{2}+{\epsilon}}}\\
% 	&+\frac{\alpha_0 a(M+1)}{2S_{2}^{\frac{1}{2}+\epsilon}}+K.
% 	\end{align*}
	It follows that
	\begin{equation}\label{020}\begin{aligned}
	&\frac{g(\theta_{n+1})}{S_{n+1}^{\epsilon}}-\frac{g(\theta_{3})}{S_{3}^{\epsilon}}\\&\le-\frac{\alpha_0}{20}\sum_{k=3}^{n}\frac{\big\|\nabla_{\theta_{k}}g(\theta_{k})\big\|^{2}}{S_{k-1}^{\frac{1}{2}+\epsilon}} +L+\sum_{k=3}^{n}\Big(I_{k}^{\le a}A_{k}^{(\epsilon)}+I_{k}^{\le a}B_{k}^{(\epsilon)}+I_{k}^{>a}X_{k}^{(\epsilon)}+I_{k}^{>a}Y_{k}^{(\epsilon)}\Big).
	\end{aligned}\end{equation}
	Note that $\{I_{k}^{\le a}A_{k}\}$, $\{I_{k}^{\le a}B_{k}\}$, $\{I_{k}^{>a}X_{k}\}$ and $\{I_{k}^{>a}Y_{k}\}$ are all martingale difference sequences,  thus it follows that $\Expect\bigg(\sum_{k=1}^{n}\Big(I_{k}^{\le a}A_{k}+I_{k}^{\le a}B_{k}+I_{k}^{>a}X_{k}+I_{k}^{>a}Y_{k}\Big)\bigg)=0$. Then we calculate mathematical expectation on \eqref{020}
	\begin{equation}\nonumber\begin{aligned}
	&\Expect\bigg(\frac{g(\theta_{n+1})}{S_{n+1}^{\epsilon}}-\frac{g(\theta_{3})}{S_{3}^{\epsilon}}\bigg)\le-\frac{\alpha_0}{20}\Expect\Bigg(\sum_{k=3}^{n}\frac{\big\|\nabla_{\theta_{k}}g(\theta_{k})\big\|^{2}}{S_{k-1}^{\frac{1}{2}+\epsilon}}\Bigg) +L.
	\end{aligned}\end{equation}
	That is
	\begin{equation}\label{021}\begin{aligned}
	&\Expect\Bigg(\sum_{k=3}^{n}\frac{\big\|\nabla_{\theta_{k}}g(\theta_{k})\big\|^{2}}{S_{k-1}^{\frac{1}{2}+\epsilon}}\Bigg)<\frac{20}{\alpha_0}\bigg(\frac{g(\theta_{3})}{S_{3}^{\epsilon}}+L\bigg)<+\infty.
	\end{aligned}\end{equation}
	From Lemma \ref{lem_summation}, it holds that
	\begin{equation}\label{022}\begin{aligned}
	&\sum_{k=3}^{n}\frac{\big\|\nabla_{\theta_{k}}g(\theta_{k})\big\|^{2}}{S_{k-1}^{\frac{1}{2}+\epsilon}}<+\infty\ \ a.s..
	\end{aligned}\end{equation}
	
	\subsection{Proof of Lemma \ref{lem9}}
	%			For this goal, we made some calculations as follows
	It follows from \eqref{020} that
	\begin{equation}\label{0000}
	\begin{aligned}
	&\frac{g(\theta_{n+1})}{S_{n+1}^{\epsilon}}\le\frac{g(\theta_{3})}{S_{3}^{\epsilon}} +L+\sum_{k=3}^{n}\Big(I_{k}^{\le a}A_{k}^{(\epsilon)}+I_{k}^{\le a}B_{k}^{(\epsilon)}+I_{k}^{>a}X_{k}^{(\epsilon)}+I_{k}^{>a}Y_{k}^{(\epsilon)}\Big).
	\end{aligned}\end{equation}
	From \eqref{def_A_B}, we obtain
	\begin{equation}\nonumber\begin{aligned}
	\sum_{k=3}^{n}\Expect\Big(\big\|I_{k}^{\le a}A_{k}^{(\epsilon)}\big\|^{2}\Big)\le \sum_{k=3}^{n}\Expect\Bigg(I_{k}^{\le a}\frac{\alpha_0^{2}}{S_{n-1}^{1+2\epsilon}}\Big(\big\|\nabla_{\theta_{n}}g(\theta_{n})\big\|^{2}-\nabla_{\theta_{n}}g(\theta_{n})^{T}\nabla_{\theta_{n}}g(\theta_{n},\xi_{n})\Big)^{2}\Bigg).\end{aligned}\end{equation}
	With   inequality $(a+b)^{2}\le 2 (a^{2}+b^{2}), \ 2a^{T}b\le a^{2}+b^{2} \ \ (a,\ b>0)$ and \eqref{021}, we   get
	\begin{equation}\label{0001}\begin{aligned}
	&\Expect\Bigg(I_{k}^{\le a}\frac{\alpha_0^{2}}{S_{n-1}^{1+2\epsilon}}\Big(\big\|\nabla_{\theta_{n}}g(\theta_{n})\big\|^{2}-\nabla_{\theta_{n}}g(\theta_{n})^{T}\nabla_{\theta_{n}}g(\theta_{n},\xi_{n})\Big)^{2}\Bigg)\\
	&\le 2\Expect\Bigg(\frac{\alpha_0^{2}\big\|\nabla_{\theta_{n}}{g(\theta_{n})}\big\|^{2}}{S_{n-1}^{1+2\epsilon}}\Big(I_{k}^{\le a}\big\|\nabla_{\theta_{n}}g(\theta_{n})\big\|^{2}\Big)\Bigg)+2\Expect\Bigg(\frac{\alpha_0^{2}I_{k}^{\le a}}{S_{n-1}^{1+2\epsilon}}\big\|\nabla_{\theta_{n}}g(\theta_{n})\big\|^{2}\big\|\nabla_{\theta_{n}}g(\theta_{n},\xi_{n})\big\|^{2}\Bigg)\\
	&\le2(M+2)\Expect\Bigg(\frac{\alpha_0^{2}a\big\|\nabla_{\theta_{n}}{g(\theta_{n})}\big\|^{2}}{S_{n-1}^{1+\epsilon}}\Bigg)\\
	&\le 2(M+2)\alpha_0^{2}a\Expect\Bigg(\frac{\alpha_0^{2}a\big\|\nabla_{\theta_{n}}{g(\theta_{n})}\big\|^{2}}{S_{n-1}^{\frac{1}{2}+\epsilon}}\Bigg)\\&<40(M+2)\alpha_0 a\bigg(\frac{g(\theta_{1})}{S_{1}^{\epsilon}}+L\bigg).%\Expect\Bigg(I_{k}^{\le a}\frac{\alpha_0^{2}}{S_{n-1}^{1+\epsilon}}\Big(\nabla_{\theta_{n}}g(\theta_{n})^{T}\nabla_{\theta_{n}}g(\theta_{n},\xi_{n})\Big)^{2}\Bigg).
	\end{aligned}\end{equation}
	It follows from Lemma \ref{lem_summation_MDS} that $\sum_{k=3}^{n}I_{k}^{\le a}A_{k}^{(\epsilon)}$ is convergent a.s. Similarly,   $\sum_{k=3}^{n}I_{k}^{\le a}B_{k}^{(\epsilon)}$, $\sum_{k=3}^{n}I_{k}^{> a}X_{k}^{(\epsilon)}$ and $\sum_{k=3}^{n}I_{k}^{> a}Y_{k}^{(\epsilon)}$ are both convergent a.s. It follows that
	\begin{equation}\nonumber
	\begin{aligned}
	&\sum_{k=3}^{n}\Big(I_{k}^{\le a}A_{k}^{(\epsilon)}+I_{k}^{\le a}B_{k}^{(\epsilon)}+I_{k}^{>a}X_{k}^{(\epsilon)}+I_{k}^{>a}Y_{k}^{(\epsilon)}\Big)<\xi^{'}<+\infty\ \ a.s.,
	\end{aligned}\end{equation}
	\begin{equation}\nonumber
	\begin{aligned}
	&\frac{g(\theta_{n+1})}{S_{n+1}^{\epsilon}}\le\frac{g(\theta_{3})}{S_{3}^{\epsilon}} +L+\xi^{'}<+\infty\ \ a.s.
	\end{aligned}\end{equation}
	For convenience, let $\xi=\frac{g(\theta_{3})}{S_{3}^{\epsilon}} +L+\xi^{'}$. Thus, it holds that
	\begin{equation}\label{000000000001}
	\begin{aligned}
	&\frac{g(\theta_{n+1})-g^{*}}{S_{n+1}^{\epsilon}}<\xi<+\infty\ \ a.s..
	\end{aligned}\end{equation}

	\subsection{Proof of Lemma \ref{lem10}}\label{append_pf_lem10}
	First of all, $\forall \ 0<\epsilon_{0}<\frac{3}{8}$,  there is $0<\frac{4}{3}\epsilon_{0}<\frac{1}{2}$. From Lemma \ref{lem8}, it follows that
	\begin{equation}\label{i-1}\begin{aligned}
	&\sum_{n=4}^{+\infty}\frac{\big\|\nabla_{\theta_{n-1}}g(\theta_{n-1})\big\|^{2}}{S_{n-2}^{\frac{1}{2}+\frac{4}{3}\epsilon_{0}}}<+\infty\ \ a.s..
	\end{aligned}\end{equation}
	It follows from \eqref{05}  that \begin{equation}\label{i0}\begin{aligned}
	&\big\|\nabla_{\theta_{n}}g(\theta_{n})\big\|^{2}-\big\|\nabla_{\theta_{n-1}}g(\theta_{n-1})\big\|^{2}\le\frac{\alpha_0 c}{\sqrt{S_{n-1}}}\Big(\big\|\nabla_{\theta_{n-1}}g(\theta_{n-1})\big\|^{2}+\big\|\nabla_{\theta_{n-1}}g(\theta_{n-1},\xi_{n-1})\big\|^{2}\Big)\\&+\frac{\alpha_0^{2}c^{2}}{S_{n-1}}\big\|\nabla_{\theta_{n-1}}g(\theta_{n-1},\xi_{n-1})\big\|^{2}.
	\end{aligned}\end{equation}
	Divide  both sides of \eqref{i0} by $S_{n-1}^{\epsilon_{0}}$ and notice that $S_{n}>S_{n-1}>S_{n-2}$, then we get 
	\begin{equation}\label{i1}\begin{aligned}
	&\frac{\big\|\nabla_{\theta_{n}}g(\theta_{n})\big\|^{2}}{S_{n}^{\epsilon_{0}}}-\frac{\big\|\nabla_{\theta_{n-1}}g(\theta_{n-1})\big\|^{2}}{S_{n-1}^{\epsilon_{0}}}\le\frac{\alpha_0 c\big\|\nabla_{\theta_{n-1}}g(\theta_{n-1})\big\|^{2}}{S_{n-2}^{\frac{1}{2}+\epsilon_{0}}}+\frac{\alpha_0 c\big\|\nabla_{\theta_{n-1}}g(\theta_{n-1},\xi_{n-1})\big\|^{2}}{S_{n-2}^{\frac{1}{2}+\epsilon_{0}}}\\&+\frac{\alpha_0^{2}c^{2}\big\|\nabla_{\theta_{n-1}}g(\theta_{n-1},\xi_{n-1})\big\|^{2}}{S_{n-2}^{1+\epsilon_{0}}}.
	\end{aligned}\end{equation}
	Note that $0<\frac{2}{3}\epsilon_{0}<\frac{1}{2}$, then it follows from {Lemma \ref{lemiuy8}} and Lemma \ref{lem9} that
	\begin{equation}\nonumber\begin{aligned}
	&\frac{\big\|\nabla_{\theta_{n-1}}g(\theta_{n-1})\big\|^{2}}{S_{n-2}^{\frac{2}{3}\epsilon_{0}}}\le \frac{2c\big(g(\theta_{n-1})-g^{*}\big)}{S_{n-2}^{\frac{2}{3}\epsilon_{0}}}<2c\xi<+\infty.
	\end{aligned}\end{equation}
It follows that
	\begin{equation}\nonumber\begin{aligned}
	&\frac{\big\|\nabla_{\theta_{n}}g(\theta_{n})\big\|^{2}}{S_{n}^{\epsilon_{0}}}-\frac{\big\|\nabla_{\theta_{n-1}}g(\theta_{n-1})\big\|^{2}}{S_{n-1}^{\epsilon_{0}}}\le\frac{\alpha_0 ct\xi}{S_{n-2}^{\frac{1}{2}+\frac{1}{3}\epsilon_{0}}}+\frac{\alpha_0 cMt\xi}{S_{n-2}^{\frac{1}{2}+\frac{1}{3}\epsilon_{0}}}+\frac{\alpha_0^{2}c^{2}Mt\xi}{S_{n-2}^{\frac{1}{2}+\frac{1}{3}\epsilon_{0}}}+\frac{\alpha_0 c a}{S_{n-2}^{\frac{1}{2}+\frac{1}{3}\epsilon_{0}}}\\&+\frac{\alpha_0^{2}c^{2}a}{S_{n-2}^{\frac{1}{2}+\frac{1}{3}\epsilon_{0}}}-\frac{\alpha_0 c}{S_{n-2}^{\frac{1}{2}+\epsilon_{0}}}\Big(\Expect\Big(\big\|\nabla_{\theta_{n-1}}g(\theta_{n-1},\xi_{n-1})\big\|^{2}\Big|\mathscr{F}_{n-2}\Big)-\big\|\nabla_{\theta_{n-1}}g(\theta_{n-1},\xi_{n-1})\big\|^{2}\Big)\\&-\frac{\alpha_0^2 c^2}{S_{n-2}^{1+\epsilon_{0}}}\Big(\Expect\Big(\big\|\nabla_{\theta_{n-1}}g(\theta_{n-1},\xi_{n-1})\big\|^{2}\Big|\mathscr{F}_{n-2}\Big)-\big\|\nabla_{\theta_{n-1}}g(\theta_{n-1},\xi_{n-1})\big\|^{2}\Big).
	\end{aligned}\end{equation}
Thus, we have
	\begin{equation}\label{i2}\begin{aligned}
	&\frac{\big\|\nabla_{\theta_{n}}g(\theta_{n})\big\|^{2}}{S_{n}^{\epsilon_{0}}}-\frac{\big\|\nabla_{\theta_{n-1}}g(\theta_{n-1})\big\|^{2}}{S_{n-1}^{\epsilon_{0}}}\le\frac{\zeta}{S_{n-2}^{\frac{1}{2}+\frac{1}{3}\epsilon_{0}}}-K_{n-1},
	\end{aligned}\end{equation}
	where
% 	$K_{n}$ and $\zeta$ is defined as follow
	\begin{equation}\nonumber\begin{aligned}
	&\zeta=\alpha_0 c t \xi(M+1+\alpha_0 c)+\alpha_0 c(a+\alpha_0 c+1)\\&
	K_{n-1}=\frac{\alpha_0 c}{S_{n-2}^{\frac{1}{2}+\epsilon_{0}}}\Big(\Expect\Big(\big\|\nabla_{\theta_{n-1}}g(\theta_{n-1},\xi_{n-1})\big\|^{2}\Big|\mathscr{F}_{n-2}\Big)-\big\|\nabla_{\theta_{n-1}}g(\theta_{n-1},\xi_{n-1})\big\|^{2}\Big)\\&+\frac{\alpha_0 c}{S_{n-2}^{1+\epsilon_{0}}}\Big(\Expect\Big(\big\|\nabla_{\theta_{n-1}}g(\theta_{n-1},\xi_{n-1})\big\|^{2}\Big|\mathscr{F}_{n-2}\Big)-\big\|\nabla_{\theta_{n-1}}g(\theta_{n-1},\xi_{n-1})\big\|^{2}\Big).
	\end{aligned}\end{equation}
	It follows that
	\begin{equation}\nonumber\begin{aligned}
	\Expect\bigg(\frac{1}{S_{n-2}^{\frac{1}{2}+\frac{1}{3}\epsilon_{0}}}\bigg)\ge\frac{1}{a}\Expect\Bigg(\frac{\big\|\nabla_{\theta_{n-1}}g(\theta_{n-1},\xi_{n-1})\big\|^{2}-M\big\|\nabla_{\theta_{n-1}}g(\theta_{n-1})\big\|^{2}}{S_{n-2}^{\frac{1}{2}+\frac{1}{3}\epsilon_{0}}}\Bigg)
	\end{aligned}\end{equation}
	By Lemma \ref{lem5}, we have that	
	\begin{equation}\nonumber\begin{aligned}
	&\sum_{n=4}^{+\infty}\frac{1}{a}\Expect\Bigg(\frac{\big\|\nabla_{\theta_{n-1}}g(\theta_{n-1},\xi_{n-1})\big\|^{2}}{S_{n-2}^{\frac{1}{2}+\frac{1}{3}\epsilon_{0}}}\Bigg)\\&>\lim_{n\rightarrow+\infty}\frac{1}{a}\Expect\bigg(\int_{S_{2}}^{S_{n-1}}\frac{1}{x^{\frac{1}{2}+\frac{1}{3}\epsilon_{0}}}dx\bigg)=\lim_{n\rightarrow+\infty}\frac{6}{2-3\epsilon_{0}}\frac{1}{a}\Expect\Big(S_{n-1}^{\frac{1}{2}-\frac{1}{3}\epsilon_{0}}-S_{2}^{\frac{1}{2}-\frac{1}{3}\epsilon_{0}}\Big).
	\end{aligned}\end{equation}
	Due to $S_{n-1}\rightarrow +\infty \ \ a.s.$, it follows that
	\begin{equation}\label{i3_a}\begin{aligned}
	\sum_{n=4}^{+\infty}\frac{1}{a}\Expect\Bigg(\frac{\big\|\nabla_{\theta_{n-1}}g(\theta_{n-1},\xi_{n-1})\big\|^{2}}{S_{n-2}^{\frac{1}{2}+\frac{1}{3}\epsilon_{0}}}\Bigg)=+\infty.	
	\end{aligned}\end{equation}
From Lemma \ref{lem8}, we get 
	\begin{equation}\label{i4}\begin{aligned}
	\sum_{n=4}^{+\infty}\frac{1}{a}\Expect\Bigg(\frac{\big\|\nabla_{\theta_{n-1}}g(\theta_{n-1})\big\|^{2}}{S_{n-2}^{\frac{1}{2}+\frac{1}{3}\epsilon_{0}}}\Bigg)<+\infty.	
	\end{aligned}\end{equation}
	Combine \eqref{i3_a} and \eqref{i4}, then  
	\begin{equation}\label{i5}\begin{aligned}
	&\sum_{n=4}^{+\infty}E\Bigg(\frac{1}{S_{n-2}^{\frac{1}{2}+\frac{1}{3}\epsilon_{0}}}\Bigg)>\sum_{n=4}^{+\infty}\frac{1}{a}\Expect\Bigg(\frac{\big\|\nabla_{\theta_{n-1}}g(\theta_{n-1},\xi_{n-1})\big\|^{2}}{S_{n-2}^{\frac{1}{2}+\frac{1}{3}\epsilon_{0}}}\Bigg)\\&-\sum_{n=4}^{+\infty}\frac{1}{a}\Expect\Bigg(\frac{M\big\|\nabla_{\theta_{n-1}}g(\theta_{n-1})\big\|^{2}}{S_{n-2}^{\frac{1}{2}+\frac{1}{3}\epsilon_{0}}}\Bigg)=+\infty
	\end{aligned}\end{equation}
	is divergent a.s. From Lemma \ref{lem8}, we get that $\sum_{n=4}^{+\infty}\Expect\big(\|K_{n-1}\|^{2}\big)<+\infty$  and $K_{n-1}$ is a martingale difference sequence. Thus, $\sum_{n=2}^{+\infty}K_{n-1}$ is convergent a.s.. Combine \eqref{i2}, \eqref{i5} and Lemma \ref{lem4}, then we have 
	\begin{equation}\label{i3}\begin{aligned}
	\frac{\big\|\nabla_{\theta_{n}}g(\theta_{n})\big\|^{2}}{S_{n-1}^{\epsilon_{0}}}\rightarrow 0 \ \ a.s..
	\end{aligned}\end{equation}

	\subsection{The proof of Theorem \ref{thm_converg_adagrad}}\label{appen_thm_converg_adagrad}
%	Given $A_{n}^{(\epsilon)}, B_{n}^{(\epsilon)} $ in \eqref{def_A_B}, set $\epsilon=0$, then let   $A_n=A_n^{(0)}$ and $B_n=A_n^{(0)}$. 
	We consider the proof under two conditions, namely, $S_{n}<+\infty\ \ a.s.$ or $S_{n}=+\infty\ \ a.s.$.
	
	First, if   $S_{n}<+\infty\ \ a.s.$, from Lemma \ref{lem8}, we get that $\forall \epsilon\in (0,\frac{1}{2})$,  it holds that
	\begin{equation}\nonumber\begin{aligned}
		&\sum_{k=3}^{n}\frac{\big\|\nabla_{\theta_{k}}g(\theta_{k})\big\|^{2}}{S_{k-1}^{\frac{1}{2}+\epsilon}}<+\infty\ \ a.s.
	\end{aligned}\end{equation}
	Thus, we   conclude that
	\begin{equation}\nonumber\begin{aligned}
		&\frac{\big\|\nabla_{\theta_{n}}g(\theta_{n})\big\|^{2}}{S_{n-1}^{\frac{1}{2}+\epsilon}}\rightarrow 0\ \ a.s..
	\end{aligned}\end{equation} Due to $S_{n-1}<+\infty\ \ a.s.$, we   get
	\begin{equation}\label{iuewqrft}\begin{aligned}
		&\big\|\nabla_{\theta_{n}}g(\theta_{n})\big\|^{2}\rightarrow 0\ \ a.s.,
	\end{aligned}\end{equation}
	Second, if   $S_{n}=+\infty \ \ a.s.$, let $\epsilon\rightarrow 0$ on \eqref{004}, then it holds that 
	\begin{equation}\label{k0}\begin{aligned}
	&g(\theta_{n+1})-g(\theta_{n})\\
	&\le\frac{\alpha_0}{20}\Bigg(\frac{\big\|\nabla_{\theta_{n-1}}g(\theta_{n-1})\big\|^{2}}{\sqrt{S_{n-2}}}-\frac{\big\|\nabla_{\theta_{n}}g(\theta_{n})\big\|^{2}}{\sqrt{S_{n-1}}}\Bigg)\\
	&+\frac{\alpha_0}{2}(M+1)\Bigg(\frac{\big\|\nabla_{\theta_{n-1}}g(\theta_{n-1})\big\|^{2}}{\sqrt{S_{n-1}}}-\frac{\big\|\nabla_{\theta_{n}}g(\theta_{n})\big\|^{2}}{\sqrt{S_{n}}}\Bigg)\\
	&-\frac{\alpha_0}{20}\frac{\big\|\nabla_{\theta_{n}}g(\theta_{n})\big\|^{2}}{\sqrt{S_{n-1}}}+4M^{2}\alpha_0^{3}c^{2}\frac{\big\|\nabla_{\theta_{n-1}}g(\theta_{n-1}),\xi_{n-1}\big\|^{2}}{S_{n-1}^{\frac{3}{2}}}+c\alpha_{0}^{2}\frac{\big\|\nabla_{\theta_{n}}g(\theta_{n},\xi_{n})\big\|^{2}}{S_{n}}\\
	&+\frac{\alpha_0^{3}c^{2}(M+1)^{2}}{2}\frac{\big\|\nabla_{\theta_{n-1}}g(\theta_{n-1},\xi_{n-1})\big\|^{2}}{S_{n-1}}+\frac{(M+1)\alpha_0^{3}c^{2}}{2S_{n-1}^{\frac{3}{2}}}\big\|\nabla_{\theta_{n-1}}g(\theta_{n-1},\xi_{n-1})\big\|^{2}\\
	&+\frac{\alpha_0 a(M+1)}{2}\Bigg(\frac{1}{\sqrt{S_{n-1}}}-\frac{1}{\sqrt{S_{n}}}\Bigg) +I_{n}^{\le a}A_{n}+I_{n}^{\le a}B_{n}+I_{n}^{>a}X_{n}+I_{n}^{>a}Y_{n},
	\end{aligned}\end{equation}
where
	\begin{equation}\nonumber\begin{aligned}
	&I_{n}^{\le a}A_{n}+I_{n}^{\le a}B_{n}+I_{n}^{>a}X_{n}+I_{n}^{>a}Y_{n}=\lim_{\epsilon\rightarrow 0}\Big(I_{n}^{\le a}A_{n}^{(\epsilon)}+I_{n}^{\le a}B_{n}^{(\epsilon)}+I_{n}^{>a}X_{n}^{(\epsilon)}+I_{n}^{>a}Y_{n}^{(\epsilon)}\Big)\\&=\frac{\alpha_0 I_{n}^{\le a}}{\sqrt{S_{n-1}}}\nabla_{\theta_{n}}g(\theta_{n})^{T}\big(\nabla_{\theta_{n}}g(\theta_{n})-\nabla_{\theta_{n}}g(\theta_{n},\xi_{n})\big)\\&+\frac{\alpha_0 \big\|\nabla_{\theta_{n}}g(\theta_{n})\big\|^{2}I_{n}^{\le a}}{2a(M+1)\sqrt{S_{n-1}}}\bigg(\big\|\nabla_{\theta_{n}}g(\theta_{n},\xi_{n})\big\|^{2}-\Expect\Big(\big\|\nabla_{\theta_{n}}g(\theta_{n})\big\|^{2}\Big|\mathscr{F}_{n-1}\Big)\bigg)\\&+\frac{\alpha_0}{2(M+1)}\frac{I_{n}^{>a}}{\sqrt{S_{n-1}}}\bigg(\Expect\Big(\big\|\nabla_{\theta_{n}}g(\theta_{n},\xi_{n})\big\|^{2}\Big|\mathscr{F}_{n-1}\Big)-\big\|\nabla_{\theta_{n}}g(\theta_{n},\xi_{n})\big\|^{2}\bigg)\\&+\frac{\alpha_0}{2}\frac{I_{n}^{>a}}{\sqrt{S_{n-1}}}\nabla_{\theta_{n}}g(\theta_{n})^{T}\big(\nabla_{\theta_{n}}g(\theta_{n})-\nabla_{\theta_{n}}g(\theta_{n},\xi_{n})\big).
	\end{aligned}\end{equation}
	Make some transformations on $\frac{\big\|\nabla_{\theta_{n}}g(\theta_{n}),\xi_{n}\big\|^{2}}{S_{n}}$ to obtain that
	\begin{equation}\nonumber\begin{aligned}
	&\frac{\big\|\nabla_{\theta_{n}}g(\theta_{n},\xi_{n})\big\|^{2}}{S_{n}}\le\frac{\big\|\nabla_{\theta_{n}}g(\theta_{n},\xi_{n})\big\|^{2}}{S_{n-1}}\\&=\frac{\Expect\Big(\big\|\nabla_{\theta_{n}}g(\theta_{n}),\xi_{n}\big\|^{2}\Big|\mathscr{F}_{n-1}\Big)}{S_{n-1}}+\frac{1}{S_{n-1}}\Big(\big\|\nabla_{\theta_{n}}g(\theta_{n},\xi_{n})\big\|^{2}-\Expect\Big(\big\|\nabla_{\theta_{n}}g(\theta_{n}),\xi_{n}\big\|^{2}\Big|\mathscr{F}_{n-1}\Big)\Big)\\&\le\frac{M\big\|\nabla_{\theta_{n}}g(\theta_{n})\big\|^{2}}{S_{n-1}}+\frac{a}{S_{n-1}}+\frac{1}{S_{n-1}}\Big(\big\|\nabla_{\theta_{n}}g(\theta_{n},\xi_{n})\big\|^{2}-\Expect\Big(\big\|\nabla_{\theta_{n}}g(\theta_{n}),\xi_{n}\big\|^{2}\Big|\mathscr{F}_{n-1}\Big)\Big),
	\end{aligned}\end{equation}
	where the last inequality is from Assumption~\ref{ass_g_poi} 5).
	Let $0<\epsilon'<\frac{3}{8}$. It follows from Lemma \ref{lem10} that $\frac{\big\|\nabla_{\theta_{n}}g(\theta_{n})\big\|^{2}}{S_{n-1}^{\epsilon'}}<\delta<+\infty$. Then we have  $\frac{M\big\|\nabla_{\theta_{n}}g(\theta_{n})\big\|^{2}}{S_{n-1}}\le\frac{M\delta}{S_{n-1}^{1-\epsilon'}}=\frac{M\delta}{S_{n-1}^{\frac{1}{2}+\epsilon_{1}}}\ \ \big(\epsilon_{1}=\frac{1}{2}-\epsilon'\in(0,\frac{1}{4})\big)$ and
	\begin{equation}\label{k1}\begin{aligned}
	&\frac{\big\|\nabla_{\theta_{n}}g(\theta_{n},\xi_{n})\big\|^{2}}{S_{n}}\le\frac{M\delta+a}{S_{n}^{\frac{1}{2}+\epsilon_{1}}}+\frac{1}{S_{n-1}}\Big(\big\|\nabla_{\theta_{n}}g(\theta_{n},\xi_{n})\big\|^{2}-\Expect\Big(\big\|\nabla_{\theta_{n}}g(\theta_{n}),\xi_{n}\big\|^{2}\Big|\mathscr{F}_{n-1}\Big)\Big)\\&+(M\delta+a)\Bigg(\frac{1}{S_{n-1}^{\frac{1}{2}+\epsilon_{1}}}-\frac{1}{S_{n}^{\frac{1}{2}+\epsilon_{1}}}\Bigg).
	\end{aligned}\end{equation}
	Similarly, we have
	\begin{equation}\label{k4}\begin{aligned}
	&\frac{\big\|\nabla_{\theta_{n-1}}g(\theta_{n-1},\xi_{n-1})\big\|^{2}}{S_{n-1}}\le \frac{M\delta+a}{S_{n}^{\frac{1}{2}+\epsilon_{1}}}+(M\delta+a)\Bigg(\frac{1}{S_{n-2}^{\frac{1}{2}+\epsilon_{1}}}-\frac{1}{S_{n}^{\frac{1}{2}+\epsilon_{1}}}\Bigg)\\&+\frac{1}{S_{n-2}}\Big(\big\|\nabla_{\theta_{n-1}}g(\theta_{n-1},\xi_{n-1})\big\|^{2}-\Expect\Big(\big\|\nabla_{\theta_{n-1}}g(\theta_{n-1}),\xi_{n-1}\big\|^{2}\Big|\mathscr{F}_{n-2}\Big)\Big).
	\end{aligned}\end{equation}
	Substitute \eqref{k1} and \eqref{k4} into \eqref{k0}, then we get
	\begin{equation}\label{k3}\begin{aligned}
	&g(\theta_{n+1})-g(\theta_{n})\le\bigg(c\alpha_0^{2}+\frac{c^{3}\alpha_0^{2}(M+1)}{2}\bigg)(M\delta+a)\frac{1}{S_{n}^{\frac{1}{2}+\epsilon_{1}}}+P_{n}+Q_{n},
	\end{aligned}\end{equation}
% 	$Q_{n}$ is defined as follow 
where
	\begin{equation}\nonumber\begin{aligned}
	Q_{n}=&\frac{c\alpha_{0}^{2}}{S_{n-1}}\Big(\big\|\nabla_{\theta_{n}}g(\theta_{n},\xi_{n})\big\|^{2}-\Expect\Big(\big\|\nabla_{\theta_{n}}g(\theta_{n}),\xi_{n}\big\|^{2}\Big|\mathscr{F}_{n-1}\Big)\Big)\\&+\frac{\alpha_0^{3}c^{2}(M+1)^{2}}{2S_{n-2}}\Big(\big\|\nabla_{\theta_{n-1}}g(\theta_{n-1},\xi_{n-1})\big\|^{2}-\Expect\Big(\big\|\nabla_{\theta_{n-1}}g(\theta_{n-1}),\xi_{n-1}\big\|^{2}\Big|\mathscr{F}_{n-2}\Big)\Big)\\&+I_{n}^{\le a}A_{n}+I_{n}^{\le a}B_{n}+I_{n}^{>a}X_{n}+I_{n}^{>a}Y_{n}.
	\end{aligned}\end{equation}
% 	and $P_{n}$ is defined as follow 
	\begin{equation}\nonumber\begin{aligned}
	P_{n}=&\frac{\alpha_0}{20}\Bigg(\frac{\big\|\nabla_{\theta_{n-1}}g(\theta_{n-1})\big\|^{2}}{\sqrt{S_{n-2}}}-\frac{\big\|\nabla_{\theta_{n}}g(\theta_{n})\big\|^{2}}{\sqrt{S_{n-1}}}\Bigg)+\frac{\alpha_0}{2}(M+1)\Bigg(\frac{\big\|\nabla_{\theta_{n-1}}g(\theta_{n-1})\big\|^{2}}{\sqrt{S_{n-1}}}-\frac{\big\|\nabla_{\theta_{n}}g(\theta_{n})\big\|^{2}}{\sqrt{S_{n}}}\Bigg)\\&+4M^{2}\alpha_0^{3}c^{2}\frac{\big\|\nabla_{\theta_{n-1}}g(\theta_{n-1}),\xi_{n-1}\big\|^{2}}{S_{n-1}^{\frac{3}{2}}}+\frac{(M+1)\alpha_0^{3}c^{2}}{2S_{n-1}^{\frac{3}{2}}}\big\|\nabla_{\theta_{n-1}}g(\theta_{n-1},\xi_{n-1})\big\|^{2}\\&+\frac{\alpha_0 a(M+1)}{2}\Bigg(\frac{1}{\sqrt{S_{n-1}}}-\frac{1}{\sqrt{S_{n}}}\Bigg)+\frac{\alpha_0^{3}c^{2}(M+1)^{2}(M\delta+a)}{2}\Bigg(\frac{1}{S_{n-1}^{\frac{1}{2}+\epsilon_{1}}}-\frac{1}{S_{n}^{\frac{1}{2}+\epsilon_{1}}}\Bigg)\\&+c\alpha_{0}^{2}(M\delta+a)\Bigg(\frac{1}{S_{n-2}^{\frac{1}{2}+\epsilon_{1}}}-\frac{1}{S_{n}^{\frac{1}{2}+\epsilon_{1}}}\Bigg).
	\end{aligned}\end{equation}
It follows from   $S_{n}\rightarrow +\infty \ \ a.s.$  that 
	\begin{equation}\nonumber\begin{aligned}
	\sum_{n=3}^{+\infty}P_{n}=&\frac{\alpha_0}{20}\sum_{n=3}^{+\infty}\Bigg(\frac{\big\|\nabla_{\theta_{n-1}}g(\theta_{n-1})\big\|^{2}}{\sqrt{S_{n-2}}}-\frac{\big\|\nabla_{\theta_{n}}g(\theta_{n})\big\|^{2}}{\sqrt{S_{n-1}}}\Bigg)\\&+\frac{\alpha_0}{2}(M+1)\sum_{n=3}^{+\infty}\Bigg(\frac{\big\|\nabla_{\theta_{n-1}}g(\theta_{n-1})\big\|^{2}}{\sqrt{S_{n-1}}}-\frac{\big\|\nabla_{\theta_{n}}g(\theta_{n})\big\|^{2}}{\sqrt{S_{n}}}\Bigg)\\&+4M^{2}\alpha_0^{3}c^{2}\sum_{n=3}^{+\infty}\frac{\big\|\nabla_{\theta_{n-1}}g(\theta_{n-1}),\xi_{n-1}\big\|^{2}}{S_{n-1}^{\frac{3}{2}}}+\frac{(M+1)\alpha_0^{3}c^{2}}{2S_{n-1}^{\frac{3}{2}}}\sum_{n=3}^{+\infty}\big\|\nabla_{\theta_{n-1}}g(\theta_{n-1},\xi_{n-1})\big\|^{2}\\&+\sum_{n=3}^{+\infty}\frac{\alpha_0 a(M+1)}{2}\Bigg(\frac{1}{\sqrt{S_{n-1}}}-\frac{1}{\sqrt{S_{n}}}\Bigg)+\sum_{n=3}^{+\infty}\frac{\alpha_0^{3}c^{2}(M+1)^{2}(M\delta+a)}{2}\Bigg(\frac{1}{S_{n-1}^{\frac{1}{2}+\epsilon_{1}}}-\frac{1}{S_{n}^{\frac{1}{2}+\epsilon_{1}}}\Bigg)\\&+\sum_{n=3}^{+\infty}c\alpha_{0}^{2}(M\delta+a)\Bigg(\frac{1}{S_{n-2}^{\frac{1}{2}+\epsilon_{1}}}-\frac{1}{S_{n}^{\frac{1}{2}+\epsilon_{1}}}\Bigg)\\
	=&\frac{\alpha_{0}}{20}\Bigg(\frac{\big\|\nabla_{\theta_{2}}g(\theta_{2})\big\|^{2}}{\sqrt{S_{1}}}-\lim_{n\rightarrow+\infty}\frac{\big\|\nabla_{\theta_{n}}g(\theta_{n})\big\|^{2}}{\sqrt{S_{n-1}}}\Bigg)\\&+\frac{\alpha_0}{2}(M+1)\Bigg(\frac{\big\|\nabla_{\theta_{2}}g(\theta_{2})\big\|^{2}}{\sqrt{S_{2}}}-\lim_{n\rightarrow+\infty}\frac{\big\|\nabla_{\theta_{n}}g(\theta_{n})\big\|^{2}}{\sqrt{S_{n}}}\Bigg)+\frac{\alpha_0 a(M+1)}{2}\Bigg(\frac{1}{\sqrt{S_{2}}}-\lim_{n\rightarrow+\infty}\frac{1}{\sqrt{S_{n}}}\Bigg)\\&+4M^{2}\alpha_0^{3}c^{2}\sum_{n=2}^{+\infty}\frac{\big\|\nabla_{\theta_{n-1}}g(\theta_{n-1}),\xi_{n-1}\big\|^{2}}{S_{n-1}^{\frac{3}{2}}}+\frac{(M+1)\alpha_0^{3}c^{2}}{2S_{n-1}^{\frac{3}{2}}}\sum_{n=2}^{+\infty}\big\|\nabla_{\theta_{n-1}}g(\theta_{n-1},\xi_{n-1})\big\|^{2}\\&+\frac{\alpha_0 a(M+1)}{2}\frac{1}{\sqrt{S_{2}}}+\frac{M\delta+a}{S_{2}^{\frac{1}{2}+\epsilon_{1}}}+\frac{\alpha_0^{3}c^{2}(M+1)^{2}(M\delta+a)}{2S_{1}^{\frac{1}{2}+\epsilon_{1}}}.
	\end{aligned}\end{equation}
	From Lemma \ref{lem10}, we get that
	\begin{equation}\nonumber\begin{aligned}
	&\lim_{n\rightarrow+\infty}\frac{\big\|\nabla_{\theta_{n}}g(\theta_{n})\big\|^{2}}{\sqrt{S_{n-1}}}\le\lim_{n\rightarrow+\infty}\frac{\big\|\nabla_{\theta_{n}}g(\theta_{n})\big\|^{2}}{S_{n-1}^{\epsilon_{0}}}=0,
	\end{aligned}\end{equation}
	and
	\begin{equation}\nonumber\begin{aligned}
	&\lim_{n\rightarrow+\infty}\frac{\big\|\nabla_{\theta_{n}}g(\theta_{n})\big\|^{2}}{\sqrt{S_{n}}}\le\lim_{n\rightarrow+\infty}\frac{\big\|\nabla_{\theta_{n}}g(\theta_{n})\big\|^{2}}{\sqrt{S_{n-1}}}\le\lim_{n\rightarrow+\infty}\frac{\big\|\nabla_{\theta_{n}}g(\theta_{n})\big\|^{2}}{S_{n-1}^{\epsilon_{0}}}=0.
	\end{aligned}\end{equation}
	It follows that
	\begin{equation}\label{k6}\begin{aligned}
	&\sum_{n=3}^{+\infty}P_{n}=\frac{\alpha_0}{20}\frac{\big\|\nabla_{\theta_{1}}g(\theta_{1})\big\|^{2}}{\sqrt{S_{0}}}+\frac{\alpha_0}{2}(M+1)\frac{\big\|\nabla_{\theta_{1}}g(\theta_{1})\big\|^{2}}{\sqrt{S_{1}}}+\frac{\alpha_0 a(M+1)}{2}\frac{1}{\sqrt{S_{1}}}\\&+4M^{2}\alpha_0^{3}c^{2}\sum_{n=2}^{+\infty}\frac{\big\|\nabla_{\theta_{n-1}}g(\theta_{n-1}),\xi_{n-1}\big\|^{2}}{S_{n-1}^{\frac{3}{2}}}+\frac{(M+1)\alpha_0^{3}c^{2}}{2S_{n-1}^{\frac{3}{2}}}\sum_{n=2}^{+\infty}\big\|\nabla_{\theta_{n-1}}g(\theta_{n-1},\xi_{n-1})\big\|^{2}\\&+\frac{\alpha_0 a(M+1)}{2}\frac{1}{\sqrt{S_{2}}}+\frac{M\delta+a}{S_{2}^{\frac{1}{2}+\epsilon_{1}}}+\frac{\alpha_0^{3}c^{2}(M+1)^{2}(M\delta+a)}{2S_{1}^{\frac{1}{2}+\epsilon_{1}}}.
	\end{aligned}\end{equation}
	By Lemma \ref{lem5}, we have
	\begin{equation}\label{k7}\begin{aligned}
	4M^{2}\alpha_0^{3}c^{2}\sum_{n=3}^{+\infty}\frac{\big\|\nabla_{\theta_{n-1}}g(\theta_{n-1}),\xi_{n-1}\big\|^{2}}{S_{n-1}^{\frac{3}{2}}}+\sum_{n=3}^{+\infty}\frac{(M+1)\alpha_0^{3}c^{2}}{2S_{n-1}^{\frac{3}{2}}}\big\|\nabla_{\theta_{n-1}}g(\theta_{n-1},\xi_{n-1})\big\|^{2}<+\infty.
	\end{aligned}\end{equation}
	Thus, $\sum_{n=2}^{+\infty}P_{n}$ is convergent. In addition, it holds that
	\begin{equation}\nonumber\begin{aligned}
	&\sum_{n=3}^{+\infty}\Expect\Big(\big\|I_{n}^{\le a}A_{n}\big\|^{2}\Big|\mathscr{F}_{n-1}\Big)\\&\le\sum_{n=3}^{+\infty}\Expect\Big(\big\|A_{n}\big\|^{2}\Big|\mathscr{F}_{n-1}\Big)\le\sum_{n=3}^{+\infty}\Expect\Bigg(\bigg\|\frac{\alpha_0}{\sqrt{S_{n-1}}}\nabla_{\theta_{n}}g(\theta_{n})^{T}\big(\nabla_{\theta_{n}}g(\theta_{n})-\nabla_{\theta_{n}}g(\theta_{n},\xi_{n})\big)\bigg\|^{2}\bigg|\mathscr{F}_{n-1}\Bigg)\\&\le 2\sum_{n=3}^{+\infty}\frac{\alpha_0^{2}\big\|\nabla_{\theta_{n}}{g(\theta_{n})}\big\|^{4}}{S_{n-1}}+2\sum_{n=3}^{+\infty}\Expect\Bigg(\frac{\alpha_0^{2}}{S_{n-1}}\big\|\nabla_{\theta_{n}}g(\theta_{n})\big\|^{2}\big\|\nabla_{\theta_{n}}g(\theta_{n},\xi_{n})\big\|^{2}\Big|\mathscr{F}_{n-1}\Bigg).
	\end{aligned}\end{equation}
	It follows from Lemma \ref{lem10} that $\exists 0<\epsilon_{0}<\frac{3}{8}$, such that
	\begin{equation}\nonumber\begin{aligned}
	\lim_{n\rightarrow+\infty}\frac{\big\|\nabla_{\theta_{n}}g(\theta_{n})\big\|^{2}}{S_{n-1}^{\epsilon_{0}}}=0\ \ a.s.,
	\end{aligned}\end{equation}
meaning that there is an almost surely bounded random variable $\delta_{0}$, such that 
	\begin{equation}\nonumber\begin{aligned}
	\frac{\big\|\nabla_{\theta_{n}}g(\theta_{n})\big\|^{2}}{S_{n-1}^{\epsilon_{0}}}<\delta_{0}<\infty\ \ a.s.
	\end{aligned}\end{equation}
It follows that
	\begin{equation}\nonumber\begin{aligned}
	\frac{\big\|\nabla_{\theta_{n}}g(\theta_{n})\big\|^{4}}{S_{n-1}}<\delta_{0}\frac{\big\|\nabla_{\theta_{n}}g(\theta_{n})\big\|^{2}}{S_{n-1}^{\frac{1}{2}+(\frac{1}{2}-\epsilon_{0})}}\ \ a.s.,
	\end{aligned}\end{equation}
Then we derive that
	\begin{equation}\label{k7}\begin{aligned}
	&\sum_{n=3}^{+\infty}\Expect\Big(\big\|I_{n}^{\le a}A_{n}\big\|^{2}\Big|\mathscr{F}_{n-1}\Big)\\&\le 2\sum_{n=3}^{+\infty}\frac{\alpha_0^{2}\big\|\nabla_{\theta_{n}}{g(\theta_{n})}\big\|^{4}}{S_{n-1}}+2\sum_{n=3}^{+\infty}\Expect\Bigg(\frac{\alpha_0^{2}}{S_{n-1}}\big\|\nabla_{\theta_{n}}g(\theta_{n})\big\|^{2}\big\|\nabla_{\theta_{n}}g(\theta_{n},\xi_{n})\big\|^{2}\Big|\mathscr{F}_{n-1}\Bigg)\\&\le2\sum_{n=3}^{+\infty}\frac{\alpha_0^{2}\big\|\nabla_{\theta_{n}}{g(\theta_{n})}\big\|^{4}}{S_{n-1}}+2\sum_{n=3}^{+\infty}\frac{\alpha_0^{2}}{S_{n-1}}\big\|\nabla_{\theta_{n}}g(\theta_{n})\big\|^{2}\Big(M\big\|\nabla_{\theta_{n}}g(\theta_{n})\big\|^{2}+a\Big)\\&=2(M+1)\sum_{n=3}^{+\infty}\frac{\alpha_0^{2}\big\|\nabla_{\theta_{n}}{g(\theta_{n})}\big\|^{4}}{S_{n-1}}+2\sum_{n=3}^{+\infty}\frac{\alpha_0^{2}a\big\|\nabla_{\theta_{n}}g(\theta_{n})\big\|^{2}}{S_{n-1}}\\&<2\big((M+1)\delta_{0}+a\big)\sum_{n=3}^{+\infty}\frac{\big\|\nabla_{\theta_{n}}g(\theta_{n})\big\|^{2}}{S_{n-1}^{\frac{1}{2}+(\frac{1}{2}-\epsilon_{0})}}<+\infty \ \ a.s..
	\end{aligned}\end{equation}

From Lemma \ref{lem_summation_MDS},  $\sum_{n=3}^{+\infty}I_{n}^{\le a}A_{n}$ is convergent almost surely. Similarly, we can prove other parts of $\sum_{n=3}^{+\infty}Q_{n}$ are also convergent almost surely. Thus,   $\sum_{n=3}^{+\infty}Q_{n}$ is convergent almost surely. By summary, we get
\begin{equation}\label{qw3214rt}\begin{aligned}
\sum_{n=1}^{+\infty}\big(P_{n}+Q_{n}\big) \ is \ convergent \ almost \ surely.\ \ \ \ \ \ \ \
\end{aligned}\end{equation}

It is easy to find that $J_{i}$ is a bounded closed set. So $\forall \epsilon >0$ we can construct an open cover $H^{(i)}_{\epsilon}=\{U(\theta,\epsilon)\}\ (\theta\in J_{i})$ of $J_{i}$. Through the $Heine–Borel \ theorem$, we can get a finite open subcover $\{U(\theta_{k},\epsilon)\ (k=0,1,...,n)$ from $H^{(i)}_{\epsilon}$. Then we assign 
$U^{(i)}_{\epsilon}=\bigcup_{k=0}^{n}U(\theta_{k},\epsilon)$. We can get $U^{(i)}_{\epsilon}$ is a open set. Under Assumption \ref{ass_g_poi},   $J=\{\theta|\nabla_{\theta}g(\theta)\}$ has only finite connected components $J_{1},J_{2},...,J_{m}$. So $\inf_{i\neq j}d(J_{i},J_{j})=\min_{i\neq j}d(J_{i},J_{j})$. Let $\delta_{0}=\min_{i\neq j}d(J_{i},J_{j})$.  It follows from Lemma \ref{pouikm,l} that $\exists \epsilon_{0}>0$, when $d(\theta,J_{i})<\epsilon_{0}$, there is
\begin{equation}
\begin{aligned}
&\|\nabla_{\theta}g(\theta)\|^{2}\le 2c|g(\theta)-g_{i}|,
\end{aligned}
\end{equation} 
where $g_{i}$ denotes $g(\theta)\ (\theta\in J_{i})$. Let $c=min\{\epsilon_{0},\delta_{0}/4\}$ and construct $U^{(1)}_{c},U^{(2)}_{c},...,U^{(m)}_{c}$. It is obvious that $\forall \ U^{(i)}_{c},U^{(j)_{c}}\ (i\neq j)$, $d(U^{(i)}_{c},U^{(j)})>\delta_{0}/2$, and $\|\nabla_{\theta}g(\theta)\|^{2}\le 2c|g(\theta)-g_{i}|\ (\theta\in U^{(i)}_{c})$.

Since $J$ is a bounded set,  $\exists N>0$, such that $J\subset K$ ($K$ is the closure of $U(0,N)$). Then we construct a set $M=K/\bigcup_{i=1}^{m}U^{(i)}_{c}$. Since $U^{(i)}_{c}$ is a open set and $K$ is a closed set, we conclude $M$ is a closed set. Since $\|\nabla_{\theta}g(\theta)\|$ is a continuous function,   $\exists \theta_{0}\in M$, $\|\nabla_{\theta_{0}}g(\theta_{0})\|=\min_{\theta\in M}\|\nabla_{\theta}g(\theta)\|$. Let $r=\|\nabla_{\theta_{0}}g(\theta_{0})\|>0$.

Then we prove that $\forall u>0, \theta \in K$, $\exists \delta>0$, if $\|\nabla_{\theta}g(\theta)\|<\delta$, makes $d(\theta, J)<u$. We prove it by contradiction. Assume $\exists u_{0}>0$, $\forall \delta_{1}>0$, $\exists \ \theta_{\delta_{1}}$ holds $\|\nabla_{\theta_{\delta_{1}}}g(\theta_{\delta_{1}})\|<\delta_{1}$ and $d(\theta, J)\geq u_{0}$. We make $\delta_{1}=1,1/2,1/3...$, and we form a sequence $\{\theta_{1/n}\}$. It is obvious that $\|\nabla_{\theta_{1/n}}g(\theta_{1/n})\rightarrow 0\|$. Since $\{\theta_{1/n}\}$ is bounded, through the $Accumulation\ point\ theorem$, there exists a convergent subsequence $\{\theta_{1/k_{n}}\}\subset \{\theta_{1/n}\}$. We defined $\theta^{(0)}=\lim_{n\rightarrow +\infty}\theta_{1/k_{n}}$. Through the continuity of $d(\theta,J)$ and $\|\nabla_{\theta}g(\theta)\|$, we get $d(\theta^{(0)})\geq u_{0}$ and $\|\nabla_{\theta^{(0)}}g(\theta^{(0)})\|=0$. It is contradiction by the definition of $J$. So $\forall u>0, \theta \in K$, $\exists \delta>0$, $\|\nabla_{\theta}g(\theta)\|<\delta$, makes $d(\theta, J)<u$ ($\exists \i$, makes $d(\theta, J)<u$). And furthermore, due to the continuity of $g(\theta)$, we can get $\forall \epsilon_{1}>0$, $\exists \ \delta'>0$, if $d(\theta,J_{i})<\delta'$, there is $|g(\theta)-g_{i}|<\epsilon_{1}$. So combine these two consequences. We can prove $\forall \epsilon_{1}>0$, $\exists b>0$, if $\theta\in U_{c}^{(i)}$ and $\|\nabla_{\theta}g(\theta)\|<b$, there is $\big|g(\theta)-g_{i}\big|<\epsilon_{1}$.

Through \eqref{022} and Lemma 6 we get there is a subsequence $\{\|\nabla_{\theta_{k_{n}}}g(\theta_{k_{n}})\|^{2}\}$ of $\{\|\nabla_{\theta_{n}}g(\theta_{n})\|^{2}\}$ which satisfies that
\begin{equation}\label{pgm}\begin{aligned}
&\lim_{n\rightarrow +\infty}\big\|\nabla_{\theta_{k_{n}}}g(\theta_{k_{n}})\big\|^{2}= 0\ \ a.s..\ \ \ \ \ \ \ \ \ \quad    
\end{aligned}\end{equation}
Next we aim to prove   $\lim_{n\rightarrow+\infty }\|\nabla_{\theta_{n}}g(\theta_{n})\|^{2}=0$. It is equivalent  to prove that $\{\|\nabla_{n}g(\theta_{n})\|^{2}\}$ has no positive accumulation points, that is to say, $\forall e_{0}>0$, there are only finite values of $\|\nabla_{\theta_{n}}g(\theta)\|$ larger than $e_{0}$. And obviously, we just need to prove $\forall 0<e_{0}<r$, there are only finite values of $\|\nabla_{\theta_{n}}g(\theta)\|$ larger than $e$. We prove this by contradiction. We suppose $\exists 0<e<a$, making the set $S=\{\|\nabla_{\theta_{n}}g(\theta_{n})\|^{2}\}$ be an infinite set. Then we assign $\epsilon_{1}=e
/8c$ and define $o=min\{b,e/4\}$. Due to \eqref{pgm}, we get there exists a subsequence $\{\theta_{p_{n}}\}$ of $\{\theta_{n}\}$ which satisfies $\|\nabla_{\theta_{p_{n}}}g(\theta_{p_{n}})\|<o$. We rank $S$ as a subsequence $\{\|\nabla_{m_{n}}g(\theta_{m_{n}})\|^{2}\}$ of $\{\|\nabla_{n}g(\theta_{n})\|^{2}\}$. Then there is an infinite subsequence  $\{\|\nabla_{m_{i_{n}}}g(\theta_{m_{i_{n}}})\|^{2}\}$ of $\{\|\nabla_{m_{n}}g(\theta_{m_{n}})\|^{2}\}$ such that $\forall n\in\mathbb{N}_+$, $\exists l,\  n_{p_{n}}\in(m_{i_{l}},m_{i_{l+1}})$. For convenient, we  abbreviate $\{m_{i_{n}}\}$ as $\{i_{n}\}$. And we construct another infinite sequence $\{q_{n}\}$ as follows
\begin{equation}\nonumber
\begin{aligned}
&q_{1}=\max\big\{n:p_{1}<n<\min \{m_{i_{l}:m_{i_{l}}>p_{1}}\},\big\|\nabla_{\theta_{n}}g(\theta_{n})\big\|\le o\big\},
\\&q_{2}=\min\big\{n:n>q_{1},\big\|\nabla_{\theta_{n}}g(\theta_{n})\big\|>e\big\},
\\&q_{2n-1}=\max\big\{n:\min\{m_{i_{l}}:m_{i_{l}}>q_{2n-3}\}<n<\min\{m_{l}:m_{l}>\min\{m_{i_{l}}:m_{i_{l}}>q_{2n-3}\},\\&\big\|\nabla_{\theta_{n}}g(\theta_{n})\big\|\le o\},
\\&q_{2n}=\min\big\{n:n>q_{2n-1},\big\|\nabla_{\theta_{n}}g(\theta_{n})\big\|>e\big\}.
\end{aligned}\end{equation}
Now we prove that $\exists N_{0}$, when $q_{2n}>N_{0}$, it has $e<\big\|\nabla_{\theta_{q_{2n}}}g(\theta_{q_{2n}})\big\|<r$. The left side is obvious (the definition of $q_{2n}$). And for the right side, we know $\big\|\nabla_{\theta_{q_{2n}-1}}g(\theta_{q_{2n}-1})\big\|\le e$. It follows from \eqref{AdaGrad} that
\begin{equation}\nonumber\begin{aligned}
&\|\theta_{n+1}-\theta_{n}\|^{2}=\frac{\alpha_{0}^{2}}{S_{n}}\big\|\nabla_{\theta_{n}}g(\theta_{q_{n}},\xi_{n})\big\|^{2}
\\&\le\frac{\alpha_{0}^{2}}{S_{n-1}}\Big(\big\|\nabla_{\theta_{n-1}}g(\theta_{n},\xi_{n})\big\|^{2}-\Expect\big( \big\|\nabla_{\theta_{n}}g(\theta_{n},\xi_{n})\big\|^{2}\big|\mathscr{F}_{n}\big)\Big)
\\&+\frac{\alpha_{0}^{2}}{S_{n-1}}\big(M\big\|\nabla_{\theta_{n}}g(\theta_{n})\big\|^{2}+a\big).
\end{aligned}\end{equation}
Through previous consequences we can easily find that 
\begin{equation}\nonumber
\begin{aligned}
&\sum_{n=2}^{+\infty}\Bigg(\frac{\alpha_{0}^{2}}{S_{n-1}}\Big(\big\|\nabla_{\theta_{n-1}}g(\theta_{n},\xi_{n})\big\|^{2}-\Expect\big( \big\|\nabla_{\theta_{n}}g(\theta_{n},\xi_{n})\big\|^{2}\big|\mathscr{F}_{n}\big)\Big)
+\frac{\alpha_{0}^{2}M\big\|\nabla_{\theta_{n}}g(\theta_{n})\big\|^{2}}{S_{n-1}}\Bigg)\\&<+\infty\ \ a.s..
\end{aligned}
\end{equation}
Note that $\alpha_{0}^{2}a/S_{n-1}\rightarrow 0,\quad  a.s.$. We conclude 
\begin{equation}\label{uyi}\begin{aligned}
\|\theta_{n+1}-\theta_{n}\|\rightarrow 0 \ \ a.s..
\end{aligned}\end{equation}
Through Assumption \ref{ass_g_poi} 2) we   get $\big|\|\nabla_{\theta_{n+1}}g(\theta_{n+1})\|^{2}-\|\nabla_{\theta_{n}}g(\theta_{n})\|^{2}\big|\le \big|\|\nabla_{\theta_{n+1}}g(\theta_{n+1})\|-\|\nabla_{\theta_{n}}g(\theta_{n})\|\big|^{2}\le \|\nabla_{\theta_{n+1}}g(\theta_{n+1})-\nabla_{\theta_{n}}g(\theta_{n})\|^{2}\le c\|\theta_{n+1}-\theta_{n}\|\rightarrow 0 \ \ a.s.,$ So $\exists N_{0}$, when $n>N_{0}$, there is $\big|\|\nabla_{\theta_{n+1}}g(\theta_{n+1})\|^{2}-\|\nabla_{\theta_{n}}g(\theta_{n})\|\big|<r-e$. Then we can get that when $q_{2n}>N_{0}+1$, there is $\big\|\nabla_{\theta_{q_{2n}}}g(\theta_{q_{n}})\big\|\le \|\nabla_{\theta_{q_{2n}-1}}g(\theta_{q_{2n}-1})\|+\big|\|\nabla_{\theta_{q_{2n}}}g(\theta_{q_{2n}})\|-\|\nabla_{\theta_{q_{2n}-1}}g(\theta_{q_{2n}-1})\|\big|\le e+r-e=r$. That means that $\theta_{q_{2n}}\in \bigcup_{i=1}^{m}U_{\tau}^{(i)}$, so we can prove $\exists i_{0}$, such that $\theta_{q_{2n}}\in U_{\tau}^{(i_{0})}$. And due to $\|\nabla_{\theta_{n}}g(\theta_{n})\|\le e<r\ \ (n\in[q_{2n-1},q_{2n}))$, we get $\forall k \in [q_{2n-1},q_{2n})$, $\exists i_{k}$, such that $\theta_{n}\in U_{\tau}^{(i_{k})}$. Due to $\|\theta_{n+1}-\theta_{n}\|\rightarrow 0 a.s.$, we know $i_{0}=i_{k} \ \ (\forall \ j\in [q_{2n-1},q_{2n}))$. For convenient, we sign $i_{0}=i_{q_{2n-1}}=...=i_{q_{2n}-1}=i_{q_{2n}}$. And then we can conclude that
\begin{equation}\nonumber\begin{aligned}
&\|\nabla_{\theta}g(\theta_{n})\|^{2}\le 2c|g(\theta_{n})-g_{i_{q_{2n}}}|\ \ (n\in [q_{2n-1},q_{2n}]).
\end{aligned}\end{equation}Due to locally sign-preserving property, we get
\begin{equation}\nonumber\begin{aligned}
&\|\nabla_{\theta}g(\theta_{n})\|^{2}\le 2c(g(\theta_{n})-g_{i_{q_{2n}}})\ \ (g(\theta_{n})\geq g_{i_{q_{2n}}})\ \ or\\& \|\nabla_{\theta}g(\theta_{n})\|^{2}\le -2c(g(\theta_{n})-g_{i_{q_{2n}}})\ \ (g(\theta_{n})\le g_{i_{q_{2n}}})\ \ (n\in [q_{2n-1},q_{2n}]).
\end{aligned}\end{equation}Ways to dispose these two cases is same, so we just show how to prove the first case. We get
\begin{equation}\nonumber\begin{aligned}
&e-o<\big\|\nabla_{\theta_{q_{2n}}}g(\theta_{q_{2n}})\big\|^{2}-\big\|\nabla_{\theta_{q_{2n-1}}}g(\theta_{q_{2n-1}})\big\|^{2}<2c\big(g(\theta_{q_{2n}})-g_{i_{q_{2n}}}\big)-\big\|\nabla_{\theta_{q_{2n-1}}}g(\theta_{q_{2n-1}})\big\|^{2}\\&=\bigg(2c\sum_{i=0}^{q_{2n}-q_{2n-1}-1}g(\theta_{q_{2n-1}+i+1})-g(\theta_{q_{2n-1}+i})\bigg)+2c\Big(g(\theta_{q_{2n-1}})-g_{i_{q_{2n}}}\Big)-\big\|\nabla_{\theta_{q_{2n-1}}}g(\theta_{q_{2n-1}})\big\|^{2}.
\end{aligned}\end{equation}From \eqref{k3}, we obtain
\begin{equation}\nonumber\begin{aligned}
&g(\theta_{q_{2n-1}+i+1})-g(\theta_{q_{2n-1}+i})\\&\le\bigg(c\alpha_0^{2}+\frac{c^{3}\alpha_0^{2}(M+1)}{2}\bigg)(M\delta+a)\frac{1}{S_{q_{2n-1}+i}^{\frac{1}{2}+\epsilon_{1}}}+P_{q_{2n-1}+i}+Q_{q_{2n-1}+i}.
\end{aligned}\end{equation}So there is
\begin{equation}\label{zzz}\begin{aligned}
&e-o<\sum_{i=0}^{q_{2n}-q_{2n-1}-1}\frac{L}{S_{q_{2n-1}+i}^{\frac{1}{2}+\epsilon_{1}}}+\sum_{i=0}^{q_{2n}-q_{2n-1}-1}\Big(P_{q_{2n-1}+i}+Q_{q_{2n-1}+i}\Big)\\&+2c\Big(g(\theta_{q_{2n-1}})-g_{i_{2n-1}}\Big)-\big\|\nabla_{\theta_{q_{2n-1}}}g(\theta_{q_{2n-1}})\big\|^{2},\ \ \ \ \ \ \ \ 
\end{aligned}\end{equation}which 
\begin{equation}\nonumber\begin{aligned}
L=2c\bigg(c\alpha_0^{2}+\frac{c^{3}\alpha_0^{2}(M+1)}{2}\bigg)(M\delta+a).
\end{aligned}\end{equation}Due to $\|\nabla_{\theta_{q_{2n-1}}}g(\theta_{q_{2n-1}})\|^{2}<o<b$, so through \eqref{qw3214rt}, we get that $g(\theta_{q_{2n-1}})-g_{i_{2n-1}}<e/8c$. Substitute it into \eqref{zzz}. We get
\begin{equation}\label{zzz}\begin{aligned}
&\sum_{i=0}^{q_{2n}-q_{2n-1}-1}\frac{1}{S_{q_{2n-1}+i}^{\frac{1}{2}+\epsilon_{1}}}>\frac{e}{2L}-\sum_{i=0}^{q_{2n}-q_{2n-1}-1}\Big(P_{q_{2n-1}+i}+Q_{q_{2n-1}+i}\Big).
\end{aligned}\end{equation}Through \eqref{qw3214rt}, we know $\sum_{n=1}^{+\infty}(P_{n}+Q_{n})$ is convergence almost surely. So we get that $\sum_{i=0}^{q_{2n}-q_{2n-1}-1}\Big(P_{q_{2n-1}+i}+Q_{q_{2n-1}+i}\Big)\rightarrow 0\ a.s.$ by the $Cauchy's \ test\  for \ convergence$. Combining $1/S_{q_{2n-1}+i}^{\frac{1}{2}+\epsilon_{1}}\rightarrow 0\ \  a.s.$, we get 
\begin{equation}\begin{aligned}
&\sum_{i=1}^{q_{2n}-q_{2n-1}-1}\frac{1}{S_{q_{2n-1}+i}^{\frac{1}{2}+\epsilon_{1}}}\\&>\frac{e}{2L}-\frac{1}{S_{q_{2n-1}}^{\frac{1}{2}+\epsilon_{1}}}-\sum_{i=0}^{q_{2n}-q_{2n-1}-1}\Big(P_{q_{2n-1}+i}+Q_{q_{2n-1}+i}\Big)\rightarrow \frac{e}{2L} \ \ a.s.,
\end{aligned}\end{equation}so there is
\begin{equation}\label{00ooii9}\begin{aligned}
&\sum_{n=1}^{+\infty}\Bigg(\sum_{i=1}^{q_{2n}-q_{2n-1}-1}\frac{1}{S_{q_{2n-1}+i}^{\frac{1}{2}+\epsilon_{1}}}\Bigg)=+\infty \ \ a.s..\ \ \ \ \ \ \ 
\end{aligned}\end{equation}But on the other hand, we know $\|\nabla_{\theta_{q_{2n-1}+i}}g(\theta_{q_{2n-1}+i})\|>o \ \ (i>0)$. Together with  \eqref{022}, we get
\begin{equation}\begin{aligned}
&\sum_{n=1}^{+\infty}\Bigg(\sum_{i=1}^{q_{2n}-q_{2n-1}-1}\frac{1}{S_{q_{2n-1}+i}^{\frac{1}{2}+\epsilon_{1}}}\Bigg)<\frac{1}{o}\sum_{n=1}^{+\infty}\Bigg(\sum_{i=1}^{q_{2n}-q_{2n-1}-1}\frac{\big\|\nabla_{\theta_{q_{2n-1}+i}}g(\theta_{q_{2n-1}+i})\big\|^{2}}{S_{q_{2n-1}+i}^{\frac{1}{2}+\epsilon_{1}}}\Bigg)\\&<\frac{1}{o}\sum_{n=3}^{n}\frac{\big\|\nabla_{\theta_{n}}g(\theta_{n})\big\|^{2}}{S_{n-1}^{\frac{1}{2}+\epsilon}}<+\infty \ \ a.s..
\end{aligned}\end{equation}It   contradicts with \eqref{00ooii9}, so we get that $\|\nabla_{\theta_{n}}g(\theta_{n})\|\rightarrow 0 \ \ a.s.$. Combining \eqref{iuewqrft}, we   get $\|\nabla_{\theta_{n}}g(\theta_{n})\|\rightarrow 0$ no matter $S_{n}<+\infty \ a.s.$ or $S_{n}=+\infty$. Under Assumption \ref{ass_g_poi} 1), it is safe to conclude that there exists a connected component $J^{*}$ of $J$ such that
$						\lim\limits_{n\rightarrow\infty}d(\theta_{n},J^{*})=0$.

\begin{equation}\nonumber\begin{aligned}
&u(t+1)-w(t+1)=u(t)-\eta\nabla L(w(t))-\beta w(t)-(1-\beta)u(t+1).
\end{aligned}\end{equation} That is
\begin{equation}\nonumber\begin{aligned}
&u(t+1)-w(t+1)=\beta(u(t)-w(t))-\eta\nabla L(w(t))-(1-\beta)(u(t)-u(t+1)).
\end{aligned}\end{equation}

%	\begin{thm}\label{thm_converg}
%		Consider the Adagrad as follow:
%		\begin{equation}\label{AdaGrad}\begin{aligned}
%	S_{n}=S_{n-1}+\big\|\nabla_{\theta_{n}} g(\theta_{n},\xi_{n})\big\|^{2},
%	\quad \theta_{n+1}=\theta_{n}-\frac{\alpha_0}{\sqrt{S_{n}}}\nabla_{\theta_{n}}g(\theta_{n},\xi_{n}).\end{aligned} \end{equation}
%		 Then for the loss of logistic regression $g(\theta)$ and the separable data set, we can get $\theta_{n}/\|\theta_{n}\|\rightarrow u^{*}\ \ a.s.$.
%	\end{thm}

\end{document}